\documentclass[12pt]{article}
\usepackage[margin=1in]{geometry}

\usepackage{graphicx}
\usepackage{amsmath}
\usepackage{natbib}

\usepackage{amsfonts}
\usepackage{amssymb}
\usepackage{amsthm}

\usepackage{bm}
\usepackage[utf8]{inputenc}
\usepackage{soul}

\usepackage{ulem}

\usepackage{hyperref}

\allowdisplaybreaks

\newcommand\mat[1]{\begin{pmatrix}#1\end{pmatrix}}
\newcommand{\Var}{{\rm Var}}
\newcommand{\Cov}{{\rm Cov}}
\newcommand{\E}{{\mathbb{E} }}

\newtheorem{theorem}{Theorem}

\newtheorem{lemma}{Lemma}
\newtheorem{proposition}{Proposition}

\newtheorem{assumption}{Assumption}

\title{Statistical Inference for Networks of High-Dimensional Point Processes}
\author{Xu Wang, Mladen Kolar \& Ali Shojaie}

\begin{document}

\maketitle

\begin{abstract}
Fueled in part by recent applications in neuroscience,
the multivariate Hawkes process has become a popular tool for
modeling the network of interactions among high-dimensional point
process data. While evaluating the uncertainty of the network
estimates is critical in scientific applications, existing
methodological and theoretical work has primarily addressed 
estimation. To bridge this gap, this paper develops a new
statistical inference procedure for high-dimensional
Hawkes processes. The key ingredient for this inference procedure is a new concentration inequality on the first- and second-order statistics for integrated
stochastic processes, which summarize the entire history of the
process. Combining recent results on martingale central limit theory
with the new concentration inequality, we then characterize
the convergence rate of the test statistics. We illustrate finite sample validity of our inferential tools via extensive simulations and demonstrate their utility 
by applying them to a neuron spike train data set.
\end{abstract}

\section{Introduction}\label{sec:intro}
Multivariate point process data have become prevalent in a number of emerging application areas. Examples include neural spike train data in neuroscience, containing times of neuron spikes of a collection of neurons \citep{Okatan2005};  social media data, recording times when each individual in an online community takes an action \citep{Zhou2013LearningSI}; and high frequency financial data, recording times of market orders \citep{CHAVEZDEMOULIN2012}. 
The latent connectivity structure of these
processes can be represented by a probabilistic graphical model
\citep{lauritzen1996} with a graph or network $G = (V,E)$ whose nodes,
$v \in V$, represent components/units in the multivariate point processes
and each \textit{directed} edge, $(u \to v) \in E$, indicates that the probability of future events of the target node $v$ depends on the
history of the source node $u$. Multivariate point process data
can be used to learn the structure of this network.
%

In a seminal work, \citet{Hawkes1971} proposed a class of
multivariate point process models, where the probability of future
events for a component can depend on the entire history of events of
other components. Because of its flexibility and interpretability
when modeling the dependence structure of component processes, the multivariate Hawkes process has become a
popular tool for studying the latent network
of point processes. 
From its early application in earthquake
prediction \citep{Ogata1988}, the multivariate Hawkes process model has been
widely used to learn the latent connectivity structure in many fields, including
neuroscience \citep{Shizhe2017}, social media \citep{Zhou2013LearningSI}, and
finance \citep{Bacry2011,Linderman2014}.

The Hawkes process, as introduced in \citet{Hawkes1971} and later studied in
\citet{hawkes1974, reynaud-bouret2007, reynaud-bouret2010, Bacry2015,
hansen2015, Etesami2016}, is considered as a \textit{mutually-exciting} process,
in which an event can only excite future events. More specifically, each event in any component may trigger future events in all other components, including itself.
However, in many applications, it is desired to allow for \textit{inhibitory}
effects of past events. For example, a spike in one neuron may inhibit the
activities of other neurons \citep{Purves2001}, which means that it decreases
the probability that other neurons would spike. \citet{costa2018} and
\citet{Shizhe2017} developed a broader class of Hawkes process models that allow
for both excitatory and inhibitory effects in a single and multivariate point
process data, respectively.

In modern applications, it is common for the number of measured components,
e.g., the number of neurons, to be large compared to the observed time period,
e.g., the duration of neuroscience experiments. The high-dimensional nature of
data in such applications poses additional challenges to learning the
connectivity network of a multivariate point process. \citet{hansen2015} and
\citet{Shizhe2017} proposed $\ell_1$-regularized estimation procedures to
address this challenge. However, there are no tools to characterize the sampling
distribution of these estimators and characterize their uncertainty. Such
inferential tools are critical in scientific applications.

Tools for statistical inference in high-dimensional linear models
\citep{ZhangZhang2014, vandegeer2014, Belloni2012Inference}, graphical models \citep{Barber2015ROCKET, Jankova2018Inference, Lu2015Posta, Yu2019Simultaneous}, and more general estimators \citep{Ning2017, Neykov2018unified} are only recently developed for the setting of independent data and cannot be directly applied in a time series setting.  Statistical inference for high-dimensional vector auto-regressive (VAR) models was recently studied by \citet{Neykov2018unified} and \citet{Zheng2018}.  While a significant step
forward, the VAR model only captures dependence for a fixed and pre-specified
time lag (or order). In contrast, the Hawkes process is dependent on the
\textit{entire} history, which introduces significant challenges in developing
inferential procedures for the high-dimensional multivariate Hawkes process. In
particular, this dependence on the entire history complicates the proof of
convergence of the test statistic for the multivariate Hawkes process. Moreover,
unlike the time series models that are often set up in a discrete time domain
\citep{Basu2015, Zheng2018}, the multivariate Hawkes process is defined in a
continuous time domain and thus requires different technical tools to
investigate its properties.

In this paper, we provide the first high-dimensional inference procedure for
multivariate Hawkes processes with both excitatory and inhibitory effects. To
this end, we adopt the de-correlated score test framework of \citet{Ning2017} to
high-dimensional point processes. We also develop confidence intervals for model parameters by extending the semi-parametric efficient confidence region of \citet{Neykov2018unified, Zheng2018} for VAR models to the setting of the
multivariate Hawkes process. While the general steps for our inference framework are similar to those in de-correlated score test and efficient confidence regions, key challenges in adopting these tools stem from the dependence of Hawkes processes on their entire past and their continuous-time nature. In particular, to establish our inference framework, we tackle two main challenges: (i) deriving concentration inequalities for summary statistics of
high-dimensional Hawkes processes; and (ii) establishing the restricted
eigenvalue (RE) condition required for the estimation consistency of
$\ell_1$-regularized estimators \citep{bickel2009}.

To address the above challenges, we first generalize the results by
\citet{costa2018} and \citet{Shizhe2017} to obtain new concentration
inequalities on the first- and second-order statistics of the integrated
stochastic process that summarizes the entire history of each component of the
multivariate Hawkes process. These inequalities are essential for developing our
high-dimensional inference procedures. For instance, together with the
martingale central limit theorem (CLT) of \citet{Zheng2018}, they allow us to
establish the convergence of our test statistics to a $\chi^2$ distribution.
They are also used to establish the maximal inequalities needed to investigate
the theoretical properties of the $\ell_1$-regularized estimator. Next, to bound
the eigenvalues of the covariance matrix of the integrated process, we link
these eigenvalues to the spectral density of the transition matrix of the Hawkes
process. By carefully examining the transition functions of the multivariate
Hawkes process, we investigate structural conditions on the transition functions
that are sufficient to bound the eigenvalues. These bounds allow us to establish
the convergence of our test statistic, and are also used to verify the
restricted eigenvalue (RE) condition \citep{bickel2009}.


\section{The Linear Hawkes Process}\label{sec:hawkes}
Let $\{t_k\}_{k\in \mathbb{Z}}$ be a sequence of real-valued random variables,
taking values in $[0, T]$, with $t_{k+1} > t_k$ and $t_1 \ge 0$ almost surely.
Here, time $t = 0$ is a reference point in time, e.g., the start of an
experiment, and $T$ is the duration of the experiment. A simple point process
$N$ on $\mathbb{R}$ is defined as a family $\{ N(A) \}_{ A \in
\mathcal{B}(\mathbb{R}) }$, where $\mathcal{B}(\mathbb{R})$ denotes the Borel
$\sigma$-field of the real line and $N(A) = \sum_k \mathbf{1}_{\{t_k \in A\}} $.
The process $N$ is essentially a simple counting process with isolated jumps of
unit height that occur at $\{t_k\}_{k\in \mathbb{Z}}$. We write $ N([t, t + dt)
)$ as $dN(t)$, where $dt$ denotes an arbitrarily small increment of $t$.

Let $\mathbf{N}$ be a $p$-variate counting process $\mathbf{N} \equiv \{
N_i\}_{i\in \{1,\dots, p \}}$, where, as above, $N_i$ satisfies $N_i(A) = \sum_k
\mathbf{1}_{\{t_{ik} \in A\}}$ for $A \in \mathcal{B}(\mathbb{R})$ with $\{
t_{i1}, t_{i2}, \dots \}$ denoting the event times of $N_i$.  Let
$\mathcal{H}_t$ be the history of $\mathbf{N}$ prior to time $t$.  The intensity
process $\{ \lambda_{1}(t), \dots, \lambda_{p}(t) \}$ is a $p$-variate
$\mathcal{H}_t$-predictable process, defined as
\begin{align}
	\lambda_i (t) dt & = \mathbb{P}(dN_{i}(t)=1 \mid \mathcal{H}_t )  .
\end{align}
\citet{Hawkes1971} proposed a class of point process models in which
past events can affect the probability of future events. The process $\mathbf{N}$ is a \textit{linear Hawkes process} if the intensity function for each unit $i$ ($i \in \{1,\ldots,p\}$) takes the form
\begin{align}
  \lambda_{i}(t)
  &=   \mu_{i} + \sum_{j=1}^{p} \left(\omega_{ij} *  dN_{j} \right )(t)
    \label{eq:linear_hawkes}  ,
\end{align}
where
\begin{align}
  \left( \omega_{ij} *  dN_{j}  \right)(t)  =
  \int_0^{t-}  \omega_{ij}(t-s)   dN_{j} (s)
  = \sum_{ k: t_{jk} < t }   \omega_{ij}( t- t_{jk}) . \label{eq:transfer_function}
\end{align}
Here, $\mu_{i}$ is the background intensity of unit $i$ and
$\omega_{ij}(\cdot): \mathbb{R}^+ \rightarrow \mathbb{R}$ is the
\textit{transfer function}. In particular, $\omega_{ij}( t - t_{jk}) $
represents the influence from the $k$th event of unit $j$ on the intensity
of unit $i$ at time $t$.

Motivated by neuroscience applications \citep{Linderman2014, MAGRANSDEABRIL2018120}, we consider a parametric transfer function
$\omega_{ij}(\cdot)$ of the form
\begin{equation} \label{eq:omega}
\omega_{ij} (t)    =   \beta_{ij} \kappa_{j}(t) 
\end{equation}
with a \textit{transition kernel}
$\kappa_j(\cdot): \mathbb{R}^+ \rightarrow \mathbb{R}$ that captures the decay of the dependence on past events. This leads to $ \left( \omega_{ij} * dN_{j}  \right)(t) = \beta_{ij} x_{j}(t)$, where the \textit{integrated stochastic process}
\begin{equation}\label{eq:design_column_xt}
 x_{j}(t)  = \int_0^{t-} \kappa_{j}(t-s) dN_{j}(s) 
\end{equation}
summarizes the entire history of unit $j$ of the multivariate Hawkes processes. A commonly used example is the exponential transition kernel, $\kappa_{j}(t) = e^{-t}$ \citep{Bacry2015}.


In this formulation, the \textit{connectivity coefficient} of the underlying network, $\beta_{ij}$, represents the strength of the dependence of unit $i$'s intensity on unit $j$'s past events. A positive $\beta_{ij}$, which implies that past events of
unit $j$ \textit{excite} future events of unit $i$, is often considered in the
literature \citep[see, e.g.,][]{Bacry2015, Etesami2016}. However, we
might also wish to allow for negative $\beta_{ij}$ values to represent
\textit{inhibitory} effect of one unit's past events on another
\citep{Shizhe2017,costa2018}, which is expected in neuroscience applications \citep{Purves2001}. 
%

Denoting
$\bm{x}(t)= ( x_1(t), \dots, x_p(t) )^\top \in \mathbb{R}^{p}$ and
$\bm{\beta}_i = (\beta_{i1}, \dots, \beta_{ip})^\top \in
\mathbb{R}^{p}$, we can write
\begin{align}
  \lambda_{i}(t)
  &=   \mu_{i} + \bm{x}^\top(t)\bm{\beta}_i  \label{eq:linear_hawkes_para_transfer} .
\end{align}
Furthermore, let $Y_i(t) = {dN_i(t)}/{dt}$ and
$\epsilon_i(t) =Y_i(t) - \lambda_i(t) $.  Then the linear Hawkes
process can be written compactly as
\begin{align}
  Y_i(t) &= \mu_{i} + \bm{x}^\top(t)\bm{\beta}_i  + \epsilon_i(t) \label{eq:y_t}.
\end{align}
As we will discuss later, a key challenge in this `linear model'
stems from heteroscedasticity: the variance of $\epsilon_i(t)$ given the history
of $\mathbf{N}$ up to $t$,
\begin{align}
  \sigma_i^2(t) \equiv \Var\left(\epsilon_i(t) \mid \mathcal{H}_t \right) = \lambda_i(t)(1-\lambda_i(t)),
  \label{eq:sigma_i}
\end{align}
may not necessarily be 1 and depends on $\bm{x}(t)$.

Throughout this paper, we assume that the linear Hawkes model
described above is \textit{stationary}, meaning that for all units
$i=1,\dots, p$, the spontaneous rates $\mu_i$ and strengths of
transition $\bm{\beta}_i$ are constant over the time range $[0,T]$
\citep{Bremaud1996, Daley2003}.


\section{Testing}\label{sec:test}

Let $J \subset \{1,\dots, p \}$ be an index set of cardinality $|J| = d$ and
denote $\bm{\beta}_{iJ}=\{ \beta_{ij}, j\in J \}$. We consider testing a
$d$-dimensional subset $\bm{\beta}_{iJ}$:
\begin{align}
	H_0 &: \beta_{ij} = 0, \qquad j \in J \label{eq:H_0}.
\end{align}
For ease of notation, we primarily focus on the case of a single parameter; that
is, we consider testing $H_0: \beta_{ij}=0$, which corresponds to $d =1$.
However, our inferential framework is developed for the more general case of $d
\ge 1$.

In order to simplify the presentation, we scale the components of $\bm{x}(t)$ by
the variance of the noise $\sigma^2_i(t)$ defined in \eqref{eq:sigma_i}. Denote
the scaled components, $z_j(t) = {x_j(t)}/{\sigma_i(t)}$ for $j=1,\dots, p$.
Next, we define the orthogonal projection of $z_j(t)$ onto $\bm{z}_{-j}(t)$,
where $\bm{z}_{-j}(t) = (z_1(t), \dots, z_{j-1}(t), z_{j+1}(t),\dots , z_p(t)
)^\top \in \mathbb{R}^{p-1}$. Let the projection coefficient $\bm{w}_j^* =
\left( w^*_{j0}, ( \bm{w}_{j,-j}^*)^\top \right)^\top \in \mathbb{R}^p $ be
such that
\begin{align}
 \mathbb{E}\left[ z^*_j(t) \bm{z}_{-j}(t) \right] = 0
 \qquad\text{and}\qquad
	\mathbb{E}\left[  z^*_j(t)    \right] = 0,
 \label{eq: w_j_star}
\end{align}
where
\begin{align}
  z^*_j(t) \equiv z_j(t)  -\mat{1, \bm{z}^\top_{-j}(t) }  \bm{w}_j^*
  \label{eq:x_tilde_star}
\end{align}
denotes the orthogonal complement of $z_j(t)$ after removing its projection onto
$\bm{z}_{-j}(t)$. In particular, $z^*_j(t)$ is uncorrelated with
$\bm{z}_{-j}(t)$. Let $\widetilde{ \epsilon}_i(t) = \epsilon_i(t)/\sigma_i(t)$.

With this notation, we define the de-correlated score statistic as
\begin{align}
S_{ij} &= \frac{1}{T} \sum_{t=1}^{T}  \widetilde{ \epsilon}_i(t) \,z^*_j(t)
\label{eq:S}.
\end{align}
The de-correlated score statistic is constructed using $z^*_j(t)$, instead of
directly using $z_j(t)$, to make its sampling distribution robust to errors
induced by the estimation of the unknown nuisance parameters, $\mu_i$ and
$\bm{\beta}_{i,-j}$ \citep{Ning2017}. In particular, we will show that  the
induced error is asymptotically negligible and the limiting distribution of the
test statistic does not depend on model selection mistakes that occur when
estimating the nuisance parameters.

\citet{Neykov2018unified} and \cite{Zheng2018} consider a similar de-correlated
score statistic in the context of VAR models. Their proof strategy exploits the
homoscedastic noise variance in VAR models and does not extend to the linear
Hawkes process. In contrast, we need to take into account the noise variance when constructing the score statistic, as the variance varies over time. Moreover, the noise variance depends on the intensity value, which is time
varying, resulting in more challenging proof in our case.

In order to construct a test, we need to characterize the quantiles of the
de-correlated score statistic. Let
\begin{align}
  \Upsilon_j &= {\Cov}\left ( z^*_j(t) \right )
               = {\mathbb E} \left(    \big(  z^*_j(t) \big)^2 \right ) ,
               \label{eq:Upsilon} \\
  V_T &= \sqrt{T} \, \Upsilon_j^{-1/2} S_{ij} \label{eq:V_T}  ,\\
  U_T &= \lVert  V_T \rVert_2^2 \label{eq:U_T} .
\end{align}
While $\Upsilon_j$ is scalar when testing a univariate $\beta_{ij}$, when testing multiple parameters $\bm{\beta}_{iJ}$, $\Upsilon_J$ is a $d\times d$ matrix defined as
\[
  \Upsilon_J = \Cov \left( \bm{z}^*_J(t) \right ).
\]
In the next section, we show that $U_T$ converges weakly to a $\chi^2$
distribution with degrees of freedom $d$, which is 1 for testing a
univariate $\beta_{ij}$. The non-centrality parameter is zero under
the null hypothesis and depends on the true parameters under the
alternative.

In practice, $\left (\mu_i, \bm{\beta}_i \right )$ and $\bm{w}_j^*$ are not known. We next describe a procedure for estimating them.

\textbf{Step 1}: Calculate $\widehat{\mu}_i$,
$\widehat{ \bm{\beta} }_i$, and $\widehat{\sigma}^2_i(t)$.  We
estimate $\widehat{\mu}_i, \widehat{ \bm{\beta} }_i$ using the lasso on the
unscaled data $(Y_i(t), \bm{x}(t) )$:
\begin{align}
  \widehat{\mu}_i, \widehat{ \bm{\beta} }_i = \arg \min_{\mu_i\in \mathbb{R}, \bm{\beta}_i \in \mathbb{R}^p}
  \frac{1}{T} \sum_{t=1}^{T} \left ( Y_i(t) -\mu_i - \bm{x}^\top(t)\bm{\beta}_i \right )^2 + \lambda \lVert \bm{\beta}_i \rVert_1 . \label{eq:lasso_beta}
\end{align}
Then,
\begin{align}
  \widehat{\lambda}_i(t) = \bm{x}^\top(t) \widehat{ \bm{\beta} }_i
	\qquad\text{and}\qquad
  \widehat{\sigma}^2_i(t) = \widehat{\lambda}_i(t)(1-\widehat{\lambda}_i(t)) \label{eq:sig2hat}.
\end{align}
The estimation consistency for $\widehat{\mu}_i$, $\widehat{ \bm{\beta} }_i$ to
the corresponding true parameters is shown in Lemma~\ref{Lemma_beta_lasso}. The
consistency of $\widehat{ \sigma}^2(t)$ to $\sigma^2(t)$ follows from the
prediction consistency of the lasso estimator. In our proof, we show that the
\textit{restricted eigenvalue} (RE) condition required for the consistency of
lasso \citep{bickel2009} is met in our case. This follows from the bounded
eigenvalue of the covariance matrix of the integrated stochastic process
$\bm{x}(t)$, which is obtained under the assumptions made in the following
section.

The tuning parameter $\lambda$ is selected via cross-validation for sequentially
dependent data \citep{Safikhani2017}. Specifically, unlike standard
cross-validation  for independent samples, the sequential cross-validation  uses
successively training sets along with validation sets that follow each of the
training sets in the sequence order.

\textbf{Step 2}: Calculate $\widehat{\bm{w}}_j$. Let $\widehat{z}_j(t) =
{x_j(t)}/{\widehat{\sigma}_i(t)}$ for $j=1,\dots, p$.
%
We estimate $\widehat{\bm{w}}_j$ by regressing the outcome
$\widehat{z}_{j}$ on the design matrix
$\widehat{\bm{z}}_{-j}$ using a lasso procedure with tuning parameter selected as in Step 1:
\begin{align}
  \widehat{\bm{w}}_j = \arg \min_{ \bm{w}_j \in \mathbb{R}^p  } \frac{1}{T} \sum_{t=1}^{T} \left(\widehat{z}_j(t)  -\mat{1 & \widehat{\bm{z}}^\top_{-j}(t) } \bm{w}_j  \right)^2 + \lambda \lVert \bm{w}_{j,-j} \rVert_1 . \label{eq:lasso_w_hat}
\end{align}
The consistency of $\widehat{\bm{w}}_j$ for $\bm{w}_j^*$ is shown in
Lemma~\ref{Lemma_w_lasso}. The proof is similar to the one used for Step 1.

\textbf{Step 3}: Calculate $\widehat{\Upsilon}_j $. Let $\widehat{z}^*_j(t) =
\widehat{z}_j(t) -\mat{1 & \widehat{\bm{z}}^\top_{-j}(t) }\widehat{\bm{w} }_j$.
When testing a univariate $\beta_{ij}$ in \eqref{eq:H_0}, ${\Upsilon}_j$
is a scalar and we estimate it by the sample covariance as
\begin{align}
  \widehat{\Upsilon}_j &=  \frac{1}{T} \sum_{t=1}^{T}
                         \big( \widehat{z}^*_j(t) \big )^2.
                         \label{eq:upsilon_hat}
\end{align}
When testing multiple $\bm{\beta}_{iJ} = \{ \beta_{ij}, j\in J \}$,
we estimate ${\Upsilon}_J$ as
\[
  \widehat{\Upsilon}_J =  \frac{1}{T} \sum_{t=1}^{T}
   \widehat{z}^*_J(t)    \left(\widehat{z}^*_J(t) \right)^\top \in \mathbb{R}^{J\times J}.
\]
Our results are valid as long as $d = |J| \ll p$.

\textbf{Step 4}: Putting everything together, we compute the de-correlated score statistic
with estimated nuisance parameters. Let
\begin{align}
  \widehat{\epsilon}_i(t) &=
                            Y_i(t) - \widehat{\mu}_i - \bm{x}^\top_{-j}(t)\widehat{\bm{\beta} }_{i,-j} ,  \\
  \widehat{S}_{ij} &= \frac{1}{T} \sum_{t=1}^{T} \frac{ \widehat{\epsilon}_i(t) }{\widehat{\sigma}_i(t) }  \widehat{z}^*_j(t)  , \label{eq:S_hat} \\
  \widehat{V}_T &= \sqrt{T} \widehat{\Upsilon}_j^{-1/2} \widehat{S}_{ij} ,  \\
  \widehat{U}_T &= \lVert  \widehat{V}_T \rVert_2^2 .\label{eq:U_hat_T}
\end{align}
The above quantities are univariate when testing a univariate $\beta_{ij}$, but
are defined as vectors and matrices when testing multivariate $\bm{\beta}_{iJ}$
of dimension $d$. Specifically, $\widehat{S}_{iJ} \in \mathbb{R}^d$,
$\widehat{\Upsilon}_J \in \mathbb{R}^{d\times d}$ and $\widehat{V}_T \in
\mathbb{R}^d$.

In the next section, we show that with high probability $\widehat{U}_T$
converges to $U_T$, which asymptotically follows a $\chi^2$ distribution with
$d$ degrees of freedom under the null hypothesis. Thus, we define our test
procedure for \eqref{eq:H_0} as
\begin{align*}
\Phi_\alpha = \textrm{I} \left\{ \widehat{U}_T \ge \chi^2_{d,1-\alpha} \right \} ,
\end{align*}
and reject the null hypothesis when $\Phi_\alpha = 1$.




\section{Theoretical Guarantees}\label{sec:theory}

In this section, we present our main theoretical results, which characterize the
limiting distribution of $\widehat{U}_T$. We start by stating our assumptions.
For a square matrix $A$, let $\Lambda_{\max}(A)$ and $\Lambda_{\min}(A)$ be its
maximum and minimum eigenvalues, respectively.
Define
$\Theta = \{ \beta_{ij} \}_{1\le i,j \le p} \in \mathbb{R}^{p\times p}$ and $\bm{\mu} = \{\mu_i \}_{1\le i \le p} \in \mathbb{R}^p$.

\begin{assumption}\label{assumption1} Let $\Omega = \{ \Omega_{ij} \}_{1\le i,j
\le p} \in \mathbb{R}^{p\times p}$ with entries $\Omega_{ij} =   \int_0^{\infty}
|\omega_{ij}(\Delta)| d\Delta$. There exists a constant $\gamma_{\Omega}$ such
that $\Lambda_{\max} ( \Omega^T \Omega ) \le \gamma^2_{\Omega}  < 1 $.
\end{assumption}
Assumption~\ref{assumption1} is necessary for stationarity of a Hawkes process
\citep{Shizhe2017}.  The constant $\gamma_{\Omega}$ does not depend on the
dimension $p$.   For any fixed $p$, \citet{Bremaud1996} show that given this
assumption the intensity process of the form \eqref{eq:linear_hawkes} is stable
in distribution and, thus, a stationary process $\mathbf{N}$ exists. Since our
connectivity coefficients of interest, $\Theta$, are ill-defined without a
stationarity, this assumption provides the necessary context for our inferential
framework.

\begin{assumption}\label{assumption2}
There exists constants $\rho_r \in (0,1)$ and $0< \rho_c < \infty $ such that
\begin{align*}
\max_{1\le i \le p}  \sum_{j=1}^p \Omega_{ij} \le \rho_r
\qquad\text{and}\qquad
\max_{1\le j \le p}  \sum_{i=1}^p \Omega_{ij} \le \rho_c .
\end{align*}
\end{assumption}
Assumption~\ref{assumption2} requires maximum in- and out- intensity flows to be
bounded, which helps in bounding the eigenvalues of the cross-covariance of
$\bm{x}(t)$. A similar assumption is also considered by \citet{Basu2015} in
the context of VAR models. The condition of $\rho_r \in (0,1)$ prevents the
intensity from concentrating to a single process \citep{Shizhe2017}.

\begin{assumption}\label{assumption3}
There exists $\lambda_{\min}$ and $\lambda_{\max}$ such that
$$
0 < \lambda_{\min} \le \lambda_{i}(t) \le \lambda_{\max} 
$$
for all $i=1,\dots, p$ and $t\in [0, T]$.
\end{assumption}
Assumption~\ref{assumption3} requires that the intensity rate is strictly
bounded, which prevents degenerate processes for all units of
the multivariate Hawkes process. As a consequence, $\sigma_i^2(t)$ will be bounded
away from 0, and hence the construction of the de-correlated score in
\eqref{eq:S}  is valid.

\begin{assumption}\label{assumption4}
The transition kernel $k_{j}(t)$ is positive and integrable over $[0,T]$, for $1\le j \le p$.
\end{assumption}
Assumption~\ref{assumption4} implies that the integrated process $x_j(t)$
defined in \eqref{eq:design_column_xt} is bounded. Together with
Assumptions~\ref{assumption3}, it also implies that $\mu_i $ and $\bm{\beta}_i$
are bounded for all $i=1,\dots,p$.

Our next two assumptions  state the rate of convergence for the estimators of
$(\mu_i,\bm{\beta}_i)$ and $\bm{w}^*_j$ that guarantee the weak convergence of the
test statistic. We let $\Pi_0$ and $\Pi_a$ denote the feasible set of
$(\bm{\mu},\Theta)$ under the null and the alternative hypotheses, respectively,
with Assumptions~\ref{assumption1}--~\ref{assumption4} satisfied. In
addition, we use $s_j = \lVert\bm{w}_j^* \rVert_0$, $s = \max_{1\le j\le p}
s_j$, $\rho_i = \lVert \bm{\beta}_i \rVert_0$, and $\rho = \max_{1\le i\le p}
\rho_i$ to denote the sparsity of $\bm{w}^*_j$ and $\bm{\beta}_i$.

\begin{assumption}[Estimation error of $\bm{\beta}_i$]\label{assumption_beta}
For $(\bm{\mu},\Theta) \in \Pi_0 \cup \Pi_a$ and $r\in \{1,2\}$,
\begin{align*}
\left \lVert \mat{\widehat{ \mu}_i \\ \widehat{\bm{\beta}}_i} - \mat{ \mu_i \\ \bm{\beta}_i} \right  \rVert_r &\le  C_1 (\rho+1)^{1/r}  T^{-2/5},
\end{align*}
for all $1\le j \le p$, with probability at least $1-C_2 p^2 T \exp(- C_3
T^{1/5})$.  Constants $C_1, C_2, C_3$ only depend on $(\bm{\mu},\Theta) $ and
the transition kernel function.
\end{assumption}

\begin{assumption}[Estimation error of $\bm{w}^*_j$]\label{assumption_w}
For $(\bm{\mu},\Theta) \in \Pi_0 \cup \Pi_a$ and $r\in \{1,2\}$,
		\begin{align*}
	\left \lVert \widehat{\bm{w} }_j  - \bm{w}^*_j \right \rVert_r &\le C_1 (s +1)^{\frac{3-r}{2}}  \rho  T^{-2/5},
\end{align*}
for all $1\le j \le p$, with probability at least $1- C_2 p^2 T \exp(- C_3
T^{1/5})$. Constants $C_1, C_2, C_3$  only depend on $(\bm{\mu},\Theta) $ and
the transition kernel function.
\end{assumption}

Assumption~\ref{assumption_beta} and~\ref{assumption_w} state the rate of
convergence for estimators of the nuisance components. Under  a stationary
linear Hawkes process that satisfies
Assumptions~\ref{assumption1}--\ref{assumption4},  Lemmas~\ref{assumption_w}
and~\ref{Lemma_beta_lasso} in Appendix~B show that the lasso estimators in
\eqref{eq:lasso_beta} and \eqref{eq:lasso_w_hat} satisfy
Assumptions~\ref{assumption_beta} and \ref{assumption_w}. However, our main
results on the limiting distribution of the test statistic are  valid for other
high-dimensional estimators,  as long as their rate of convergence satisfies the
requirements  in  Assumptions~\ref{assumption_beta} and \ref{assumption_w}.

The rate of convergence naturally depends on the sparsity of $\bm{\beta}_i$ and
$\bm{w}^*_j$. In general, the relationship between the sparsity of $\bm{w}_j^*$
and the sparsity of $\bm{\beta}_i$ is not straightforward ---  it depends on the
sign and scale of the connectivity coefficients, as well as the transition
kernel.
Lemma~\ref{lemma_sparsity_s_rho} in Appendix~B shows that the sparsity of
$\bm{w}_j^*$ is upper bounded by the sparsity of $\bm{\beta}_i$ as $s \le 2\rho
+ 1$ when
connectivity  matrix is block diagonal. The rate of convergence for an
estimator of $\bm{w}^*_j$ in  Assumption~\ref{assumption_w} depends on both the sparsity of $\bm{w}^*_j$ and $\bm{\beta}_i$. This is because estimation of $\bm{w}^*_j$ requires an estimate of the unknown variance, which in turn depends on the estimates of $(\mu_i, \bm{\beta}_i)$.

Next, we introduce our first result, which establishes the weak convergence of
$\widehat{U}_T$ under the null hypothesis.

\begin{theorem}\label{theorem1}
	Suppose the linear Hawkes process defined in \eqref{eq:linear_hawkes}
  satisfies Assumptions \ref{assumption1}--\ref{assumption4}.
	Furthermore assume that $\left(\widehat{\mu}_i, \widehat{\bm{\beta} }_i\right)$ and
  $\widehat{\bm{w} }_j$ satisfy Assumptions~\ref{assumption_beta} and~\ref{assumption_w}.
  If $s^2 \rho^2 \log p = o\left( T^{1/5} \right)$, then, under the null hypothesis in \eqref{eq:H_0},
\begin{align}
\underset{ (\Theta, \mu) \in \Pi_0, x\in \mathbb{R}  }{\sup} \left |  {\mathbb P}\left(\widehat{U}_T \le x\right) - F_d(x)  \right |
\le C_1 p^2 T \exp(- C_2 T^{1/5}) + C_3 s^2 \rho^2   T^{-1/5}+ C_4 T^{-1/8}  ,
\end{align}
where
$F_d$ is the cdf of the $\chi^2$-distribution with $d$ degrees of freedom and
$C_k,k=1,\dots,4$, are constants only depending on the model parameter $(\bm{\mu}, \Theta)$ and the transition kernel function.
\end{theorem}

Theorem~\ref{theorem1} shows that $\widehat{U}_T$ converges to $\chi_d^2$ in
distribution ($d=1$ when testing univariate $\beta_{ij}$). The proof involves
quantifying the difference between the cdf of $U_T$ and $F_d(x)$ and
establishing the convergence of $\widehat{U}_T$ to $U_T$.  The main challenge in
establishing these results stems from the dependency structure of the
multivariate Hawkes process whose intensity depends on the entire history of
each component. The additional complexity due to the dependency structure leads
to a slower rate of convergence in quantifying the difference between
$\widehat{U}_T$ and $U_T$ than those obtained for the VAR model \citep{Zheng2018}.
Moreover, this dependence also leads to a difference between cdf of $U_T$ and
$F_d(x)$ that is dominated by $T^{-1/8}$ (using the martingale central limit
theorem in Proposition~\ref{prop_clt}) rather than $T^{-1/2}$ (using the
standard central limit theorem).

Next, we investigate the distribution of $\widehat{U}_T$ under the alternative
hypothesis. More specifically, for $\phi > 0$, we assume
\begin{align}
	H_a: \beta_{ij} = T^{-\phi}\Delta \label{eq:H_a} .
\end{align}

\begin{theorem}\label{theorem2}
Suppose the linear Hawkes process defined in (\ref{eq:linear_hawkes}) satisfies
Assumptions~\ref{assumption1}--\ref{assumption4}.
Furthermore assume that $\left(\widehat{\mu}_i, \widehat{\bm{\beta} }_i\right)$
and $\widehat{\bm{w}}_j$ satisfy Assumptions~\ref{assumption_beta} and~\ref{assumption_w}.
Let $F_{d, \delta }$ be the cdf of a non-central
$\chi^2$-distribution with $d$ degrees of freedom and non-centrality parameter
$\delta$.
If $ s^2 \rho^2 \log p = o\left( T^{1/5} \wedge T^{  2\phi-\frac{7}{5} }  \right)$,
then, under the alternative hypothesis in \eqref{eq:H_a},
we have:
\begin{itemize}
	\item[--]if  $\phi > \frac{1}{2}$,
	\begin{multline}
	\underset{ (\Theta, \mu) \in \Pi_a, x\in \mathbb{R}  }{\sup} \left |  {\mathbb P}\left(\widehat{U}_T \le x\right) - F_d(x)  \right |
	\le  C_1 p^2 T \exp(- C_2 T^{1/5}) \\
	+ C_3 s^2 \rho^2 \left( T^{-1/5} \vee T^{ \frac{7}{5} - 2\phi } \right)  + C_4 T^{-1/8}   ;
	\end{multline}
	\item[--]if $\phi = \frac{1}{2}$,
	\begin{multline}
	\underset{ (\Theta, \mu) \in \Pi_a, x\in \mathbb{R}  }{\sup}
	\left |  {\mathbb P}\left(\widehat{U}_T \le x\right) - F_{d, \lVert \widetilde{\Delta} \rVert_2^2 }(x)  \right| \le
	C_1 p^2 T \exp(- C_2 T^{1/5}) \\
	+ C_3 s^2 \rho^2  T^{-1/5}  + C_4 T^{-1/8}
	 ,
	\end{multline}
	where $\widetilde{\Delta} = \Upsilon_j^{1/2} \Delta$ with $\Upsilon_j$ defined in \eqref{eq:Upsilon};
	\item[--]if $  \phi <  \frac{1}{2}$,
	\begin{multline}
	\underset{ (\Theta, \mu) \in \Pi_a, x\in \mathbb{R}  }{\sup}
	\left |  {\mathbb P}\left(\widehat{U}_T \le x\right)  \right| \le
	C_1 p^2 T \exp(- C_2 T^{1/5})  + C_3 T^{-1/8}\\
	 + C_4 \exp(  -( C_5 T^{1/2 -\phi} - C_6 \sqrt{x} )^2  ) ;
	\end{multline}
\end{itemize}
here $C_k, k=1,\dots,6$, are constants only depending on the model parameter $(\bm{\mu}, \Theta)$ and the transition kernel function.
\end{theorem}
Theorem~\ref{theorem2} establishes the asymptotic distribution of
$\widehat{U}_T$ under the alternative hypothesis. Depending on the scaling of
$\beta_{ij}$ with respect to $T$, which is parameterized by $\phi$, the
asymptotics are different. When $\phi > 1/2$, our test does not distinguish
$H_a$ from $H_0$, since in both cases $\widehat{U}_T$ convergences to
$\chi^2_d$. When $\phi < 1/2$, $\widehat{U}_T$ diverges to $+\infty$ in
probability, resulting in trivial rejection of the null hypothesis. Finally, when $\phi = 1/2$, $\widehat{U}_T$ converges to a non-central $\chi^2$-distribution with $d$ degrees of freedom and non-centrality parameter $\lVert \widetilde{\Delta} \rVert_2^2$. This result should be compared with Theorem~3.2 in \citet{Zheng2018} developed for VAR models.  However, since the multivariate Hawkes process is defined on a continuous time domain with intensity rate depending on the entire history, rather than a pre-specified time lag (or order), our rate of convergence is slower compared with the VAR model.

The results in Theorem~\ref{theorem1} and \ref{theorem2} are established by
extending the concentration inequality developed in \citet{Shizhe2017}, which is
built on a Bernstein type inequality for weakly dependent observations of the
point process at different time points \citep{Merlevede}. This weak dependence
leads to a slower rate of convergence in the second order statistics of
$\bm{x}(t)$ compared with the standard sub-Gaussian deviation bound for
independent samples; see \citet[][Theorem~4]{Shizhe2017} or
\citet[][Theorem~1]{Merlevede} for details.
As an alternative, in Lemma~\ref{lemma_concen_ineq_x2} in the Appendix~C, we write the relevant statistics (i.e. the second order statistics of $\bm{x}(t)$ in our case) based on independent `residuals', referred to as martingale compensated processes \citep{Bacry2011}. Then, considering the point process in a discrete time domain, we use the Hansen-Wright inequality to obtain a faster rate of convergence that is comparable to the standard Gaussian deviation bounds for VAR models in \citet{Zheng2018}. However, this is obtained under a more stringent requirement on the structure of the transfer function of the Hawkes processes.

\section{Confidence Regions}\label{sec:confidence}

We next describe a procedure for constructing a confidence intervals for
$\beta_{ij}$. Similar to \citet{Ning2017}, our confidence interval is based on
the one-step estimator of $\beta_{ij}$. Let $\widehat{\beta}_{ij}$ be the lasso
estimator in \eqref{eq:lasso_beta}, or any other consistent estimator with the
same order of the estimation error, and let
\begin{align}
	\label{eq:upsilon_hat_tilde}
	\widetilde{\Upsilon}_j =
	\frac{1}{T} \sum_{t=1}^{T}
	\widehat{z}^*_j(t)  \widehat{z}_j(t)
	\qquad\text{and}\qquad
\widetilde{S}_{ij} =
\frac{1}{T} \sum_{t=1}^{T} \frac{ Y_i(t) - \widehat{\mu}_i - \bm{x}^\top(t) \widehat{\bm{\beta} }_i }{\widehat{\sigma}_i(t)  } \widehat{z}^*_j(t) .
\end{align}
Note that $\widetilde{S}_{ij}$ involves the entire $\widehat{\bm{\beta} }_i$,
compared with $\widehat{S}_{ij}$ which uses $\widehat{\bm{\beta}}_{i,-j}$.
We define the one-step estimator
of $\beta_{ij}$ as
\begin{align}
  \widehat{b}_{ij} = \widehat{\beta}_{ij} - \left( \widetilde{\Upsilon}_j \right)^{-1} \widetilde{S}_{ij} \label{eq:one_step_est_bij}.
\end{align}
Finally, let
\begin{align}
	\widehat{R}_T = T \left (\widehat{b}_{ij} - \beta_{ij} \right )^\top  \widehat{\Upsilon}_j \left (\widehat{b}_{ij} - \beta_{ij} \right ) \label{eq:R_hat_T}.
\end{align}

Our next result shows that $\widehat{R}_T$ converges weakly to $\chi^2_d$.
Therefore, we construct an asymptotically $1-\alpha$ confidence region for $\beta_{ij}$ as
\begin{align}
	\textrm{CR}(\alpha) = \left \{
	\theta :
	T \left (\widehat{b}_{ij} - \theta \right )^\top  \widehat{\Upsilon}_j \left (\widehat{b}_{ij} - \theta \right )\le \chi_d^2(1-\alpha)
	\right \}  \label{eq:CI}.
\end{align}

\begin{theorem}\label{theorem3}
Suppose the linear Hawkes process defined in \eqref{eq:linear_hawkes}
satisfies Assumptions \ref{assumption1}-~\ref{assumption4}.
Furthermore $\left(\widehat{\mu}_i, \widehat{\bm{\beta} }_i\right)$ and
$\widehat{\bm{w} }_j$ satisfy Assumption~\ref{assumption_beta} and~\ref{assumption_w}.
If $ s^2 \rho^2 \log p = o\left( T^{1/5} \right)$, then
\begin{align*}
\sup_{x\in \mathbb{R} } \left  |\mathbb{P}(\widehat{R}_T \le x) - F_d(x) \right |
 \le   C_1 p^2 T \exp(- C_2 T^{1/5})
 + C_3 s^2 \rho^2  T^{-1/5}  + C_4 T^{-1/8}   ,
\end{align*}
where $C_k,k=1,\dots,4$, are constants only depending on the model parameter
$(\bm{\mu}, \Theta)$ and the transition kernel function.
\end{theorem}

\section{Simulation Studies}\label{sec:sims}

We illustrate finite sample properties of the proposed inference procedure
through extensive simulations. We consider the linear Hawkes process with the
transfer function specified in~\eqref{eq:linear_hawkes_para_transfer}.  For the
connectivity matrix $\Theta = \{\beta_{ij} \}_{1\le i,j\le p}$, we consider
three structures: chain, block and random, with $p=50$ component processes.  The chain structure contains nodes of component processes sequentially connected; the block structure contains 25 blocks with 2 component processes mutually connected within each block; the random structure is created by randomly assigning edges over all possible pairs of the component processes with a total sparsity of about 2\%. Figure~\ref{fig:simu_setting} illustrates the
connectivity matrices under the three graph structures. The background intensity
$\mu_i$ is set to be 0.2 and the scale of non-zero elements $\beta_{ij}$ is set
to be 0.3. The transfer kernel function $k_{ij}(t)$ is chosen to be $\exp(-t)$.
This setting satisfies our assumptions of a stationary Hawkes process.

To assess the performance of our method, we test each of the $p^2$ coefficients
in the connectivity matrix. We calculate the type-I error (i.e. the
rejection rate among zero coefficients) and the power (i.e. the rejection rate
among non-zero coefficients). We also investigate the empirical coverage
of the 95\% confidence intervals for zero and non-zero coefficients.
We consider experiments lengths $T \in \{200, 1000, 2000\}$. As a benchmark, we compare the performance of our methods against an oracle procedure, which knows what coefficients are non-zero.

Figure~\ref{fig:hdi_result} illustrates the simulation results for chain, block
and random structure separately. It can be seen that as the experiment length
increases, our test properly controls the type-I error rate. Moreover, the 95\%
confidence intervals have reasonable converge. Finally, our test also achieves
power close to the oracle procedure.

	\begin{figure}[!t]
		\centering
		\includegraphics[width=1\linewidth, clip=TRUE, trim=0mm 10mm 0mm 0mm]{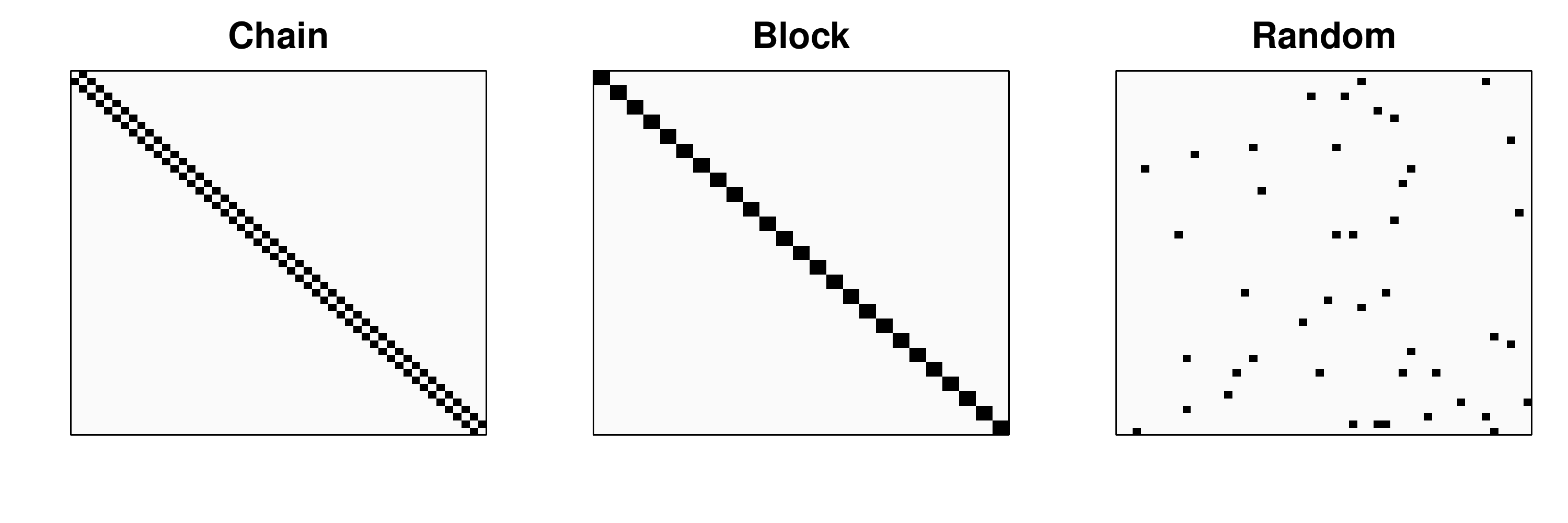}
		\caption{Connectivity matrices under chain, block and random graph structures. Zero coefficients are shown in gray and non-zero coefficients are shown in black.}
		\label{fig:simu_setting}
	\end{figure}

	\begin{figure}[!ht]
		\centering
		\includegraphics[width=1\linewidth]{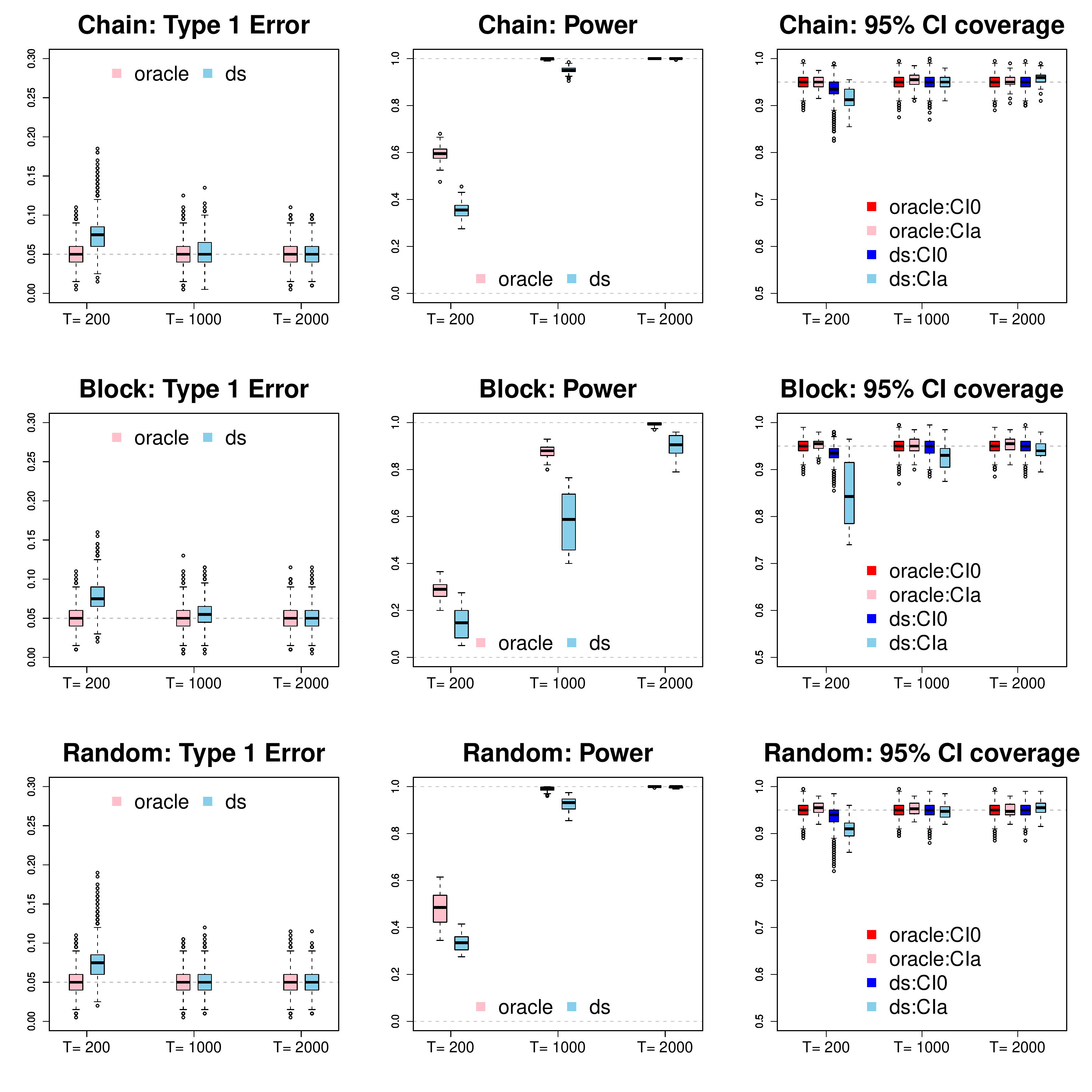}
		\caption{Type-I errors, powers and coverage of confidence intervals under chain, block and random graph structures.
		The \texttt{oracle} corresponds to the score test under the true model
		with known zero coefficients and \texttt{ds} corresponds to the de-correlated
		score test with nuisance coefficients.
		In the last column, CI0:95\% and CIa:95\% correspond to the coverage of
		confidence intervals for zero and non-zero coefficients, respectively.}
		\label{fig:hdi_result}
	\end{figure}

\section{Application}\label{sec:data}

We consider the task of learning the functional connectivity network among
population of neurons, using the spike train data from \citet{Boldingeaat6904}.
In this experiment, spike times are recorded at 30 kHz on a region of the mice
olfactory bulb (OB), while a laser pulse is applied directly on the OB cells of
the subject mouse. The laser pulse has been applied at increasing intensities
from 0 to 50 ($mW/mm^2$). The laser pulse at each intensity level lasts 10
seconds and is repeated 10 times on the same set of neuron cells of the subject
mouse. 

The experiment in total collects spike train data on 23 mice. We consider the spike train data collected at two intensity levels, 0 $mW/mm^2$ and 20 $mW/mm^2$, in the subject mouse with the most neurons (25 neurons). 
In particular, we use the spike train data from one laser pulse at each intensity level. Since one laser pulse spans 10 seconds and the spike train data is recorded at 30 kHz, there are 300,000 time points per experimental replicate. We apply our inference procedure separately for each intensity level, and obtain the estimated connectivity coefficients and the corresponding 95\% confidence interval for the 25-neuron network.

Figure~\ref{fig:dataexample} illustrates the estimated connectivity coefficients that are specific to each laser condition in a graph representation, where each node represents a neuron and a directed edge indicates a statistically significant estimated connectivity coefficient.
Compared with the control (0 $mW/mm^2$ laser) we see more condition-specific edges as laser is applied (at 20 $mW/mm^2$). This agrees with the observation by neuroscientists that the OB response is sensitive to the intensity level of the external stimuli \citep{Boldingeaat6904}.
Figure~\ref{fig:dataexample} also shows the 95\% confidence interval for 12 unique edges with largest estimated connectivity coefficients in one of the conditions. As expected, the confidence intervals corroborate with testing results and provide additional insight into differences in connectivity coefficients. 

\begin{figure}[t]
	\centering
	\includegraphics[width=1\linewidth, clip=TRUE, trim=0mm 0mm 0mm 0mm]{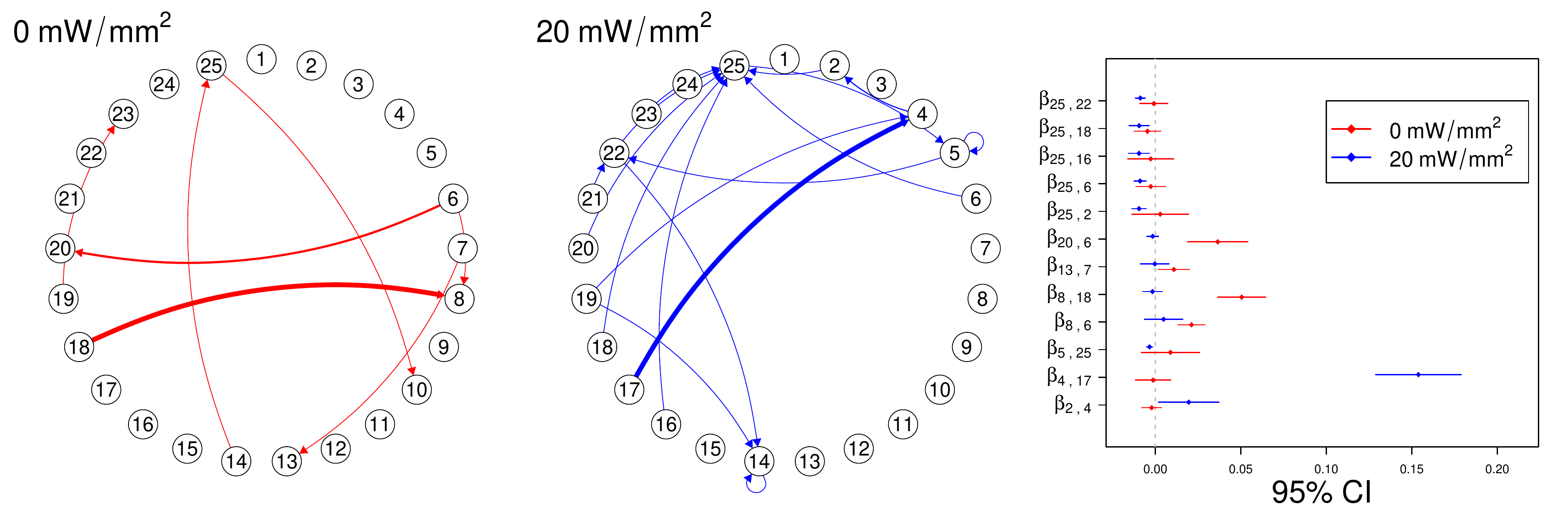}
	\caption{Estimated functional connectivities among neuronal populations
	using the spike train data from \citet{Boldingeaat6904}. In the condition-specific connectivity
	 graphs, 
	  red edges are unique to 0 $mW/mm^2$,
	 and blue edges are unique to 20 $mW/mm^2$. The last plot
	 shows 95\% confidence intervals for 12 unique edges with largest estimated
	  connectivity coefficients in one of the conditions. 
  }
		\label{fig:dataexample}
	\end{figure}

As discussed in Section~\ref{sec:theory}, our inference procedure is asymptotically valid. That means with large enough samples, if the other assumptions in Section~\ref{sec:theory} are satisfied, the type-I error should be controlled at the nominal level. Assessing the validity of the assumptions and estimating the type-I error in real data applications is challenging. However, we can verify the sample size requirement by estimating the type-I error rate in a `permuted’ data set where each neuron’s spike train is permuted. This permutation destroys both the connections between neurons and also the temporal dependence in each neuron. As a result, the neuronal connectivity network corresponding to this data set contains no edges. Moreover, some of the other assumptions in Section~\ref{sec:theory}  --- e.g. the sparsity of $\bm{w}^*_j$ and $\bm{\beta}_i$ and the structure of the transition matrix --- are trivially satisfied for this permuted data set. Thus, if the sample size is sufficient and the other assumptions are satisfied, we should not reject more than $\alpha=0.05$ of the tests. This is indeed the case: the total rejection rate is 0.32\%, suggesting that the conditions are likely satisfied.

\section{Discussion}\label{sec:disc}

We proposed a statistical inference procedure with theoretical guarantees for
high-dimensional linear Hawkes processes. To overcome the challenges arising
from the dependence of a Hawkes process on its entire history,  we develop a new
concentration inequality for the first- and second-order statistics of an 
integrated stochastic process; these integrated processes   summarize the entire
history for each component. We combine this new  concentration inequality with a
recent martingale central limit theorem,  to give an upper bounds for the
convergence rate of the test statistics.  We also provide a procedure for
constructing confidence intervals for the parameters.
Our results establish the first inferential framework for high-dimensional point
processes.

We considered a parametric transition function for the Hawkes process. Given the
complex nature of a point process, one may consider modeling  the transition
function nonparametrically and learn the form  adaptively from data. In
addition, since non-linear link functions are often used when analyzing spike
train data \citep{PANINSKI2007, Pillow2008},  it would also be of interest to
develop statistical inference procedure for non-linear Hawkes processes.

\bibliographystyle{rss.bst}
\bibliography{ref.bib}

\clearpage
\section*{Appendix~A: Proof of Main Results}
\label{sec:proof_outline_detail}


Before presenting the formal proof of Theorems~\ref{theorem1}--~\ref{theorem3}, we outline the key technical steps in this section.
We also briefly discuss key steps in the proof of lemmas used to prove the theorems, as well as auxiliary lemmas, which are presented in Appendix~B.
Recall that $S_{ij}$ is the de-correlated score defined in \eqref{eq:S}, $\Upsilon_j$, defined in \eqref{eq:Upsilon}, is the covariance of $z^*_j(t)$, where $z^*_j(t)$ is the scaled version of the design column $x_j(t)$ after removing its projection onto the other columns. Here $x_j(t)$, defined in \eqref{eq:design_column_xt}, is an integrated stochastic process that summarizes the past events of the $j$th feature of the multivariate Hawkes process. For theoretical convenience, when calculating $S_{ij}$, we scale $x_j(t)$ by the standard deviation of the $i$th process at time $t$, $\sigma_i(t)$ (see details in Section \ref{sec:test}).

\paragraph{Theorem~\ref{theorem1}:}
This theorem establishes the convergence of the test statistics $\widehat{U}_T$ \eqref{eq:U_hat_T} to a $\chi^2$-distribution under the null hypothesis. While the result is comparable to that in \citet{Neykov2018unified} and \citet{Zheng2018}, in our case $\widehat{U}_T$ is a function of the estimate of the time-varying variance of the point process, $\widehat{\sigma}^2(t)$. Proving the convergence in this case is different and requires additional care. To this end, we (i) show the convergence in probability of $\widehat{U}_T$ to the test statistic $\widehat{U}^0_T$, which is defined similar to $\widehat{U}_T$ but with $\widehat{\sigma}^2(t)$ replaced with the true $\sigma^2(t)$; (ii) show that $\widehat{U}^0_T$ converges in probability to $U_T$; and (iii) establish that $U_T$ weakly converges to a $\chi^2$-distribution.
%
Next, we provide some details on each of these steps.

To show the weak convergence of $U_T$ to a $\chi^2$-distribution under the null hypothesis, we adopt the recently developed martingale central limit theorem (CLT) \citep{Grama2006, Zheng2018}, which is given as a special case of Proposition~\ref{prop_clt} (for $\Delta = 0$).

To show the convergence of $\widehat{U}^0_T$ to $U_T$, we expand the difference in terms of differences between $\widehat{S}^0_{ij}$ and $S_{ij}$, and $\widehat{\Upsilon}^0_j$ and $\Upsilon_j$, and bound each term. Here, $\widehat{S}^0_{ij}$ and $\widehat{\Upsilon}^0_j$ are estimates of $S_{ij}$ and $\Upsilon_j$ but with the true $\sigma^2(t)$.
By the construction of the decorrelated score \citep{Neykov2018unified}, bounding the difference between $\widehat{S}^0_{ij}$ and $S_{ij}$ is equivalent to evaluating the estimation error of the lasso estimators of $\left(\mu_i,\bm{\beta}_i \right)$ and $\bm{w}^*_j$. Bounds on these estimation errors are given in Lemmas~\ref{Lemma_beta_lasso} and~\ref{Lemma_w_lasso}, respectively, in both $\ell_1$- and $\ell_2$-norms.
To establish these bounds, we show that the restricted eigenvalue condition \citep{bickel2009} is satisfied with high probability in our case, by using the concentration bounds for the first and second order statistics of $\bm{x}(t)$ (shown in Lemma~\ref{lemma_concen_ineq_x}). We also show that the eigenvalue of the covariance matrix of $\bm{x}(t)$ are bounded under Assumptions~\ref{assumption1}--\ref{assumption4} (shown in Proposition~\ref{prop_eigen}).
To bound the difference between $\widehat{S}^0_{ij}$ and $S_{ij}$, we need to bound the $\ell_\infty$ norms of (i) the first and the second order statistics of $z^*_{j}(t)$, and (ii) the average of the scaled version of the residual $\epsilon_i(t)$ and its inner product with the scaled integrated feature, $z^*_j(t)$.
The bounds for (i) are presented in Lemma~\ref{lemma_concen_ineq_z} using the concentration bound for the first and second order statistics of $\bm{x}(t)$ (shown in Lemma~\ref{lemma_concen_ineq_x}). The bounds for (ii) are given in Lemma~\ref{lemma_concen_ineq_z} by a direct application of a martingale inequality \citep{vandegeer1995} using the fact that $x_j(t), \sigma_i(t)$ are bounded under Assumptions~\ref{assumption3} and~\ref{assumption4}.
The proof for the convergence of $\widehat{\Upsilon}^0_j$ to $\Upsilon_j$, which is shown in Lemma~\ref{lemma_Rhat_R0_R}, is similar to that for the convergence of $\widehat{S}^0_{ij}$ to $S_{ij}$; however, we only need the estimation error bound for $\bm{w}_j$ as the construction of $\Upsilon_j$ only involves the integrated process (or the design columns) $\bm{x}(t)$.

To complete the proof of Theorem~\ref{theorem1}, we establish the convergence of $\widehat{U}_T$ to $\widehat{U}^0_T$ in Lemma~\ref{lemma_U0_U}.
The proof of the lemma first expands the difference between $\widehat{U}_T$ and $\widehat{U}^0_T$ in terms of differences between (i) $\widehat{S}^0_{ij}$ to $\widehat{S}_{ij}$ \eqref{eq:S_hat}, and (ii) $\widehat{\Upsilon}^0_j$ and $\widehat{\Upsilon}_j$ \eqref{eq:upsilon_hat}. (ii) can be bounded using Lemma~\ref{lemma_Rhat_R0_R}. To quantify the bound of (i), we first replace $\sigma^2(t)$ by $\mu_i + \bm{x}^\top(t) \bm{\beta}_i$ in $\widehat{S}^0_{ij}$ and then bound the term by carefully applying the bound for convergence of $(\widehat{\mu}_i,\widehat{\bm{\beta}}_i)$ to $(\mu_i, \bm{\beta}_i)$ by Assumption~\ref{assumption_beta} which is also proved to be satisfied in Lemma~\ref{Lemma_beta_lasso}.

%
%
\paragraph{Theorem~\ref{theorem2}:}
This theorem characterizes the behavior of the test statistics $\widehat{U}_T = \lVert  \widehat{V}_T \rVert_2^2$ \eqref{eq:U_hat_T} under the alternative hypothesis. More specifically, it shows that
for different signal strengths characterized by parameter $\phi$, the test statistic behaves differently around the cutoff of $\phi = 1/2$. This is because (i) $\widehat{U}_T$ converges to $U_T= \lVert  V_T \rVert_2^2$ defined in \eqref{eq:U_T} in probability (see the following paragraphs for the proof outline), and (ii) under the alternative, we have $\mathbb{E}(V_T) = O(T^{1/2-\phi}\Delta)$.
If the alternative signal is too small with $\phi > 1/2$, the expectation of $V_T$ converges to 0 as $T$ increases. In this case, we can show that $\widehat{ U}_T$ converges weakly to a central $\chi^2$ distribution, and is hence indistinguishable from null. In contrast, if the signal strength is too large with $\phi <1/2$, then $V_T$ diverges as $T$ goes to $\infty$. Finally, when the alternative signal is moderate with $\phi = 1/2$, the expectation of $V_T$ is a constant, $\widetilde{\Delta} = \Upsilon^{1/2}_j\Delta$. In this case, we apply Proposition~\ref{prop_clt} to show the weak convergence of $\widehat{U}_T$ to a non-central $\chi^2$ distribution with non-centrality parameter $\lVert \widetilde{\Delta} \rVert^2_2$.

For the case of $\phi > 1/2$, we use a similar strategy as that for the proof of Theorem~\ref{theorem1} under the null hypothesis. More specifically, we split the proof into two parts by bounding $\widehat{U}_T - \widehat{U}^0_T$ and $\widehat{U}^0_T - U_T$. $U_T$ is shown to follow a central $\chi^2$ distribution in Proposition~\ref{prop_clt}. The bound of $\widehat{U}_T - \widehat{U}^0_T$ is given in Lemma~\ref{lemma_U0_U}. Similar to in the proof of Theorem~\ref{theorem1}, the key in bounding $\widehat{ U}^0_T - U_T$ under the alternative hypothesis is also to bound $\sqrt{T} \left\lVert \left( \Upsilon_j^0 \right)^{-1/2}(\widehat{S}^0_{ij} - S_{ij} ) \right\rVert_2$. The difference from the proof of Theorem~\ref{theorem1} is an extra term involving $T^{1/2-\phi}\Delta$ in the bound of $\lVert \widehat{S}^0_{ij} - S_{ij} \rVert_2$ because under the alternative hypothesis, $\beta_{ij} = T^{-\phi}\Delta$.
Such difference leads to an extra term of order $ O\left( s^2 \rho^2 \left( T^{ \frac{7}{5} - 2\phi } \right)  \right)$ in bounding $\widehat{ U}^0_T - U_T$.
As a result, we reach a modified rate of weak convergence of $\widehat{ U}_T$ in this case compared to the rate in Theorem~\ref{theorem1}.

For the case of $\phi = 1/2$, our proof uses a strategy similar to the proof of Theorem~\ref{theorem1}. Specifically, we split the proof by bounding $\widehat{U}^0_T$ to $\widehat{U}_T$, and $\widehat{U}^0_T$ to $ V_T + \widetilde{\Delta}$. The proof of the first part is the same as that for Theorem~\ref{theorem1} and is given in Lemma~\ref{lemma_U0_U}. Although the second part involves non-zero $\widetilde{\Delta}$, the proof strategy is similar and involves explicitly expressing the difference between $\widehat{U}^0_T$ and $ V_T + \widetilde{\Delta}$ in terms of the difference between $\widehat{S}^0_{ij}$ and $S_{ij}$, the difference between $\widehat{\Upsilon}^0_j$ and $\widehat{\Upsilon}_j$, and an extra item involving $\Delta$. The first two items are bounded using the same strategy as in Theorem~\ref{theorem1}. The term involving $\Delta$ is also bounded since the leading term involves the $\ell_\infty$ bound of
the first and the second order statistics of $\bm{z}(t)$,
which are bounded by Lemma~\ref{lemma_concen_ineq_z}. We complete the proof by applying Proposition~\ref{prop_clt} to show that $V_T + \widetilde{\Delta}$ converges weakly to $\chi^2_{d, \lVert \widetilde{\Delta} \rVert^2_2}$.

For the case of $ \phi < 1/2$, the test statistic diverges in probability as $T$ goes to $\infty$. Therefore, we derive a lower bound for $\widehat{U}_T$ which requires quantifying the lower bound of $\widehat{S}_{ij} - S_{ij}$ and the upper bound of $\lVert V_T \rVert_2$.
Due to the estimate on the unknown variance involved in $\widehat{S}_{ij}$, we split the difference in the first part into (i) $\widehat{S}^0_{ij} - S_{ij}$ and (ii) $\widehat{S}_{ij} - \widehat{S}^0_{ij}$. We show that the lower bound of part (i) is in order of $T^{1/2-\phi}$ under the setting of the alternative of $\phi < 1/2$.  To bound part (ii), similar as before, we expand the term into components involving differences between the estimates of $(u_i,\bm{\beta}_i)$ and $\bm{w}^*_j$, whose error bounds are assumed in Assumptions~\ref{assumption_beta} and~\ref{assumption_w}, and parts with the first and second order statistics of $\bm{x}(t)$ and $\bm{z}(t)$ given in Lemmas~\ref{lemma_concen_ineq_x} and~\ref{lemma_concen_ineq_z}. To quantify the upper bound of $\lVert V_T \rVert_2$, we apply a result on the tail bound of the central $\chi^2$ distribution since $\lVert  V_T \rVert^2_2$ follows a central $\chi^2$-distribution by Proposition~\ref{prop_clt}; this bound is also used to prove Theorem~3.2 of \citet{Zheng2018}.
%
%

\paragraph{Theorem~\ref{theorem3}:}
Recall from \eqref{eq:one_step_est_bij} that $\widehat{b}$ is the one-step debiased lasso estimates of $\bm{\beta}_i$. This theorem shows that $\widehat{R}=  T \left(\widehat{b}_{ij} - \beta_{ij}\right)^\top  \widehat{\Upsilon}_j \left(\widehat{b}_{ij} - \beta_{ij}\right)$ converges weakly to a central $\chi^2$ distribution. This allows us to construct the optimal confidence regions in \eqref{eq:CI}, since the asymptotic variance of $\widehat{R}$ is close to the inverse of the partial information, $\Upsilon_j=  \textrm{Cov}\left (z^*_j(t) \right ) \eqref{eq:Upsilon}$.

For this proof, we introduce a new quantity $\check{S}_{ij}$:
\begin{align*}
\check{S}_{ij} &=\frac{1}{T} \sum_{t=1}^{T} \frac{ Y_i(t) - x_j(t)\beta_{i,j} - \widehat{\mu}_i - \bm{x}^\top_{-j}(t) \widehat{\bm{\beta}}_{i,-j}  }{\widehat{\sigma}_i(t)  } \widehat{ z}^*_j(t)  \\
&=\frac{1}{T} \sum_{t=1}^{T} \frac{ \epsilon_i(t) + ( \mu_i - \widehat{\mu}_i)
	+\bm{x}^\top_{-j}(t)  (\bm{\beta}_{i,-j} - \widehat{\bm{\beta}}_{i,-j} )  }{\widehat{\sigma}_i(t)  } \widehat{ z}^*_j(t) ,
\end{align*}
which is equivalent to the $\widehat{S}_{ij}$ defined in \eqref{eq:S_hat} under the null, and similar to that under the alternative, except that here $\Delta =0$ or $\phi = \infty$.
Letting $\check{U}_T = \check{S}_{ij}^\top \widehat{\Upsilon}_j^{-1} \check{S}_{ij}$, we then show that $\check{U}_T$ weakly converges to a $\chi^2$ distribution under both  null and alternative hypotheses. For the null, we follow Theorem~\ref{theorem1}, while for the alternative, we repeat the steps in Theorem~~\ref{theorem2} (case $\phi > 1/2$) but replace $\Delta = 0$ or $\phi=\infty$ throughout.

Since by the construction of $\widehat{R}_T$, we have $\widehat{R}_T = \left( \check{S}_{ij} \right)^\top  \left(\widetilde{\Upsilon}_j   \right)^{-1} \widehat{\Upsilon}_j \left(\widetilde{\Upsilon}_j   \right)^{-1} \check{S}_{ij}$, what is left is to bound the difference between $\widehat{R}_T$ and $\check{U}_T$. This requires to bound $ \widetilde{\Upsilon}_j - \widehat{\Upsilon}_j$ where $\widetilde{\Upsilon}_j$ and $\widehat{\Upsilon}_j$ are defined in \eqref{eq:upsilon_hat_tilde} and \eqref{eq:upsilon_hat}, respectively. To obtain such a bound with the unknown variance, we define $\widetilde{\Upsilon}^0_j$ and $\widehat{\Upsilon}^0_j$ similar to $\widetilde{\Upsilon}_j$ and $\widehat{\Upsilon}_j$, but with $\widehat{\sigma}^2$ replace by the true $\sigma^2(t)$. We then separately bound (i) $\widetilde{\Upsilon}_j -
\widetilde{\Upsilon}^0_j$, (ii) $\widetilde{\Upsilon}^0_j- \widehat{\Upsilon}^0_j$, and (iii) $\widehat{\Upsilon}^0_j - \widehat{\Upsilon}_j$: (i) is bounded according to the consistency of $\widehat{\sigma}^2(t)$ for the true $\sigma^2(t)$ following a similar strategy used in Lemma~\ref{lemma_concen_ineq_z}; (ii) is bounded by carefully evaluating the lasso estimation error of $\bm{w}^*_j$; (iii) is bounded by Lemma~\ref{lemma_concen_ineq_z}.


\paragraph{Key steps in proof of lemmas:}
Proofs of supporting lemmas crucially rely on properties of the integrated stochastic process $\bm{x}(t)=\mat{x_1(t)&\cdots&x_p(t)}$, which summarizes all the past event history of the Hawks process. More specifically, the main theorems rely on the bounded eigenvalue of $\Upsilon_x= \textrm{Cov}(\bm{x}(t))$ (Proposition~\ref{prop_eigen}) and the concentration bounds on the first and second order statistics of scaled $\bm{x}(t)$ (i.e. $\bm{z}(t)$) (Lemmas~\ref{lemma_concen_ineq_x} and~\ref{lemma_concen_ineq_z}). These results are used to show the condition of the martingale CLT (Proposition~\ref{prop_clt}), the restrict eigenvalue (RE) condition used in the estimation consistency of lasso (Lemma~\ref{Lemma_beta_lasso} and~\ref{Lemma_w_lasso}), and the convergence rate of $\widehat{\Upsilon}_j$ (Lemma~\ref{lemma_Rhat_R0_R}).

A key challenge in establishing results for the integrated process stems from the complicated (and non-Markovian) dependence structure of the Hawkes process. In particular, each column of $\bm{x}(t)$, $x_j(t)$, is a stochastic process with non-trivial serial dependence due to the integration over the past history.
To show that the eigenvalues of $\Upsilon_x= \textrm{Cov}(\bm{x}(t))$ are bounded, we show that the eigenvalues of the cross-variance for a stationary stochastic process can be bounded by its spectral density in the Hawkes process. The proof follows a similar strategy as \citet{Basu2015} in the VAR models, but is specialized for the Hawkes process in a continuous time domain with an integrable transition kernel. Next, we establish a relationship between the spectral densities of the cross-covariance and the transition matrix of the Hawkes process by generalizing Theorem~1 in \citet{Bacry2011} or Theorem~3 in \citet{Etesami2016} (where they assume a non-negative transfer function) to real-value transfer functions.
At last, we utilize the martingale inequality in \citet{vandegeer1995} and the concentration inequality on weak dependent samples in \citet{Shizhe2017} to establish the concentration bounds on the first and second order statistics of $\bm{x}(t)$ in Lemma~\ref{lemma_concen_ineq_x}.
%
The final concentration bound for $\bm{z}(t)$ in Lemma~\ref{lemma_concen_ineq_z} directly follows from the deviation bounds for $\bm{x}(t)$ but requires a lengthy derivation, due to the scaling factor of the time-varying variance of the point process, $\sigma^2(t)$, in $\bm{z}(t)$.


For ease of notations, we prove the result for testing an univariate
$\beta_{ij}$; i.e., $d=1$. Our proof can be extended to $d>1$ by replacing the
scalars by vectors or matrices in the corresponding norms when needed. In the
following, we use $C_k$, $c_k$ with some subscript $k$ to represent constants
that only depend on the model parameter $(\bm{\mu},\Theta)$ and the
transition kernel function.

\subsection*{Proof of Theorem~\ref{theorem1}}

Our proof strategy is to relate $\widehat{U}_T$ in \eqref{eq:U_hat_T} to $U_T$
in \eqref{eq:U_T}, and show that $U_T$ converges weakly to $\chi_d^2$.
Directly comparing $\widehat{U}_T$ with $U_T$ is difficult due to the unknown
time-varying variance $\sigma_i^2(t)$ involved. Therefore, we start by introducing
$\widehat{U}^0_T$ and then show that $\widehat{U}_T \approx \widehat{U}^0_T$.
Let
\begin{align}
  \widehat{S}^0_{ij}
  &= \frac{1}{T}\sum_{t=1}^{T}  \frac{1}{\sigma_i(t)} \left ( Y_i(t) - \widehat{\mu}_i - \bm{x}^\top_{-j}(t)\widehat{ \bm{\beta}}_{i,-j} \right ) \left (x_j(t)/\sigma_i(t)- \widehat{w}_{j0} - \bm{x}^\top_{-j}(t)/\sigma_i(t) \widehat{\bm{w}}_{j,-j}  \right ),
    \label{eq:S_0_hat} \\
  \widehat{\Upsilon}_j^0
  &=  \frac{1}{T} \sum_{t=1}^T
    \left (x_j(t)/\sigma_i(t)- \widehat{w}_{j0} - \bm{x}^\top_{-j}(t)/\sigma_i(t) \widehat{\bm{w}}_{j,-j}  \right )^2,
    \label{eq:Upsilon_hat_0} \\
  \widehat{V}^0_T
  &= \sqrt{T} \left ( \widehat{\Upsilon}^{0}_j\right )^{-1/2} \widehat{S}^0_{ij}
    \label{eq:V_0_hat} , \\
  \widehat{U}^0_T &= \lVert  \widehat{V}^0_T \rVert_2^2 \label{eq:U_0_hat}.
\end{align}
The difference between
$\widehat{U}_T$ and $\widehat{U}_T^0$
is that we replace $\widehat{ \sigma}_i(t) $ by $\sigma_i(t) $.
%
%

For any  $\delta  > 0$, we have
\begin{multline*}
\mathbb{P}\left( \widehat{U}_T \le x\right) - F_d(x)
\le \mathbb{P}( U_T \le x + \delta) + \mathbb{P}\left(\left| \widehat{U}_T - \widehat{U}^0_T \right|> \delta \right)
- F_d(x) \\
\le \left|  \mathbb{P}(U_T \le x + \delta ) - F_d (x+\delta)  \right|
+ F_d(x+ \delta) - F_d(x) + \mathbb{P}\left (\left|\widehat{U}_T - \widehat{U}^0_T  \right|> \delta \right )  ,
\end{multline*}
and
\begin{multline*}
F_d(x) - \mathbb{P}\left ( \widehat{U}_T \le x \right )
\le  \mathbb{P}( \widehat{U}^0_T  > x - \delta) + \mathbb{P}\left (\left|\widehat{U}_T - \widehat{U}^0_T \right|> \delta \right )
- (1-F_d(x)  ) \\
\le \left|  F_d (x-\delta) -  \mathbb{P}(U_T \le x - \delta )   \right|
+ F_d(x) - F_d(x- \delta) + \mathbb{P}\left (\left| \widehat{U}_T - \widehat{U}^0_T  \right|> \delta \right ).
\end{multline*}
Combining the two inequalities gives 
\begin{multline}
\left |\mathbb{P}\left (\widehat{U}_T \le x \right ) - F_d(x) \right |  \le
\sup_{y\in \mathbb{R} } \left |\mathbb{P}\left ( \widehat{U}^0_T \le y \right ) - F_d(y) \right |  \\
+   F_d(x+ \delta ) - F_d(x -\delta)
+  \mathbb{P}\left ( \left |\widehat{U}_T - \widehat{U}^0_T \right | > \delta \right ) .
\label{eq:bound_U_T_Fd}
\end{multline}
Next, we bound
\[
  \left | \mathbb{P}\left(\widehat{U}^0_T \le x\right) - F_d(x) \right|
\]
by showing  $\widehat{U}^0_T \approx U_T \approx \chi^2_d$.
Following a similar deduction as \eqref{eq:bound_U_T_Fd}, for any  $\epsilon > 0$, we have
\begin{multline}
  \left | \mathbb{P}(\widehat{U}^0_T \le x) - F_d(x) \right |
  \le
  \underbrace{ \sup_{y\in \mathbb{R} } | \mathbb{P}( U_T \le y) - F_d(y)  | }_{A} \\
  + \underbrace{ F_d(x+ \epsilon ) - F_d(x -\epsilon) }_{B}
+\underbrace{   \mathbb{P}( |\widehat{U}^0_T - U_T| > \epsilon ) }_{C}.
\label{eq:bound_U0_T_Fd}
\end{multline}
Direct application of Proposition~\ref{prop_clt} shows that
$\text{A} = O\left(  T^{-1/8} \right) $.
Using the fact that a $\chi_d^2$ random variable has bounded density
gives $|\text{B}| =  O\left( \epsilon\right) $.
Thus, we control the term C in the rest of the proof. Notice that
\begin{align}
\left | \widehat{U}^0_T - U_T \right |
=& \left|  T  \left ( \widehat{S}^0_{ij}  \right )^\top  \left ( \widehat{\Upsilon}^0_j \right )^{-1} \widehat{S}^0_{ij}  -
S^T_{ij} \Upsilon_j^{-1} S_{ij} \right| \nonumber \\
\le& \left|  T  \left ( \widehat{S}^0_{ij}  \right )^\top  \left ( \left (\widehat{\Upsilon}^0_j \right )^{-1}  - \Upsilon_j^{-1} \right )\widehat{S}^0_{ij}
+ T   \left ( \widehat{S}^0_{ij}  \right )^\top   \Upsilon_j^{-1} \widehat{S}_{ij}
- S^T_{ij} \Upsilon_j^{-1} S_{ij} \right| \nonumber \\
\le&
\left \lVert  \Upsilon_j^{1/2} \left (\widehat{\Upsilon}^0_j \right )^{-1}   \Upsilon_j^{1/2} - I \right \rVert_\infty
\left \lVert  \sqrt{T} \Upsilon_j^{-1/2} \widehat{S}^0_{ij} \right \rVert_1^2 \nonumber \\
&+   \lVert  \sqrt{T} \Upsilon_j^{-1/2} (S_{ij}- \widehat{S}^0_{ij} ) \rVert_1^2 \nonumber \\
&+ 2 \lVert V_T \rVert_2  \left \lVert  \sqrt{T} \Upsilon_j^{-1/2} (S_{ij}- \widehat{S}^0_{ij} ) \right \rVert_2.
\label{eq: U_T0_UT}
\end{align}
Let $E = \sqrt{T} \Upsilon_j^{-1/2} (S_{ij}- \widehat{S}^0_{ij} )$,
where $\Upsilon_j$ is defined in \eqref{eq:Upsilon}.
 Then
\begin{align}
| \widehat{U}^0_T - U_T |  &\le
\lVert E \rVert_2^2 +  2 \lVert V_T \rVert_2 \lVert E \rVert_2 +
\left \lVert \Upsilon_j^{1/2} \left( \widehat{\Upsilon}^0_j \right)^{-1}  \Upsilon_j^{1/2} - I  \right \rVert_\infty
\left ( \lVert V_T \rVert_2 + \lVert E \rVert_2 \right )^2 .
\label{eq:bound_U0_UT_null}
\end{align}
Next, we provide probabilistic bounds for $\lVert E \rVert_2$, $\lVert V_T \rVert_2$,
and
$\left \lVert \Upsilon_j^{1/2} \left( \widehat{\Upsilon}^0_j \right)^{-1}   \Upsilon_j^{1/2} - I \right  \rVert_\infty$
in order to bound $| \widehat{U}^0_T - U_T |$.
First, from Proposition~\ref{prop_eigen} and
Lemma~\ref{lemma_S0_S},
we get
$ \lVert E \rVert_2^2  = O_p \left(  s^2\rho^4  T^{-3/5} \right) $, with probability at least $1- c_1 p^2 T \exp( - c_2 T^{1/5})$.
Second, Lemma~\ref{lemma_VT_ineq} and Proposition~\ref{prop_clt} lead
$ \mathbb{P}(\lVert V_T \rVert_2 > T^{1/10} ) = O\left( T^{-1/8}   \right) . $
Third, Lemma~\ref{lemma_Rhat_R0_R} gives
\begin{align*}
\left \lVert  \Upsilon_j^{1/2} \left (\widehat{\Upsilon}^0_j\right )^{-1}  \Upsilon_j^{1/2} - I \right \rVert_\infty = O_p( s^2\rho^2 T^{-2/5} ) ,
\end{align*}
with probability at least $1- c_3 p^2 T \exp( - c_4 T^{1/5})$.
Therefore,
\begin{align}
\mathbb{P}\left( | \widehat{U}^0_T - U_T |  > c_5  s^2 \rho^2 T^{-1/5} \right)
\le c_6 p^2 T \exp( - c_7 T^{1/5}) + c_8 T^{-1/8}  .
\end{align}
Turning back to \eqref{eq:bound_U0_T_Fd} with the bounds of $A$, $B$, and $C$ gives us
\begin{align}
\sup_{x \in \mathbb{R} } | \mathbb{P}(\widehat{U}^0_T \le x) - F_d(x)  | \le
c_6 p^2 T \exp( - c_7 T^{1/5}) + c_8 T^{-1/8} + c_9  s^2 \rho^2 T^{-1/5} .
\label{eq:U_hat_0_to_chisq}
\end{align}
By the bounded density of $\chi^2$ distribution
\[
  \left| F_d(x+ \delta ) - F_d(x -\delta) \right | \le C(d) \delta.
\]
Finally, Lemma~\ref{lemma_U0_U} gives us
\begin{align*}
\mathbb{P}\left( | \widehat{U}_T - \widehat{U}^0_T |  > c_{10}  \rho T^{-1/5} \right)
\le c_{11} p^2 T \exp( - c_{12} T^{1/5}) + c_{13} T^{-1/8} + c_{14} s^2\rho^2 T^{-1/5} .
\end{align*}
Combining the last three displays gives us a bound on
\eqref{eq:bound_U_T_Fd}, which completes the proof.


\subsection*{Proof of Theorem~\ref{theorem2}}

%

We study the three cases separately.

\noindent
\textbf{Case}:  $\phi > 1/2$. The proof for this case is similar to the proof of
Theorem~\ref{theorem1}. However, due to $\beta_{ij}\ne 0$ under the alternative
hypothesis, we have a modified rate of convergence depending on the scale of $\phi$.
A new bound is needed for  $\left |  \widehat{U}^0_T - U_T \right |$, which is
obtained modifying the bound for the difference between $\widehat{S}^0_{ij}$ and $S_{ij} $.
To be specific,
\begin{align}
\widehat{S}^0_{ij} - S_{ij} &=
( \widehat{\bm{w} }_j  - \bm{w}^*_j )^\top  \frac{1}{T}
\sum_{t=0}^{T-1} \widetilde{\epsilon}_i(t)  \bm{z}^\top _{-j}(t) \nonumber \\
& +  \frac{1}{T} \sum_{t=0}^{T-1} z^*_j
\mat{1 & \bm{z}_{-j}^\top(t) }
\left ( \mat{ \widehat{\mu}_{i} \\ \widehat{\bm{\beta}}_{i,-j} } - \mat{ \mu_{i} \\ \bm{\beta}_{i,-j} } \right )
\nonumber   \\
& - ( \widehat{\bm{w} }_j  - \bm{w}^*_j )^\top
\left (   \frac{1}{T} \sum_{t=0}^{T-1}  \mat{1 \\ \bm{z}_{-j}(t) } \mat{1 & \bm{z}_{-j}^\top(t)  }
\right )
\left ( \mat{ \widehat{\mu}_{i} \\ \widehat{\bm{\beta}}_{i,-j} } - \mat{ \mu_{i} \\ \bm{\beta}_{i,-j} } \right )
\nonumber \\
& - T^{-\phi} \frac{1}{T}\sum_{t=1}^T  \left ( z_j(t) -   \mat{1 & \bm{z}^\top(t) } \widehat{\bm{w} }_j  \right ) z_j(t) \Delta .
\label{eq:th2_S0_S_diff}
\end{align}
The first three terms are bounded by $O_p(s \rho^2 T^{-4/5})$, as shown in the proof of Lemma~\ref{lemma_S0_S} (see \eqref{eq:theorem1_bound_S0_S}).
The last term shows up because under the alternative hypothesis setting,
$\beta_{ij} = T^{-\phi}\Delta \ne 0$.
Now we examine the fourth item in \eqref{eq:th2_S0_S_diff}. By Lemma~\ref{lemma_bounded_terms}, $ \lVert z_j(t) \rVert_\infty =O(1)$. Then,
$$
\left \lVert
\frac{1}{T}\sum_{t=1}^T  \left( z_j(t) -   \mat{1 & \bm{z}^\top(t) } \widehat{\bm{w} }_j  \right)
z_j(t)  \Delta
\right\rVert_2 \le |\Delta| \lVert z_j(t) \rVert^2_\infty = O(1).
$$
Using the same notation as the proof in Theorem~\ref{theorem1}, let
$ E  = \sqrt{T}\left( \Upsilon_j \right)^{-1/2}( \widehat{S}^0_{ij} - S_{ij}) $.
Then, with the bounds in each part of \eqref{eq:th2_S0_S_diff},
\begin{align}
|E | = O_p \left(   s \rho^2T^{-3/10} + T^{1/2-\phi} \right)  ,
 \label{eq:S0_S_phigt0.5}
\end{align}
with probability at least $1- c_1p^2 T \exp(-c_2 T^{1/5})$.
Here we see that $\phi > 1/2$ is important to bound the term;
otherwise, it is not guaranteed that the term is bounded as $T$ increases.

By repeating the same steps in Theorem~\ref{theorem1} with the bound of $|E|$,
we show
\begin{align*}
\left| \widehat{U}^0_T  - U_T \right|  = O_p\left( s^2 \rho^2 \left( T^{-1/5} \vee T^{ \frac{7}{5} - 2\phi } \right)  \right),
\end{align*}
with probability at least
$$
1 - c_3 p^2 T \exp(- c_4 T^{1/5}) - c_5 s^2 \rho^2 \left( T^{-1/5} \vee T^{ \frac{7}{5} - 2\phi } \right)
$$
and
\begin{align*}
\sup_{x\in \mathbb{R}} |   \mathbb{P}(\widehat{U}^0_T \le x) - F_d(x)  |
\le c_3 p^2 T \exp(- c_4 T^{1/5}) + c_5 s^2 \rho^2 \left( T^{-1/5} \vee T^{ \frac{7}{5} - 2\phi } \right) + c_6 T^{-1/8} .
\end{align*}
At last, using \eqref{eq:bound_U_T_Fd} and Lemma~\ref{lemma_U0_U} to bound $|\widehat{U}_T - \widehat{U}^0_T |$, we reach the conclusion.


\noindent
\textbf{Case}: $\phi = 1/2$. The proof strategy here is
to quantify the difference between the cdf of $\widehat{U}^0_T$
and $\chi^2$ distribution.
For any $\epsilon > 0$, we have
\begin{multline*}
\left |   \mathbb{P}(\widehat{U}^0_T \le x) - F_{d, \lVert \widetilde{\Delta} \rVert_2^2 }(x) \right |
\le
\underbrace{ \sup_{y\in \mathbb{R} }
	\left |   \mathbb{P}( \lVert V_T + \widetilde{\Delta} \rVert_2^2 \le y) -  F_{d, \lVert \widetilde{\Delta}\rVert_2^2  }(y) \right | }_{A} \nonumber \\
+ \underbrace{  F_{d, \lVert \widetilde{\Delta} \rVert_2^2 } (x+ \epsilon ) -  F_{d, \lVert \widetilde{\Delta}\rVert_2^2  } (x -\epsilon) }_{B}
+\underbrace{    \mathbb{P}\left ( \left |\widehat{U}^0_T -  \lVert V_T + \widetilde{\Delta} \rVert_2^2 \right| > \epsilon \right ) }_{C} ,
\end{multline*}
where $\widetilde{\Delta} = \Upsilon^{1/2}_j\Delta$, and $\widehat{ U}^0_T$, $V_T$ are  defined in \eqref{eq:U_0_hat} and \eqref{eq:V_T}, respectively.

Lemma~\ref{prop_clt} gives us
$A= O \left( T^{-1/8} \right)$.
By the bounded density of non-central $\chi^2$ distribution, we have
$B = O\left( \epsilon \right)$.
Therefore, in the rest of the proof, we bound part $C$.

Let $E = \widehat{V}^0_T  - V_T - \widetilde{\Delta}$,
where $\widehat{V}^0_T = \sqrt{T} \left(\Upsilon_j \right)^{-1/2} \widehat{S}^0_{ij}$
is defined in \eqref{eq:V_0_hat} and $\widehat{ U}^0_T = \lVert  \widehat{V}^0_T \rVert^2_2$.
Then,
\begin{multline}
\left| \widehat{U}^0_T - \left\lVert V_T + \widetilde{\Delta} \right\rVert_2^2  \right  |
\le
\lVert E \rVert_2^2 +
\left\lVert V_T + \widetilde{\Delta} \right \rVert_2 \lVert E \rVert_2  \\
+
\left \lVert \left( \Upsilon_j \right)^{1/2} \left( \widehat{\Upsilon}^0_j \right)^{-1}  \left( \Upsilon_j\right) ^{1/2} - I \right \rVert_\infty
\left( \left\lVert V_T + \widetilde{\Delta} \right\rVert_2
+   \lVert E \rVert_2 \right)^2.
\label{eq:U0_hat_VT_phi0.5}
\end{multline}
We bound
$\left \lVert \left( \Upsilon_j \right)^{1/2} \left( \widehat{\Upsilon}^0_j \right)^{-1}  \left( \Upsilon_j\right) ^{1/2} - I \right \rVert_\infty $ using Lemma~\ref{lemma_Rhat_R0_R},
as we did in Theorem~\ref{theorem1}.
However, the bounds for $E$ and $V_T + \widetilde{\Delta}$ need to be modified.

Let $\bm{w}_j$ be such that $z^*_j(t) =  \mat{1 & \bm{z}^\top (t)}\bm{w}_j$,
where
\begin{align} \label{eq:w_j}
\bm{w}_j = \left (w^*_{j0} , \{ w^*_{jl}\mathbf{1}(l\ne j) + \mathbf{1}(l = j)\}_{1\le l \le p} \right )^\top \in \mathbb{R}^{p+1} .
\end{align}
Then
\begin{align*}
\Upsilon_j = \textrm{Cov}( z^*_j(t))
= \textrm{Cov}\left ( \mat{1 & \bm{z}^\top(t)}\bm{w}_j \right )
=  \bm{w}_j^\top  \Upsilon  \bm{w}_j  ,
\end{align*}
where
$\Upsilon = \mathbb{E} \left(  \mat{1 \\ \bm{z}(t)}  \mat{1 & \bm{z}^\top(t)  } \right)$
and $\lVert \bm{w}_j \rVert_1 = \lVert \bm{w}_j^* \rVert_1 + 1$.
Recall that
$ S_{ij} = \frac{1}{T} \sum_{t=1}^{T}  \widetilde{ \epsilon}_i(t) \,z^*_j(t)  $
as defined in (\ref{eq:S}). In addition, define
$ \widetilde{S}_{ij} = S_{ij} + \frac{1}{T} \sum_{t=1}^T z_j(t)\beta_{ij} z^*_j $.
Thus, with $\beta_{ij} = T^{-1/2}\Delta$,
\begin{align*}
V_T + \widetilde{\Delta} &=
\sqrt{T}\left(\Upsilon_j \right)^{-1/2}S_{ij} + \Upsilon^{1/2}_j \Delta  \\
&= \sqrt{T}\left(\Upsilon_j \right)^{-1/2}\left(
S_{ij} + \Upsilon_j \beta_{ij}
\right)\\
&=
\sqrt{T}\left(\Upsilon_j \right)^{-1/2}\left(
S_{ij} + z_j(t)\beta_{ij} z^*_j  -
z_j(t)\beta_{ij} z^*_j +
\Upsilon_j \beta_{ij} \right) \\
&=
\sqrt{T}\left(\Upsilon_j \right)^{-1/2}\left(
\widetilde{S}_{ij} -  \bm{w}_j^\top  \left(  \frac{1}{T}\sum_{t=1}^T
 \mat{1\\\bm{z}(t)} z_j(t)   - \Upsilon_{\cdot ,j} \right) \beta_{ij}
\right) .
\end{align*}
%
%
Then, by Lemma~\ref{prop_eigen} of the bounded eigenvalue of $\Upsilon_j$,
\begin{align}
\lVert E \rVert_2
&= \lVert \widehat{V}^0_T  - V_T - \widetilde{\Delta} \rVert_2 \nonumber \\
&=
\left \lVert
\sqrt{T} \left(\Upsilon_j \right)^{-1/2} \widehat{S}^0_{ij}
-\sqrt{T}\left(\Upsilon_j \right)^{-1/2}\left(
\widetilde{S}_{ij} -  \bm{w}_j^\top  \left(  \frac{1}{T}\sum_{t=1}^T
\mat{1\\\bm{z}(t)} z_j(t)   - \Upsilon_{\cdot ,j} \right) \beta_{ij}
\right)
\right \rVert_2 \nonumber \\
&\le
\left \lVert  \sqrt{T}\left(\Upsilon_j \right)^{-1/2}\left( \widehat{S}^0_{ij} - \widetilde{S}_{ij}  \right) \right \rVert_2 \nonumber \\
&+ \left \lVert
\left(\Upsilon_j \right)^{-1/2}
\bm{w}_j \left(  \frac{1}{T}\sum_{t=1}^T
\mat{1\\\bm{z}(t)} \mat{1 & \bm{z}^\top(t)}   - \Upsilon_{\cdot ,j} \right)
\sqrt{T}\beta_{ij}
\right \rVert_2 \nonumber \\
&=  O\left( \sqrt{T} \left \lVert \widehat{S}^0_{ij} - \widetilde{S}_{ij} \right \rVert_2 \right)  \nonumber \\
&+  O\left(  \left |\sqrt{T}\beta_{ij} \right | \lVert \bm{w}_j \rVert_1
\left \lVert
\frac{1}{T}\sum_{t=1}^T
\mat{1\\\bm{z}(t)} \mat{1 & \bm{z}^\top(t)}    - \Upsilon \right \rVert_\infty
\right) .
\end{align}
First, we examine the second item on RHS.
Using Lemma~\ref{lemma_w_bound_norm}, $\lVert \bm{w}_j \rVert_1 = O(\sqrt{s})$,
and using Lemma~\ref{lemma_concen_ineq_z} we get
$$ \left \lVert
\frac{1}{T}\sum_{t=1}^T
\mat{1\\\bm{z}(t)} \mat{1 & \bm{z}^\top(t)}    - \Upsilon \right \rVert_\infty = O_p(\rho T^{-2/5}) ,
$$
with probability at least $1- c_1 p^2 T \exp(- c_2 T^{1/5} )$. Then, with $\beta_{ij} = T^{-1/2}\Delta$,
\begin{align*}
\left |\sqrt{T}\beta_{ij} \right | \lVert \bm{w}_j \rVert_1
\left \lVert
\frac{1}{T}\sum_{t=1}^T
\mat{1\\\bm{z}(t)} \mat{1 & \bm{z}^\top(t)}    - \Upsilon \right \rVert_\infty
= O_p\left ( \sqrt{s}\rho T^{-2/5}    \right).
\end{align*}

Next, we quantify $ \sqrt{T} \left \lVert \widehat{S}^0_{ij} - \widetilde{S}_{ij} \right \rVert_2 $. Expanding the difference, we have
\begin{align}
\widehat{S}^0_{ij} - \widetilde{S}_{ij}
=& (\widehat{\bm{w} }_j  - \bm{w}^*_j)^\top  \frac{1}{T} \sum_{t=1}^T   \mat{ 1 \\ \bm{z} _{-j}(t) }   \widetilde{\epsilon}_i(t)  \nonumber \\
+& \left(  \mat{\widehat{\mu}_i & \widehat{\bm{\beta}}^\top_{i,-j} } -
\mat{\mu_i & \bm{\beta}^\top_{i,-j} }  \right)
\frac{1}{T} \sum_{t=1}^T    \mat{ 1/\sigma_i(t) \\ \bm{z}_{-j}(t) }    z^*_j(t) \nonumber \\
-& \left(  \mat{\widehat{\mu}_i & \widehat{\bm{\beta}}^\top_{i,-j} } -
\mat{\mu_i & \bm{\beta}^\top_{i,-j} }  \right)
\frac{1}{T} \sum_{t=1}^T   \mat{ 1/\sigma_i(t) \\ \bm{z}_{-j}(t) }
\mat{ 1 & \bm{z}^\top_{-j}(t) }  \left( \widehat{\bm{w} }_j  - \bm{w}^*_j \right)  \nonumber \\
+& (\widehat{\bm{w} }_j  - \bm{w}^*_j)^\top  \frac{1}{T} \sum_{t=1}^T   \mat{ 1 \\ \bm{z} _{-j}(t) }  z_{j}(t) \beta_{ij}  .  \label{eq:S_0_S_tilde_diff}
\end{align}
According to the proof of Lemma~\ref{lemma_S0_S}, the first three terms above are bounded as
$O_p \left( s\rho^2 T^{-4/5} \right)$.

With $\lVert z_j(t)\rVert_\infty = O(1)$ by Lemma~\ref{lemma_bounded_terms}, and $\lVert  \widehat{\bm{w} }_j  - \bm{w}^*_j \rVert_1 = O_p( s\rho T^{-2/5})$ in Assumption~\ref{assumption_w}, the last term is bounded as
\begin{align*}
 (\widehat{\bm{w} }_j  - \bm{w}^*_j)^\top  \frac{1}{T} \sum_{t=1}^T   \mat{ 1 \\ \bm{z} _{-j}(t) }  z_{j}(t) \beta_{ij}
 = O\left( \lVert  \widehat{\bm{w} }_j  - \bm{w}^*_j \rVert_1 \cdot | \beta_{ij} | \right)
 = O_p\left(  s\rho \Delta T^{-9/10 }  \right),
\end{align*}
with probability at least $1- c_3 p^2 T \exp(- c_4
T^{1/5})$, where $\beta_{ij} = T^{-1/2}\Delta$ as set in the alternative.
Thus,
\begin{align*}
\left| \widehat{S}^0_{ij} - \widetilde{S}_{ij}  \right|
= O_p \left( s\rho^2 T^{-4/5} \right) +O_p\left(  s\rho \Delta T^{-9/10 }  \right)
= O_p\left( s\rho^2  T^{-4/5} \right)   .
\end{align*}
As a result,
$$
\lVert E \rVert_2  =
O_p\left( s  \rho^2  T^{-3/10} \right)+
O_p\left ( \sqrt{s}\rho T^{-2/5} \Delta  \right)
=  O_p\left ( s\rho^2  T^{-3/10} \right),
$$
with probability at least $1-c_5 p^2 T \exp(- c_6 T^{1/5})$.

Next, we quantify the bound of $V_T + \widetilde{\Delta} $.
Setting $y =  T^{1/10}  $ in \eqref{eq:P_V_T} of Lemma~\ref{lemma_VT_ineq} and $\lVert V_T \rVert^2_2$ weakly converges to $\chi^2$ in Proposition~\ref{prop_clt}, we get
\begin{align}
P\left(   \left\lVert V_T + \widetilde{\Delta} \right\rVert_2  > T^{1/10}  \right)= O( T^{-1/8} ).
\end{align}
At last, we take the bounds of  $\left\lVert V_T + \widetilde{\Delta} \right\rVert_2$
and $\lVert E \rVert_2$ back to \eqref{eq:U0_hat_VT_phi0.5}, and
\begin{align*}
\left| \widehat{U}^0_T - \left\lVert V_T + \widetilde{\Delta} \right\rVert_2^2  \right  |
= O_p\left( s^2\rho^2 T^{-1/5} \right) ,
\end{align*}
with probability at least $1 - c_7 p^2 T \exp(-c_8T^{1/5}) - c_9 T^{-1/8}$.
Finally, using decomposition in \eqref{eq:bound_U0_T_Fd},
we get
\begin{align}
\sup_{x \in \mathbb{R} } | \mathbb{P}(\widehat{U}^0_T \le x) - F_{d,\lVert \widetilde{\Delta}\rVert_2^2  }(x)  | \le
c_7 p^2 T \exp( - c_8 T^{1/5}) + c_{10} T^{-1/8} + c_{11}  s^2 \rho^2 T^{-1/5} .
\end{align}
The conclusion follows from the decomposition in \eqref{eq:bound_U_T_Fd} and Lemma~\ref{lemma_U0_U} that quantify the bound of $|\widehat{U}_T - \widehat{U}^0_T |$.

\noindent
\textbf{Case}: $\phi < 1/2$.
Recall that $S_{ij}$ is defined in (\ref{eq:S}) and $V_T$ is defined in (\ref{eq:V_T}).
First, notice that
\begin{align}
\widehat{U}_T &=   T \widehat{S}_{ij} \left( \widehat{\Upsilon_j }\right)^{-1} \widehat{S}_{ij} \nonumber \\
&=
 T \widehat{S}_{ij} \left( \Upsilon_j^{-1}+  \left( \widehat{\Upsilon_j }\right)^{-1}
  -   \Upsilon_j^{-1}
 \right )
  \widehat{S}_{ij} \nonumber \\
&\ge T \lVert \left(  \Upsilon_j \right)^{-1/2}\widehat{S}_{ij}  \rVert_2^2
\left(
1 - d \left \lVert \Upsilon_j^{1/2}  \left( \widehat{\Upsilon_j }\right)^{-1}  \Upsilon_j^{1/2} - I \right\rVert_\infty
\right )  \nonumber \\
&\ge c_1 T \lVert  \left(  \Upsilon_j \right)^{-1/2}\widehat{S}_{ij}  \rVert_2^2 \nonumber \\
&= c_1  \left( \lVert
\sqrt{T}  \left( \Upsilon_j \right)^{-1/2}(\widehat{S}_{ij} - S_{ij} )
\rVert_2  - \lVert  V_T \rVert_2  \right)^2,\label{eq:U_0_phi_le_0.5}
\end{align}
where the second inequality follows from
Lemma~\ref{lemma_Rhat_R0_R} as $\left \lVert \Upsilon_j^{1/2}  \left( \widehat{\Upsilon_j^0 }\right)^{-1}  \Upsilon_j^{1/2} - I \right\rVert_\infty$ convergences to 0 when $s^2\rho^2 \log p = o(T^{1/5})$.

Next, we bound $\widehat{S}_{ij} - S_{ij}$ by separately
bounding $\widehat{S}^0_{ij}- S_{ij}$
and $\widehat{S}_{ij} - \widehat{S}^0_{ij} $,
where $\widehat{S}^0_{ij}$ is defined in \eqref{eq:S_0_hat}.

First,
\begin{align}
\widehat{S}^0_{ij} - S_{ij}
= \widehat{S}^0_{ij} - \widetilde{S}_{ij} +
\frac{1}{T} \sum_{t=1}^T
z_j^*(t) z^\top_j(t) \beta_{ij}.
\label{eq:S0hat_S}
\end{align}
A similar deduction as \eqref{eq:S_0_S_tilde_diff},
 but with $\beta_{ij} = T^{-\phi}\Delta$ leads to $\widehat{S}^0_{ij} - \widetilde{S}_{ij}= O_p(s\rho T^{-\frac{2}{5} - \phi})$.
 Thus, in the follows, we give a lower bound for $\frac{1}{T} \sum_{t=1}^T z^*_j(t) z^\top_j(t) \beta_{ij}$.

By the construction of $z_j^*(t)$ and the projection coefficients $w_j^*$ defined in \eqref{eq: w_j_star},
$\Upsilon_j = \Cov\left ( z^*_j(t) \right)
            = \E \left( z_j^*(t)
           \left ( \bm{z}_j - \bm{w}^*_{-j} \mat{ 1\\ \bm{z}_{-j}(t)  }\right )   \right)
            = \E\left( z_j^*(t) z_j(t) \right)  $.
Then,
\begin{align}
\left|
\frac{1}{T} \sum_{t=1}^T
z_j^*(t) z_j\left(t\right)
- \Upsilon_j
\right |
&= \left|
\frac{1}{T} \sum_{t=1}^T
z_j^*(t) z_j (t)
- \E\left( z_j^*(t) z_j(t) \right)
\right | \nonumber \\
&\le   \left( \lVert  \bm{w}^*_j \rVert_1 +1  \right)
\left \lVert   \frac{1}{T} \sum_{t=1}^T   \mat{1 \\ \bm{z} (t ) }   \mat{1 & \bm{z}^\top(t) }   -
\Upsilon   \right \rVert_\infty .
\end{align}
By Lemma~\ref{lemma_concen_ineq_z}, $\left \lVert   \frac{1}{T} \sum_{t=1}^T   \mat{1 \\ \bm{z} (t ) }   \mat{1 & \bm{z}^\top(t) }   -
\Upsilon   \right \rVert_\infty = O_p( \rho T^{-2/5})$. In addition,  $\lVert  \bm{w}^*_j \rVert_1 = O(\sqrt{s})$ by Lemma~\ref{lemma_w_bound_norm}.
%
%
Therefore,
$$
\left|
\frac{1}{T} \sum_{t=1}^T
\widehat{z}_j^*(t) z_j\left(t\right)
- \Upsilon_j
\right | =  O_p( (\sqrt{s}+1)\rho T^{-2/5}) ,
$$
with probability at least $1-c_2p^2 T \exp(-c_3 T^{1/5} )$.
Then, by $\beta_{ij} = T^{-\phi}\Delta$,
$s^2\rho^2 \log p= o(T^{1/5})$ and $\Lambda_{\min}\left(\Upsilon_j \right) >0$
in Proposition~\ref{prop_eigen},
\begin{align*}
\frac{1}{T} \sum_{t=1}^T
\widehat{z}^*_j(t) z^\top_j (t ) \beta_{ij}
\ge \left( \Upsilon_j  - c_4 T^{-2/5 } (\sqrt{s}+1) \rho \right)  \left|\beta_{ij}\right |
\ge  c_5 T^{-\phi} ,
\end{align*}
with probability at least $1-c_2p^2 T \exp(-c_3 T^{1/5} )$.
Taking the bounds above back to \eqref{eq:S0hat_S}, under $s^2\rho^2 \log p = o(T^{1/5})$,
we have
\begin{align}
T   \left \lVert \left( \Upsilon_j \right)^{-1/2} \left( \widehat{S}^0_{ij} - S_{ij} \right) \right \rVert^2_2
\ge O_p(T^{1-2\phi })  .
\end{align}
with probability at least $1-c_6 p^2 T \exp(- c_7 T^{1/5})$.


Next, we bound $ \widehat{S}_{ij} - \widehat{S}^0_{ij}$.
Without repeating the technical details, the bound of $ \widehat{S}_{ij} - \widehat{S}^0_{ij}$ can be derived as follows:
under $\phi < 1/2$, we add an addition term of $x_j(t)\beta_{ij}$ to $\epsilon_i(t)$ because $\beta_{ij}\ne 0$ under the alternative; then the bound for $A_1$ and $B_1$ in the proof of Lemma~\ref{lemma_U0_U} are dominated by $O_p(\rho T^{-\frac{2}{5}-\phi} \vee T^{-4/5})$ instead of $O_p(\rho T^{-\frac{4}{5}})$ under the alternative, which implies
$ \lVert \widehat{S}_{ij} - \widehat{S}^0_{ij}  \rVert_2
=  O_p\left(\rho T^{-\frac{2}{5}-\phi} \vee T^{-4/5}  \right) $.
Combining the results above,
\begin{align}
T \left \lVert \left( \Upsilon_j \right)^{-\frac{1}{2}} \left( \widehat{S}_{ij} - S_{ij} \right) \right \rVert^2_2
&\ge
T   \left \lVert \left( \Upsilon_j \right)^{-\frac{1}{2}}  \left(
\widehat{S}_{ij}  - \widehat{S}^0_{ij}  +
\widehat{S}^0_{ij}  - S_{ij} \right) \right \rVert^2_2  \nonumber \\
&\ge
T   \left  \lVert  \left( \Upsilon_j \right)^{-\frac{1}{2}}  \left(
\widehat{S}^0_{ij}  - S_{ij} \right)   \right \rVert^2_2  -
T   \left   \lVert \left( \Upsilon_j \right)^{-\frac{1}{2}} \left(
\widehat{S}_{ij}  - \widehat{S}^0_{ij} \right)  \right \rVert^2_2  \nonumber \\
& = O_p( T^{1-2\phi }) , \label{eq:S_hat_S_phi_le_0.5}
\end{align}
with probability at least $1-c_{10} p^2 T \exp(- c_{11} T^{1/5})$.

Note that $ \lVert V_T \rVert^2_2$ weakly converges to $\chi^2$ by Proposition~\ref{prop_clt}. In addition, taking $y=  c_{12} T^{1/2 -\phi} - c^{-1}_1 \sqrt{x} $ in Lemma~\ref{lemma_VT_ineq},
\begin{align}
\mathbb{P}( \lVert V_T \rVert_2 \ge c_{12} T^{1/2 -\phi} - c^{-1}_1 \sqrt{x}  )
\le   c_{13} T^{-1/8}  + c_{14} \exp(  -( c_{12} T^{1/2 -\phi} - c^{-1}_1 \sqrt{x} )^2  ) .
\label{eq: V_T_tail_bound}
\end{align}
Taking
\eqref{eq:S_hat_S_phi_le_0.5} and
\eqref{eq: V_T_tail_bound} back to \eqref{eq:U_0_phi_le_0.5}, we reach the conclusion:
\begin{multline*}
\mathbb{P}\left( \widehat{ U}_T \ge  x  \right) \\
\ge \mathbb{P}\left(
\left\{ \lVert
\sqrt{T}  \left( \Upsilon_j \right)^{-1/2}(\widehat{S}_{ij} - S_{ij} )
\rVert_2 \ge c_{12} T^{1/2 -\phi}\right\} \cap
\left\{ \lVert  V_T \rVert_2 \le c_{12} T^{1/2 -\phi} - c^{-1}_1 \sqrt{x} \right\} \right) \\
\ge 1 - c_{10} p^2 T \exp(- c_{11} T^{1/5})  - c_{13} T^{-1/8}  - c_{14} \exp\left (  -\left ( c_{12} T^{1/2 -\phi} - c^{-1}_1 \sqrt{x} \right )^2  \right )  .
\end{multline*}

\subsection*{Proof of Theorem~\ref{theorem3}}

Denote
$$
\check{S}_{ij} =
\frac{1}{T} \sum_{t=1}^{T} \frac{ Y_i(t) - x_j(t) \beta_{i,j} - \widehat{\mu}_i - \bm{x}^\top_{-j}(t)\widehat{ \bm{\beta}}_{i,-j}  }{\widehat{\sigma}_i(t)  } \widehat{z}^*_j(t) .
$$
We have the following decomposition.
	\begin{align*}
	\widetilde{S}_{ij} &= \check{S}_{ij} +
	\frac{1}{T} \sum_{t=1}^T  \widehat{z}^*_j(t) \widehat{z}_j(t)( \widehat{\beta}_{ij} - \beta_{ij} ) =
	\check{S}_{ij} +
	\widetilde{\Upsilon}_j \left(   \widehat{\beta}_{ij} - \beta_{ij}  \right).
	\end{align*}
By the decomposition above, we have
	$$
	\widehat{b}_{ij} -  \beta_{ij} = - \left(\widetilde{\Upsilon}_j   \right)^{-1} \check{S}_{ij} ,
	$$
	and
	$$
	\widehat{R}_T = \left( \check{S}_{ij} \right)^\top  \left(\widetilde{\Upsilon}_j   \right)^{-1} \widehat{\Upsilon}_j \left(\widetilde{\Upsilon}_j   \right)^{-1} \check{S}_{ij}.
	$$
	Notice that
	$$
	\check{S}_{ij} =
	\frac{1}{T} \sum_{t=1}^{T} \frac{ \epsilon_i(t) + ( \mu_i - \widehat{\mu}_i)
		+ \bm{x}^\top_{-j}(t) (\bm{\beta}_{i,-j} - \widehat{ \bm{\beta}}_{i,-j} )  }{\widehat{\sigma}_i(t)  } \widehat{z}^*_j(t) ,
	$$
	which is equivalent to $\widehat{S}_{ij}$ defined in (\ref{eq:S_hat}) under the null hypothesis.
	Let
	$$
	\check{U}_T = \check{S}_{ij}^\top \widehat{\Upsilon}_j^{-1} \check{S}_{ij} .
	$$
	$\check{U}_T$ weakly converges to $\chi^2$ distribution following the proof of Theorem~\ref{theorem1}:
	\begin{align} \label{eq:U_check}
	\sup_{ x\in \mathbb{R} } \left |  {\mathbb P}(\check{U}_T \le x) - F_d(x) \right | \le
	c_1 p^2 T \exp(- c_2 T^{1/5}) + c_3 s^2 \rho^2   T^{-1/5}+ c_4 T^{-1/8}.
	\end{align}

	Next, we bound
\begin{align} \label{eq:RThat_checkUT}
	\widehat{R}_T - \check{U}_T =
\check{S}_{ij}^\top  \left( \left(\widetilde{\Upsilon}_j   \right)^{-1} \widehat{\Upsilon}_j \left(\widetilde{\Upsilon}_j   \right)^{-1}
-  \left( \widehat{\Upsilon}_j \right)^{-1}
\right) \check{S}_{ij} .
\end{align}
Using Assumptions~\ref{assumption3}-\ref{assumption_w} and the consistency of estimators in Assumption~\ref{assumption_beta} and~\ref{assumption_w}, it is easy to see that
$\check{S}_{ij} = O_p(1)$.
  Therefore, it is enough to quantify
	$ (\widetilde{\Upsilon}_j   )^{-1} \widehat{\Upsilon}_j (\widetilde{\Upsilon}_j   )^{-1}
	-  ( \widehat{\Upsilon}_j )^{-1}$ in order to quantify $  \widehat{R}_T - \check{U}_T$.
Let $E = 	\widetilde{\Upsilon}_j   - \widehat{\Upsilon}_j$. Then,
	\begin{align*}
	\widetilde{\Upsilon}_j     \left( \widehat{\Upsilon}_j \right)^{-1} \widetilde{\Upsilon}_j
	=
	\left(\widehat{\Upsilon}_j  + E  \right)
	\left( \widehat{\Upsilon}_j \right)^{-1}
	\left(\widehat{\Upsilon}_j  + E  \right)
	= \widehat{\Upsilon}_j   + E + E + E \left( \widehat{\Upsilon}_j \right)^{-1}  E   ,
	\end{align*}
	which leads to
	\begin{align}
	   \widetilde{\Upsilon}_j     \left( \widehat{\Upsilon}_j \right)^{-1} \widetilde{\Upsilon}_j - \widehat{\Upsilon}_j
	 = E + E + E \Upsilon^{-1}_j  E
	 + E \left( \left( \widehat{\Upsilon}_j \right)^{-1} -  \Upsilon^{-1}_j  \right)  E .
	 \label{eq:th3_Rtilde_Rhat}
	\end{align}
	Proposition~\ref{prop_eigen} implies that $E  \Upsilon^{-1}_j  E   \le O(E^2) $.
	Thus, the first three items are bounded by $O(|E|\vee E^2)$.
  By Proposition~\ref{prop_eigen}
  $\Upsilon_j^{-1} = O(1)$, and Theorem 2.5 in \citet{Stewart90matrixperturbation}
  gives us
	\begin{align*}
	\left \lVert \left(\widehat{\Upsilon}_j  \right)^{-1}  - \Upsilon_j^{-1} \right \rVert_2 \le
	\frac{   \lVert \Upsilon_j^{-1}\rVert_2 \lVert \widehat{\Upsilon}_j -  \Upsilon_j \rVert_2  }{
		1-   \lVert \Upsilon_j^{-1}\rVert_2 \lVert \widehat{\Upsilon}_j -  \Upsilon_j \rVert_2  } = O( \lVert \widehat{\Upsilon}_j -  \Upsilon_j \rVert_2 ) .
	\end{align*}
	Then, taking this result back to \eqref{eq:th3_Rtilde_Rhat}, we get
	$$
	\widetilde{\Upsilon}_j     \left( \widehat{\Upsilon}_j \right)^{-1} \widetilde{\Upsilon}_j - \widehat{\Upsilon}_j = O( |E|) + O(E^2) + O(E^2)O\left(\left\lVert  \widehat{\Upsilon}_j -  \Upsilon_j \right\rVert_2 \right).
	$$
	By Lemma~\ref{lemma_Rhat_R0_R},
	$\lVert \widehat{\Upsilon}_j -  \Upsilon_j \rVert_2   = O_p\left( s^2\rho^2 T^{-2/5}\right)$. Therefore, to bound $\widetilde{\Upsilon}_j     \left( \widehat{\Upsilon}_j \right)^{-1} \widetilde{\Upsilon}_j - \widehat{\Upsilon}_j$, it is sufficient to quantify the bound of $E$.
   We first write $E=\widetilde{\Upsilon}_j   - \widehat{\Upsilon}_j   =  \widetilde{\Upsilon}_j -
	\widetilde{\Upsilon}^0_j +   \widetilde{\Upsilon}^0_j
	- \widehat{\Upsilon}^0_j + \widehat{\Upsilon}^0_j - \widehat{\Upsilon}_j$ and then quantify the bound for each difference.

	We start with the bound of $ E^0 \equiv \widetilde{\Upsilon}^0_j - \widehat{\Upsilon}^0_j$.
	Notice that
	\begin{align*}
	\lVert E^0 \rVert_\infty
	&\le \left\lVert \frac{1}{T} \sum_{t=1}^T  \left( z_j(t) - \mat{1 & \bm{z}_{-j}^\top(t)  } \widehat{\bm{w} }_j   \right) \mat{1 & \bm{z}_{-j}^\top(t)  }  \widehat{\bm{w} }_j   \right \rVert_\infty \\
	&\le
	\left\lVert \frac{1}{T} \sum_{t=1}^T  \left( z_j(t) -   \mat{1 & \bm{z}_{-j}^\top(t)  }  \bm{w}^*_j  \right) \bm{z}_{-j}(t) \right \rVert_\infty ( \lVert \bm{w}_j^* \rVert_1  + \lVert \bm{w}_j^* - \widehat{\bm{w} }_j  \rVert_1 ) \\
	&+  \left | (\widehat{\bm{w} }_j  - \bm{w}^*_j)^\top  \frac{1}{T} \sum_{t=1}^T \left(  \mat{1 \\ \bm{z}_{-j}(t)  }   \mat{1 & \bm{z}_{-j}^\top(t)  }  \right)   (\widehat{\bm{w} }_j  - \bm{w}^*_j) \right|  \\
	&+ \left \lVert \frac{1}{T}\sum_{t=1}^T   \mat{1 \\ \bm{z}_{-j}(t)  }   \mat{1 & \bm{z}_{-j}^\top(t)  }   \bm{w}_j^* \right \rVert_\infty \lVert \widehat{\bm{w} }_j  - \bm{w}^*_j \rVert_1 .
	\end{align*}
	Using Lemma~\ref{lemma_w_empirical} (to bound $\left\lVert \frac{1}{T} \sum_{t=1}^T  \left( z_j(t) -   \mat{1 & \bm{z}_{-j}^\top(t)  }  \bm{w}^*_j  \right) \bm{z}_{-j}(t) \right \rVert_\infty$), Lemma~\ref{lemma_w_bound_norm} of bounded norm of $\bm{w}^*_j$, and the estimation consistency of $\widehat{\bm{w}}_j$ in Assumption~\ref{assumption_w},
	the first item on the RHS is $O_p( \sqrt{s} \rho T^{-2/5} )$. Using Assumption~\ref{assumption_w} again and Lemma~\ref{lemma_bounded_terms}, the second item on RHS is $O_p( s\rho T^{-2/5})$.
	by Lemma~\ref{lemma_w_bound_norm} and~\ref{lemma_bounded_terms},
	\begin{align*}
	\left \lVert \frac{1}{T}\sum_{t=1}^T   \mat{1 \\ \bm{z}_{-j}(t)  }   \mat{1 & \bm{z}_{-j}^\top(t)  }   \bm{w}_j^* \right \rVert_\infty
	&\le  \lVert \bm{w}^*_j \rVert_1
	       \left \lVert    \mat{1 \\ \bm{z}_{-j}(t)  }   \mat{1 & \bm{z}_{-j}^\top(t)  }    \right\rVert_\infty
	       =  O(\sqrt{s}) .
	\end{align*}
	Then, under Assumption~\ref{assumption_w}, the last term on the RHS is bounded by $O_p(s^{3/2}\rho T^{-2/5} )$.
	Therefore,
	 $$ E^0 =  \widetilde{\Upsilon}^0_j - \widehat{\Upsilon}^0_j =O_p(s^{3/2}\rho T^{-2/5} )  , $$
	with probability at least $1-c_5p^2 T\exp(-c_6T^{1/5})$.
By Lemma~\ref{lemma_Rhat_R0_R}, $\widehat{\Upsilon}^0_{j} - \widehat{\Upsilon}_{j}$
is bounded by $O_p(\rho T^{-2/5})$ and $\widetilde{\Upsilon}^0_{j} - \widetilde{\Upsilon}_{j} = O_p(\rho T^{-2/5})$ .
	Combining these results, we get
	$|E|= O_p(s^{3/2}\rho T^{-2/5} )$.
	Next, turning back to \eqref{eq:th3_Rtilde_Rhat},
	$\widetilde{\Upsilon}_j     \left( \widehat{\Upsilon}_j \right)^{-1} \widetilde{\Upsilon}_j - \widehat{\Upsilon}_j =O_p(s^{3/2}\rho T^{-2/5} )$.
	Then, referring to \eqref{eq:RThat_checkUT} at beginning,
	\begin{align*}
	\widehat{R}_T - \check{U}_T &=
	\check{S}_{ij}^\top  \left( \left(\widetilde{\Upsilon}_j   \right)^{-1} \widehat{\Upsilon}_j \left(\widetilde{\Upsilon}_j   \right)^{-1}
	-  \left( \widehat{\Upsilon}_j \right)^{-1}
	\right) \check{S}_{ij} \\
	&= O_p(1) O_p(s^{3/2} \rho T^{-2/5}) O_p(1)
 = O_p(s^{3/2} \rho T^{-2/5}) ,
	\end{align*}
	with probability at least $1-c_6p^2 T\exp(-c_7T^{1/5})$.
	At the end, taking $\delta = s^{3/2} \rho T^{-2/5}$ and using \eqref{eq:U_check},
	\begin{align*}
	&\sup_{x\in \mathbb{R} }  |\mathbb{P}(\widehat{R}_T \le x) - F_d(x) |    \\
	& \le
	\sup_{y\in \mathbb{R} } |\mathbb{P}( \check{U}_T  \le y) - F_d(y)  |
	+  \sup_{x\in \mathbb{R} } \left( F_d(x+ \delta ) - F_d(x -\delta) \right)
	+  \mathbb{P}( | \widehat{R}_T - \check{U}_T  | > \delta ) \\
	&\le 	c_8 p^2 T \exp(- c_9 T^{1/5}) + c_{10} s^2 \rho^2   T^{-1/5}+ c_{11} T^{-1/8} .
	\end{align*}


\clearpage
\section*{Appendix~B: Auxiliary Results}\label{sec:auxilary}

The following result shows that
$ \mathbb{P}\left( U_T \le y\right) \approx F_d(y) $ uniformly in $y$.
The proof is based on a martingale central limit theorem.
Recall that $s_j = \lVert\bm{w}_j^* \rVert_0$ and $s = \max_{1\le j\le p} s_j$; $\rho_i
= \lVert \bm{\beta}_i \rVert_0$ and $\rho = \max_{1\le i\le p} \rho_i$.

\begin{proposition}\label{prop_clt}
  Suppose the linear Hawkes model with its intensity function defined in
 \eqref{eq:linear_hawkes_para_transfer} is stationary and satisfies Assumptions~\ref{assumption1} --~\ref{assumption4}.
  Then  $\forall u \in \mathbb{R}^d$,
  \begin{align}
    \sup_{y\in \mathbb{R} } \left |  \mathbb{P}( \lVert V_T + u \rVert^2_2 \le y) - F_{d,\lVert u \rVert_2^2 }(y)  \right|
    \le C(\lVert u \rVert_2, d) T^{-1/8},
  \end{align}
  where $C(\lVert u \rVert_2, d) $ is a constant that is non-decreasing w.r.t.
  $\lVert u \rVert_2 $ and $d$.
\end{proposition}

\begin{proof}

The main part of the proof is based on the result on the martingale difference sequence in Lemma~\ref{lem:mart_diff_seq}. To reach the conclusion, we verify the conditions required
to apply the result in the setting of the multivariate Hawkes process.

Let
\begin{align*}
\xi_{T,t} = - \frac{1}{\sqrt{T}}\left( \Upsilon_j \right)^{-1/2}
\frac{\epsilon_i(t)}{\sigma_i(t)}   z^*_j (t),
\end{align*}
where $\sigma_i(t)$, $z^*_j (t)$ are defined in \eqref{eq:sigma_i} and  \eqref{eq:x_tilde_star}, respectively.

Recall that  $\mathcal{H}_{T,t}$ is information filtration of the past.
Then $( \xi_{T,t}, \mathcal{H}_{T,t} )$ is a martingale difference sequence. Then, $V_T = \sum_{t=1}^{T} \xi_{T,t}$, where $V_T$ is defined in \eqref{eq:V_T}.
Following the same notation as Lemma~\ref{lem:mart_diff_seq}, denote
\begin{align}
L_{\delta}^{n,d} &=  \sum_{t=1}^{T}  \mathbb{E}\lVert \xi_{Tt} \rVert_2^{2 + 2\delta}
, \\
N_{\delta}^{T,d} & =
E\left \lVert
\left( \Upsilon_j \right)^{-1/2}
\left( \frac{1}{T} \sum_{t=1}^{T}   z^*_j(t) \left( z^*_j(t) \right)^\top    - \Upsilon_j \right )
\left( \Upsilon_j \right)^{-1/2}
\right \lVert_{tr}^{1+ \delta }.
\end{align}
To apply Lemma~\ref{lem:mart_diff_seq}, we need to evaluate the bound of
$R_{\delta}^{n,d} \equiv L_{\delta}^{n,d} + N_{\delta}^{T,d}$.

We start with the bound of $L_{\delta}^{n,d}$.
Notice that
\begin{align*}
L_{\delta}^{n,d} =  \sum_{t=1}^{T}  \mathbb{E}\lVert \xi_{Tt} \rVert_2^{2 + 2\delta}
\le  \Lambda^{-1/2}_{\min}\left( \Upsilon_j\right) T^{-(1+\delta)}
\sum_{i=1}^{T}  \mathbb{E}\left \lVert  \frac{1}{\sigma_i(t)}\epsilon_i(t) z^*_j (t) \right \rVert_2^{2 + 2\delta}.
\end{align*}

Proposition~\ref{prop_eigen} implies $ \Lambda^{-1/2}_{\min}\left( \Upsilon_j\right)  = O(1)$.
Lemma~\ref{lemma_bounded_terms} shows that  $0 < \sigma^2_i(t)$  and $x_j(t)$, $\epsilon_i(t)$ and $z_j(t)$ are bounded. Therefore,
\begin{align*}
\mathbb{E}\left \lVert  \frac{1}{\sigma_i(t)}\epsilon_i(t) z^*_j (t) \right \rVert_2^{2 + 2\delta}
\le O(
\mathbb{E}\left \lVert  z^*_j (t) \right \rVert_2^{2 + 2\delta}
 ).
\end{align*}

Recall that  $z^*_j(t)$
is $z_j(t)$ after removing its projection onto $\bm{z}_{-j}$. Then,
$ \E \left\lVert z^*_j (t) \right\rVert_2^2 \le  \E\left\lVert z_j (t) \right\rVert^2_2   =O(1) $. In addition, since $x^{1+\delta}$ is convex under $\delta \in [0,1/2]$ for $x \ge 0$,
\begin{align*}
\mathbb{E} \left( \left \lVert  z^*_j (t) \right \rVert_2^2 \right)^{1 + \delta}
\le
\left(  \mathbb{E}  \left \lVert  z^*_j (t) \right \rVert_2^2 \right)^{1 + \delta}
\le
\left(  \mathbb{E}  \left \lVert  z_j (t) \right \rVert_2^2 \right)^{1 + \delta} = O(1).
\end{align*}
Thus, $L_{\delta}^{n,d} = O\left( T^{-\delta} \right )$.


Next, we quantify the bound of $N_{\delta}^{n,d}$. Notice that
\begin{align*}
\sum_{t=0}^{T-1}  \mathbb{E} \left(  \xi_{T,t} \xi_{T,t}^\top \mid \mathcal{H}_t \right)  - I
&=  \left( \Upsilon_j \right)^{-1/2}
\left( \frac{1}{T} \sum_{t=1}^{T}  z^*_j(t) \left( z^*_j(t) \right)^\top   - \Upsilon_j  \right)
\left( \Upsilon_j \right)^{-1/2}.
\end{align*}
By Proposition~\ref{prop_eigen}, the rank of
\[
\left( \Upsilon_j \right)^{-1/2}
\left(  \frac{1}{T} \sum_{t=1}^{T}  z^*_j(t) \left( z^*_j(t) \right)^\top    - \Upsilon_j   \right)
\left( \Upsilon_j \right)^{-1/2}
\]
is at most $d$ (where $d=1$ in the case of testing an univariate $\beta_{ij}$).
Since $\lVert B \rVert_{tr}  \le d \lVert B \rVert_2$
and $\lVert B_d \rVert_2 \le d \lVert B_d \rVert_\infty$ for $B \in \mathbb{R}^{d\times d}$ ,
we have
\begin{align*}
N_{\delta}^{T,d} & =
E\left \lVert
\left( \Upsilon_j \right)^{-1/2}
\left( \frac{1}{T} \sum_{t=1}^{T}   z^*_j(t) \left( z^*_j(t) \right)^\top    - \Upsilon_j \right )
\left( \Upsilon_j \right)^{-1/2}
\right \lVert_{tr}^{1+ \delta }
\\
& \le  \mathbb{E}\left(  d\left \lVert
\left( \Upsilon_j \right)^{-1/2}
\left( \frac{1}{T} \sum_{t=1}^{T} z^*_j(t) \left( z^*_j(t) \right)^\top    - \Upsilon_j  \right)
\left( \Upsilon_j \right)^{-1/2}
\right \lVert_{2} \right)^{1+ \delta } \\
&\le \Lambda_{\min}\left( \Upsilon_j  \right)^{-1} \mathbb{E}\left(
d^2 \left \lVert
\frac{1}{T} \sum_{t=1}^{T}z^*_j(t) \left( z^*_j(t) \right)^\top    - \Upsilon_j
\right \rVert_\infty  \right)^{1+\delta } ,
\end{align*}
where the last step follows from Proposition~\ref{prop_eigen}.

Now,
\begin{multline*}
\left \lVert
\frac{1}{T} \sum_{t=1}^{T} z^*_j(t) \left( z^*_j(t) \right)^\top   - \Upsilon_j
\right \rVert_\infty \le
\\
\left (1 + \lVert \bm{w}_{j,-j}^* \rVert^2_1 \right ) \left \lVert
\frac{1}{T} \sum_{t=1}^{T} \bm{z}(t) \bm{z}^\top(t)    -
\mathbb{E} \left( \frac{1}{T} \sum_{t=1}^{T} \bm{z}^\top(t)   \bm{z}^\top (t)  \right)
\right \rVert_\infty
+  \lVert w_{j0} \rVert^2_1 \left \lVert \frac{1}{T} \sum_{t=1}^T  \bm{z}(t) -  \mathbb{E}\left(\bm{z}(t) \right)  \right \rVert_\infty,
\end{multline*}
where bounds on 
\[
\left \lVert
\frac{1}{T} \sum_{t=1}^{T} \bm{z}^\top (t)   \bm{z}(t)     -
\mathbb{E} \left( \frac{1}{T} \sum_{t=1}^{T}   \bm{z}^\top (t)    \bm{z}(t)  \right)
\right \rVert_\infty
\quad\text{and}\quad
\left \lVert \frac{1}{T} \sum_{t=1}^T  \bm{z}(t) -  \mathbb{E}\left(\bm{z}(t) \right)  \right \rVert_\infty
\]
are given in Lemma~\ref{lemma_concen_ineq_z}. 
In addition, the $\ell_1$-norm of $\bm{w}_j^*$ is bounded in Lemma~\ref{lemma_w_bound_norm}.
Therefore, assuming $s^2\rho^2 \log p = o(T^{1/5})$,
\begin{align*}
N_{\delta}^{T,d}
& \le \int_0^\infty \mathbb{P} \left( \left(
d^2 \left \lVert \frac{1}{T} \sum_{t=1}^{T}  z^*_j(t)  \left( z^*_j(t) \right)^\top   - \Upsilon_j  \right \rVert_\infty  \right)^{1+\delta } > r \right) dr  \\
&=
\int_0^\infty
\mathbb{P} \left(
d^2 \left \lVert  \frac{1}{T} \sum_{t=1}^{T}    z^*_j(t)  \left( z^*_j(t) \right)^\top     - \Upsilon_j \right \rVert_\infty
> r^{1/(1+\delta)}
\right) d r  \\
&\le
\int_0^\infty C_1 \exp \left( - C_2
\min\left \{  \left( \frac{T}{s\rho} r^{1/(1+\delta)} \right)^{1/3},  \frac{T}{\rho} r^{1/(1+\delta)}
\right \}
  \right)  dr  \\
&\le   C(\delta)  T^{-1-\delta } (s\rho)^{1+\delta} ,
\end{align*}
where the last step is based on the integral of the gamma function.

Then,
$$
R_{\delta}^{n,d} = L_{\delta}^{n,d} + N_{\delta}^{n,d}
\le
C(\delta) T^{-\delta}
+
C'(\delta)  T^{-1-\delta } (s\rho)^{1+\delta}  ,
$$
where the first term dominates under $s^2\rho^2 \log p = o(T^{1/5})$.

Then taking $\delta = \frac{1}{2}$,
Therefore, by Lemma~\ref{lem:mart_diff_seq}, we have for any $x \ge 0$, $u \in \mathbb{R}^d$,
and $\delta \in [0, 1/2] $,
\begin{align}
| \mathbb{P}(\widehat{U}_T + u \le x) - F_{d, \lVert u \rVert_2^2}(x)  | \le
C(\lVert u \rVert_2^2 , d  )  T^{-1/8},
\end{align}
which completes the proof.
\end{proof}

\begin{proposition}\label{prop_eigen}
  Suppose the linear Hawkes model with its intensity function defined in
  \eqref{eq:linear_hawkes_para_transfer} is stationary and
  satisfies Assumptions~\ref{assumption1} --~\ref{assumption4}.
  Let $\Upsilon_x = \textrm{Cov}( \bm{x}(t) )$.
  Then,
	\begin{align*}
	&0 < C_1  \le \Lambda_{\min}\left ( \Upsilon_x \right ) \le \Lambda_{\max}\left ( \Upsilon_x \right )  \le C_2 < \infty , \\
	&C_3\Lambda_{\min}\left ( \Upsilon_x \right ) \le  \Lambda_{\min}\left ( \Upsilon_j \right ) \le \Lambda_{\max} \left ( \Upsilon_j  \right )  \le C_4\Lambda_{\max}\left ( \Upsilon_x \right ) ,
	\end{align*}
where constants $C_k, k=1,\dots,4$, only depend on $(\Theta, \bm{\mu})$
and the transition kernel function.
\end{proposition}

\begin{proof}

Recall that $ x_j(t) = \int_0^{t-} k_j(t-s) dN_j(s) $ for $j = 1, \dots, p$. 
Let $$K_{t} =\mat{ k_1(t) & & \\ & \ddots& \\ & & k_p(t) }$$
and $\bm{dN}_s = (dN_1(s), \dots, dN_p(s) )^\top$.
Then, $\bm{x}(t) = (x_1(t) , \dots , x_p(t) )^\top = \int_0^{t-} K_{t-s}  \bm{dN}_s $. We consider
bounding eigenvalues of
\begin{align*}
\Upsilon_x &=  \textrm{Cov}( \bm{x}(t) )  \\
&=  \int_0^{t^-} \int_0^{t^-}  K_{t-s}  \mathbb{E} \left(  \bm{dN}_s - \Lambda ds\right )
\left ( \bm{dN}^\top_{r} - \Lambda^\top dr \right)   K_{t-r}    \\
&= \int_{0}^{t^-} \int_0^{t^-}  K_{t-s}  \Gamma(s-r)  K_{t-r}   dr ds ,
\end{align*}
where $\Gamma(l) =  \mathbb{E} \left(  \bm{dN}_t - \Lambda dt\right )
\left ( \bm{dN}^\top_{t+l} - \Lambda^\top dt \right) / (dt)^2  \in \mathbb{R}^{p\times p}$.

Let $f_\Gamma(\theta ) =  \int_{-\infty}^{\infty} \Gamma(l) e^{ -i\theta l}dl $. Thus, $\Gamma(l) = \frac{1}{2\pi}   \int_{-\pi}^{\pi} f_\Gamma(\theta ) e^{il\theta }d\theta $. In addition, let $G_{t,s}(\theta) = \int_0^{t-} K(t-s) e^{-is\theta}ds$. Then,
\begin{align*}
\Upsilon_x
&=   \int_{0}^{t^-} \int_0^{t^-}  K_{t-s}e^{i\theta s}  \Gamma(s-r)e^{i\theta(r-s)}  K_{t-r} e^{-i\theta r}  dr ds \\
&=   \int_{0}^{t^-} \int_0^{t^-}  K_{t-s}e^{i\theta s} \left( \frac{1}{2\pi}  \int_{-\pi}^{\pi} f_\Gamma(\theta )  e^{i(s-r)}  d\theta \right)e^{i\theta(r-s)}     K_{t-r} e^{-i\theta r}  dr ds \\
&= \int_{0}^{t^-} \int_0^{t^-}  K_{t-s}e^{i\theta s} \left( \frac{1}{2\pi}  \int_{-\pi}^{\pi} f_\Gamma(\theta )  d\theta \right)   K_{t-r} e^{-i\theta r}  dr ds \\
&=  \frac{1}{2\pi}  \int_{-\pi}^{\pi}  G^*_{t,s}(\theta)   f_\Gamma(\theta )G_{t,r}(\theta)      d\theta  .
\end{align*}
Since $ f_\Gamma(\theta ) $ is Hermitian and $G_t^*(s)  f_\Gamma(\theta )  G_t(r) $ is real,
$$
  \mathfrak{m}(f_\Gamma)G^*_{t,s}(\theta)G_{t,r}(\theta) \le G^*_{t,s}(\theta) f_\Gamma(\theta )  G_{t,r}(\theta)  \le  \mathfrak{M}(f_\Gamma)G^*_{t,s}(\theta)G_{t,r}(\theta)  ,
$$
where
\begin{align*}
\mathfrak{M}(f_\Gamma) & = ess\sup_{\theta \in [-\pi, \pi]} \sqrt{  \Lambda_{\max} \left( f_\Gamma(\theta ) f_\Gamma(\theta )^* \right)  } , \\
\mathfrak{m}(f_\Gamma) & = ess\inf_{\theta \in [-\pi, \pi]} \sqrt{  \Lambda_{\min} \left( f_\Gamma(\theta ) f_\Gamma(\theta )^* \right) } .
\end{align*}
In addition, notice that
\begin{align*}
\frac{1}{2\pi}  \int_{-\pi}^{\pi} G^*_{t,s}(\theta) G_{t,r}(\theta)  d\theta
&= \frac{1}{2\pi}  \int_{-\pi}^{\pi} \int_0^{t^-} \int_0^{t^-}
K_{t-s}  K_{t-r} e^{i\theta (s-r)}  ds dr  d\theta  \\
&=  \int_0^{t^-}  \int_{-t^-}^{t^-}  K_{t-r-l}  \frac{1}{2\pi}  \int_{-\pi}^{\pi} e^{i\theta l }  d\theta  K_{t-r}  dl dr \\
&= \int_0^{t^-} K^2_{t-r} dr .
\end{align*}
Letting $Q(t) =  \int_0^{t^-}  K^2_{t-r} dr$,
\begin{align*}
     \mathfrak{m}(f_\Gamma)\Lambda_{\min}(Q(t))
     \le \Lambda_{\min}\left(\Upsilon_x  \right)
     \le
\Lambda_{\max}\left(\Upsilon_x  \right)
\le \mathfrak{M}(f_\Gamma) \Lambda_{\max}(Q(t) ) ,
\end{align*}

Next, we introduce the result linking $f_\Gamma(\cdot)$ to the transition matrix $\Omega$.

\begin{lemma} \label{lem:f_gamma}
For a stationary multivariate Hawkes process,
$$
f_\Gamma(\theta)   =  \left( I - f_\omega(\theta) \right)^{-1} diag(\Lambda) \left( I - f^*_\omega(\theta)  \right)^{-1} ,
$$
where $\Lambda= \mat{ \mathbb{E} \lambda_1(t) & \dots & \mathbb{E} \lambda_p(t) }^\top \in \mathbb{R}^{p}$
and $ f_\omega(\theta)$ is the Fourier transformation on the matrix of the transition function $\omega(t) = (\omega_{ij}(t))_{1\le i,j\le p}$.
\end{lemma}

A similar result as Lemma~\ref{lem:f_gamma} has been shown by \citet[][Theorem~1]{Bacry2011} or \citet[][Theorem~3]{Etesami2016}, but they assume a non-negative transfer function.
Lemma~\ref{lem:f_gamma} extends the class of linear Hawkes processes by
including non-mutually exciting structure. The proof is directly
established based on the proof in \citet[][Theorem~1]{Bacry2011} or \citet[][Theorem~3]{Etesami2016} but we re-write a real function into its positive part and its negative part, and then using the distributive property of the convolution operation.

As a result, one set of sufficient conditions to bound the spectral radius
of $\Gamma$ is to assume:
\begin{itemize}
	\item $0 < \lambda_{\min} \le \min_{1\le i \le p}\lambda_i(t) \le \max_{1\le i \le p} \lambda_i(t ) \le \lambda_{\max} <    \infty  $ as stated in Assumption~\ref{assumption3};
	\item $\Lambda_{\max} \left( \Omega  \right) < 1 $ which leads
	$\mathfrak{m} \left(  I - f^*_\omega(\theta) \right) > 0$ assumed in Assumption~\ref{assumption1};
	\item Bounded row and column sum of $\Omega$ as assume in Assumption~\ref{assumption2}; that is,
	\begin{align*}
	\mathfrak{M} \left( I - f^*_\omega(\theta) \right)  \le 1 +
	\left( \max_{1\le i\le p} \sum_{j=1}^{p} \Omega_{ij}
	+ \max_{1\le j\le p} \sum_{i=1}^{p} \Omega_{ij} \right)/2  < \infty.
	\end{align*}
\end{itemize}
Then, with the assumptions above,
\begin{align*}
\frac{ 2\min_{1\le i \le p}(\mathbb{E} \lambda_i(t)  )}{ \left( \mathfrak{M}\left( I - f^*_\omega(\theta) \right)  \right)^2 }
\le \mathfrak{m}(f_\Gamma)   \le \mathfrak{M}( f_\Gamma)
\le 	\frac{ 2\max_{1\le i \le p}(\mathbb{E} \lambda_i(t)  )}{ \left( \mathfrak{m}\left( I - f^*_\omega(\theta) \right) \right)^2 } .
\end{align*}
Therefore, with bounded $Q(t)$ by an integrable and non-trivial transition kernel function $k_j(t)$
in Assumption~\ref{assumption4}, we reach the conclusion. Since both $\Lambda$ and $\omega$ are
constants depending only on $\Theta = (\beta_{ij})_{1\le i,j\le p}$,
$\bm{\mu}=(\mu_1,\dots,\mu_p)^\top$,
we have constants $C_1,C_2$ only depending on the model parameter and the transition kernel function such that
\begin{align*}
0 < C_1(\Theta, \mu)  \le \Lambda_{\min}\left( \Upsilon_x \right) \le \Lambda_{\max}\left( \Upsilon_x \right)  \le C_2(\Theta, \mu) < \infty.
\end{align*}

Since $\lambda_i(t)$ is bounded, there exists $c_1, c_2$ such that
$0 < c_1 \le  \sigma^2_i(t) =  \lambda_i(t)(1-\lambda_i(t)) \le  c_2 \le \infty $.
Let $\Upsilon = \textrm{Cov}\left( \bm{x}(t)/\sigma_i(t) \right)$,
\begin{align}
c^{-1}_2 \Upsilon_x \le  \Upsilon  \le c^{-1}_1 \Upsilon_x  .
\end{align}
Notice that
\begin{align*}
\textrm{Cov}( z^*_j(t))^{-1}  = \left( \Upsilon^{-1} \right)_{jj} ,
\end{align*}
which means that $ \textrm{Cov}(  z^*_j(t))^{-1}   $ is a principal submatrix of $  \Upsilon^{-1} $.
Then, by the Cauchy's interlace theorem for eigenvalues of Hermitian matrices,
\begin{align*}
\Lambda_{\min} \left(  \textrm{Cov}( z^*_j(t))^{-1}  \right) &\ge \Lambda_{\min} \left( \Upsilon^{-1}  \right) ,\\
\Lambda_{\max} \left(  \textrm{Cov}( z^*_j(t))^{-1}  \right) &\le \Lambda_{\max} \left( \Upsilon^{-1}  \right) .\\
\end{align*}
Therefore,
\begin{align*}
c_1 \Lambda_{\min} \left( \Upsilon^{-1}_x  \right)    \le    \Lambda_{\min} \left( \Upsilon^{-1}_j  \right)
\le \Lambda_{\max} \left( \Upsilon^{-1}_j  \right) \le c_2 \Lambda_{\max} \left( \Upsilon^{-1}_x  \right),
\end{align*}
which completes the proof.
\end{proof}


%
\begin{lemma}[ Lemma C.1 \citet{Zheng2018} ]
	\label{lem:mart_diff_seq}
	Let $( \xi_{n,i}, \mathcal{H}_{n,i} )_{0\le i \le n}$ be a martingale difference
	sequence taking values in $\mathbb{R}^d$.
	Let
	\[
	X_n^k = \sum_{i=1}^k \xi_{ni}
	\quad\text{and}\quad
	\langle X^n  \rangle_k = \sum_{i=1}^k a_{ni} \equiv \sum_{i=1}^k
	\mathbb{E} \left(   \xi_{ni} \xi_{ni}^\top  \mid \mathcal{H}_{n,i-1}  \right).
	\]
	Define $ R_{\delta}^{n,d} = L_{\delta}^{n,d} + N_{\delta}^{n,d}$, where
	\[
	L_{\delta}^{n,d} = \sum_{i=1}^{n}  \mathbb{E}\lVert \xi_{ni} \rVert_2^{2 + 2\delta}
	\quad\text{and}\quad
	N_{\delta}^{n,d} = \sum_{i=1}^{n}  \mathbb{E}\lVert \langle X^n \rangle_n - I  \rVert_{tr}^{1 + \delta}.
	\]
	If $R_{\delta}^{n,d}  \le 1$, then
	for any $u \in \mathbb{R}^d$, $r \ge 0$, and $0< \delta \le 1/2$,
	we have
	$$
	P\left(  \lVert X_n^n + u \rVert_2 \le r \right) -
	P\left(  \lVert Z + u \rVert_2 \le r \right)
	\le C(\lVert u \rVert_2, d, \delta  )  \left( R_{\delta}^{n,d}  \right)^{\frac{1}{3+ 2\delta}} ,
	$$
	where $Z_{d\times 1} \sim N(0, I)$ and
	$C(\lVert u \rVert_2, d, \delta  )$ is a non-decreasing as $\lVert u \rVert_2$ increases.
\end{lemma}

\begin{lemma}\label{lemma_S0_S}
	Suppose the linear Hawkes model with its intensity function defined in
	\eqref{eq:linear_hawkes_para_transfer} is stationary and satisfies Assumptions~\ref{assumption1} --~\ref{assumption4}. In addition, $(\widehat{\mu}_i,\widehat{\bm{\beta}}_i)$ and $\widehat{w}_j$ satisfy Assumption~\ref{assumption_beta} and~\ref{assumption_w}, respectively. Given $S_{ij}$ in \eqref{eq:S} and $\widehat{S}^0_{ij}$ in \eqref{eq:S_0_hat},
\begin{align*}
\mathbb{P}\left( \lVert \widehat{S}^0_{ij} - S_{ij} \rVert_2  \le C_1 s\rho^2 T^{-4/5} \right)  \ge 1- C_2 p^2 T \exp( - C_3 T^{1/5}) ,
\end{align*}
where $C_k,k=1,\dots, 3$ are constants only depending on the model parameter $(\bm{\mu}, \Theta)$ and the transition kernel function.
\end{lemma}

\begin{proof}

We start with the following decomposition
\begin{multline}
\widehat{S}^0_{ij} - S_{ij} =
\underbrace{( \widehat{\bm{w} }_j  - \bm{w}^*_j )^\top  \frac{1}{T}
	\sum_{t=1}^{T}  \widetilde{\epsilon}_i(t)  \mat{ 1 \\ \bm{z}_{-j}(t) }    }_A
  \\
 + \underbrace{ \frac{1}{T} \sum_{t=1}^{T}  z^*_j(t)
\mat{ 1/\sigma_i(t) & \bm{z}^\top_{-j}(t) }   \left( \mat{ \widehat{\mu}_i \\ \widehat{\bm{\beta}}_{i,-j} }- \mat{ \mu_i \\ \bm{\beta}_{i,-j} } \right)
}_B
 \\
 -\underbrace{  ( \widehat{\bm{w} }_j  - \bm{w}^*_j )^\top
\left (   \frac{1}{T} \sum_{t=1}^{T} \mat{1 \\ \bm{z}_{-j}(t)  }
\mat{1/\sigma_i(t) & \bm{z}^\top_{-j}(t)  }
\right )
\left (
\mat{ \widehat{\mu}_i \\ \widehat{\bm{\beta}}_{i,-j} }- \mat{ \mu_i \\ \bm{\beta}_{i,-j} } \right ) }_C \label{eq:theorem1_bound_S0_S} .
\end{multline}
In the follows, we bound $A,B,C$.
Lemma~\ref{lemma_beta_empirical} and Assumption~\ref{assumption_w} give
$$A
\le
\left \lVert  \widehat{\bm{w} }_j  - \bm{w}^*_j \right \rVert_1 \left \lVert   \frac{1}{T}
\sum_{t=1}^{T}  \widetilde{\epsilon}_i(t)  \mat{ 1 \\ \bm{z}_{-j}(t) }  \right  \rVert_\infty
= O_p( s\rho T^{-4/5}) .
$$
Lemma~\ref{lemma_w_empirical} and Assumption~\ref{assumption_beta} give
$$
B \le
\left \lVert \frac{1}{T} \sum_{t=1}^{T}  z^*_j(t)
\mat{ 1/\sigma_i(t) & \bm{z}^\top_{-j}(t) }  \right \rVert_\infty
\left \lVert \mat{ \widehat{\mu}_i \\ \widehat{\bm{\beta}}_{i,-j} }- \mat{ \mu_i \\ \bm{\beta}_{i,-j} } \right \rVert_1
=O_p( s^{1/2}\rho^2 T^{-4/5}).
$$
Combining Assumption~\ref{assumption_beta} and~\ref{assumption_w} and Lemma~\ref{lemma_bounded_terms} gives
\begin{align*}
C &\le
\left\lVert  \widehat{\bm{w} }_j  - \bm{w}^*_j \right\rVert_1
\left\lVert    \frac{1}{T} \sum_{t=1}^{T} \mat{1 \\ \bm{z}_{-j}(t)  }
\mat{1/\sigma_i(t) & \bm{z}^\top_{-j}(t)  }
\right \rVert_\infty
\left \lVert
\mat{ \widehat{\mu}_i \\ \widehat{\bm{\beta}}_{i,-j} }- \mat{ \mu_i \\ \bm{\beta}_{i,-j} } \right \rVert_1  \\
&= O_p(s\rho^2 T^{-4/5} ) .
\end{align*}

%
%

%

Therefore,
%
\begin{align}
\left| \widehat{S}^0_{ij} - S_{ij} \right| = O_p\left(  s\rho^2 T^{-4/5} \right) ,
\end{align}
which probability at least $1- c_1p^2T\exp(-c_2 T^{1/5})$ ,
where $c_1,c_2$ are constants only depending on the model parameter $(\bm{\mu}, \Theta)$ and the transition kernel function.

\end{proof}

\begin{lemma}\label{lemma_Rhat_R0_R}
	Suppose the linear Hawkes model with its intensity function defined in
	\eqref{eq:linear_hawkes_para_transfer} is stationary and satisfies Assumptions~\ref{assumption1} --~\ref{assumption4}. In addition, $(\widehat{\mu}_i,\widehat{\bm{\beta}}_i)$ and $\widehat{w}_j$ satisfy Assumption~\ref{assumption_beta} and~\ref{assumption_w}, respectively. Then,
	\begin{align*}
     \left 	\lVert   \widehat{\Upsilon}_j - \widehat{\Upsilon}^0_j   \right \rVert_\infty &= O_p( \rho T^{-2/5} ), \\
      \left 	\lVert    \widetilde{\Upsilon}_j - \widetilde{\Upsilon}^0_j \right \rVert_\infty
      &= O_p( \rho T^{-2/5} ),  \\
       \left 	\lVert    \widehat{\Upsilon}^0_j - \Upsilon_j  \right \rVert_\infty &= O_p( s^2\rho^2 T^{-2/5} ),
	\end{align*}
	and
	\begin{align*}
\left 	\lVert  \Upsilon_j^{1/2} \widehat{\Upsilon}_j^{-1}  \Upsilon_j^{1/2} - I  \right \rVert_\infty
	&= O_p( s^2\rho^2 T^{-2/5} ),\\
\left 	\lVert  \Upsilon_j^{1/2} \left (\widehat{\Upsilon}^0_j\right )^{-1}  \Upsilon_j^{1/2} - I \right \rVert_\infty
	&= O_p(  s^2\rho^2 T^{-2/5}) , \\
\left \lVert \left( \widehat{\Upsilon}^0_j \right)^{1/2} \left(  \widehat{\Upsilon}_j \right)^{-1}  \left( \widehat{\Upsilon}^0_j \right)^{1/2} - I  \right \rVert_\infty &= O_p(\rho T^{-2/5}) ,
	\end{align*}
	with probability at least $1- C_1 p^2 T \exp( - C_2 T^{1/5})$,
	where $C_1, C_2$ are constants only depending on the model parameter $(\bm{\mu}, \Theta)$ and the transition kernel function, and
	$\widehat{\Upsilon}^0_j$ and $\widetilde{ \Upsilon}_j $ are defined in \eqref{eq:Upsilon_hat_0},
	 and \eqref{eq:upsilon_hat_tilde}, separately.
\end{lemma}

\begin{proof}

First, applying Talyor expansion on $ \Upsilon_j^{1/2} \widehat{\Upsilon}_j^{-1}  \Upsilon_j^{1/2} - I $ (treating $\widehat{\Upsilon}_j^{-1}$ as variable) at $\Upsilon_j$ leads to
\begin{align}
\Upsilon_j^{1/2} \widehat{\Upsilon}_j^{-1}  \Upsilon_j^{1/2} - I
= \Upsilon_j^{1/2}  \Upsilon_j^{-1}  \Upsilon_j^{1/2} - I  +
\Upsilon_j^{-1} (\widehat{\Upsilon}_j - \Upsilon_j) + o(\widehat{\Upsilon}_j - \Upsilon_j ).
\end{align}
Then, by Proposition~\ref{prop_eigen} that $\Lambda_{\min} \left( \Upsilon_j\right) > 0$,
\begin{align}
\left \lVert  \Upsilon_j^{1/2} \widehat{\Upsilon}_j^{-1}  \Upsilon_j^{1/2} - I  \right \rVert_\infty
=O\left(  \left  \lVert  \Upsilon_j^{-1} (\widehat{\Upsilon}_j - \Upsilon_j \right\rVert_\infty   \right)
 = O\left(  \left \lVert \widehat{\Upsilon}_j - \Upsilon_j\right  \rVert_\infty \right).
 \label{eq:RR_hatR}
\end{align}
In the following, we focus on quantifying the bound of $\lVert  \widehat{\Upsilon}_j - \Upsilon_j \rVert_\infty $.

Recall
\begin{align*}
\widehat{\Upsilon}_j^0 =  \frac{1}{T} \sum_{t=1}^T
\left ( z_j(t)  - \mat{1 & \bm{z}^\top_{-j}(t) } \widehat{\bm{w} }_j \right )^2
\end{align*}
defined in \eqref{eq:Upsilon_hat_0}.
The difference between $\widehat{\Upsilon}_j^0$ and $\widehat{\Upsilon}_j$ is
that we replace $\widehat{ \sigma}_i^2(t)$ by the true value $\sigma^2_i(t)$.
We bound the two parts $\widehat{\Upsilon}^0_j - \Upsilon_j$
and $\widehat{\Upsilon}_j - \widehat{\Upsilon}^0_j$ separately.

First, we have
\begin{align*}
\widehat{\Upsilon}^0_j - \Upsilon_j =&
\left( \frac{1}{T} \sum_{t=1}^T  \left( z^*_j(t) \right)^2   - \Upsilon_j \right) \\
+& (\widehat{\bm{w}}_j - \bm{w}^*_j)^\top  \left( \frac{2}{T} \sum_{t=1}^T
\mat{1 \\ \bm{z}_{-j}(t) }  z^*_j(t)   \right)
\\
+& (\widehat{\bm{w}}_j - \bm{w}^*_j )^\top  \left( \frac{1}{T} \sum_{t=1}^T    \mat{1 \\ \bm{z}_{-j}(t) }
\mat{1 & \bm{z}^\top_{-j}(t) }  \right)(\widehat{\bm{w}}_j - \bm{w}^*_j ) \\
\equiv & E_1 + 2E_2  + E_3 .
\end{align*}
By Lemma~\ref{lemma_bounded_terms} and~\ref{lemma_w_bound_norm},  Assumption~\ref{assumption_w},
\begin{align*}
 \left 	\lVert   E_2 \right \rVert_1 &\le
\left \lVert \widehat{\bm{w}}_j - \bm{w}^*_j  \right \rVert_1
\left \lVert   \frac{2}{T} \sum_{t=1}^T
\mat{1 \\ \bm{z}_{-j}(t) }  z^*_j(t)   \right\rVert_\infty  \\
&=
\left \lVert \widehat{\bm{w}}_j - \bm{w}^*_j  \right \rVert_1
\left \lVert   \frac{2}{T} \sum_{t=1}^T
\mat{1 \\ \bm{z}_{-j}(t) } \mat{1 & \bm{z}(t) }   \right\rVert_\infty
\left ( \lVert\bm{w}^*_j \rVert_1 + 1 \right )\\
&= O_p( s^{3/2}\rho T^{-2/5} ).
\end{align*}
Using Lemma~\ref{lemma_bounded_terms} and Assumption~\ref{assumption_w},
\begin{align*}
 \left\lVert  E_3\right \rVert_1 \le
 \left\lVert \widehat{\bm{w}}_j - \bm{w}^*_j \right \rVert_1
 \left\lVert  \frac{1}{T} \sum_{t=1}^T    \mat{1 \\ \bm{z}_{-j}(t) }
 \mat{1 & \bm{z}^\top_{-j}(t) }  \right\rVert_\infty
 \left \lVert \widehat{\bm{w}}_j - \bm{w}^*_j \right \rVert_1
 =O_p(s^2\rho^2 T^{-4/5} ) .
\end{align*}

Finally, we bound $E_1$.
Recall that $\bm{w}_j$ is defined such that $z^*_j(t) =  \mat{1 & \bm{z}^\top (t)}\bm{w}_j$,
where $\bm{w}_j = (w^*_{j0} , \{ w^*_{jl}\mathbf{1}(l\ne j) + \mathbf{1}(l = j)\}_{1\le l \le p} )^\top \in \mathbb{R}^{p+1}$
and $\lVert \bm{w}_j \rVert_1  =\lVert \bm{w}^*_j \rVert_1  + 1 $. Then,
\begin{align*}
\left ( z^*_j(t) \right)^2
&= \bm{w}^\top_j \mat{ 1 & \bm{z}^\top (t) } \mat{1 \\ \bm{z}(t) } \bm{w}_j .
\end{align*}
Using Lemma~\ref{lemma_w_bound_norm} that $\lVert \bm{w}^*_j \rVert^2_1 = O(s)$
and Lemma~\ref{lemma_concen_ineq_z},
\begin{align*}
| E_1  | &
\le
\lVert  \bm{w}_j \rVert^2_1
\left\lVert  \frac{1}{T} \sum_{t=1}^T  \mat{1 \\ \bm{z}(t) }\mat{ 1 & \bm{z}^\top (t) }  -
\mathbb{E}  \mat{1 \\ \bm{z}(t) }\mat{ 1 & \bm{z}^\top (t) }  \right\rVert_\infty
\\
&= O_p \left( s \rho T^{-2/5}  \right).
\end{align*}
Therefore, combining the bounds of $E_1,E_2,E_3$,
\begin{align}
| \widehat{\Upsilon}^0_j - \Upsilon_j | = O_p\left( s^2\rho^2 T^{-2/5}  \right) \label{eq:R0_R},
\end{align}
with probability at least $1- C_1 p^2 T \exp( - C_2 T^{1/5})$.

Next, we bound $\widehat{ \Upsilon}^0_j -\widehat{ \Upsilon}_j $.
Letting $D_i = \frac{\widehat{\sigma}_i(t)}{ \sigma_i(t)} $, we have
\begin{align*}
\widehat{ \Upsilon}^0_j -\widehat{ \Upsilon}_j
&= \frac{1}{T}\sum_{t=1}^T  (D^2_i -1 ) (\widehat{z}_j(t) - \widehat{\bm{z}}^\top_{-j}(t)\widehat{ \bm{w}}_{j,-j} )^2
- 2 (D_i -1 ) \widehat{w}_{j0}\left( \widehat{z}_j(t) - \widehat{\bm{z}}^\top_{-j}(t)\widehat{ \bm{w}}_{j,-j} \right).
\end{align*}
Recall that $\sigma^2_i(t) = \lambda_i(t)(1-\lambda_i(t))$
and $\lambda_i(t)= \mu_i + \bm{x}^\top (t)\bm{\beta}_i$.
By Lemma~\ref{lemma_bounded_terms}, $\sigma^2_i(t) = O(1)$.
By Lemma~\ref{lemma_bounded_terms} that $\bm{x}(t)=O(1)$,
and $\left \lVert   \mat{ \mu_i \\ \bm{\beta}_i} -
\mat{ \widehat{\mu}_i \\ \widehat{\bm{\beta}}_i}
\right\rVert_1 = O_p\left( \rho T^{-2/5} \right)$ in Assumption~\ref{assumption_beta},
\begin{align*}
| \lambda_i(t) -\widehat{\lambda}_i(t) |
&= \left | \mat{1  & \bm{x}(t) } \left(  \mat{ \mu_i \\ \bm{\beta}_i} -
 \mat{ \widehat{\mu}_i \\ \widehat{\bm{\beta}}_i}
  \right) \right | \\
 &\le
  \left \lVert  \mat{1  & \bm{x}(t) } \right\rVert_\infty
  \left \lVert   \mat{ \mu_i \\ \bm{\beta}_i} -
  \mat{ \widehat{\mu}_i \\ \widehat{\bm{\beta}}_i}
  \right\rVert_1
= O_p\left( \rho T^{-2/5} \right).
\end{align*}
Thus,
\begin{align*}
| D^2_i -1 |
&=
\left | \frac{ \widehat{\sigma}_i^2(t) - \sigma_i^2(t) }{  \sigma_i^2(t) } \right |   \\
&=
\left | \frac{ \widehat{\lambda}_i(t) -\widehat{\lambda}^2_i(t)
-  \lambda_i(t) + \lambda^2_i(t) }{  \sigma_i^2(t) } \right |   \\
&=
\left | \frac{ \left(\lambda_i(t) -\widehat{\lambda}_i(t) \right)
				\left(1 - 2\lambda_i(t)+ \lambda_i(t) -\widehat{\lambda}_i(t) \right)
   			 }
             {  \sigma_i^2(t) } \right |   \\
&\le  \frac{1}{  \sigma_i^2(t) }
  \left| \lambda_i(t) -\widehat{\lambda}_i(t)  \right|
\left(\left| 1 - 2\lambda_i(t) \right| + \left| \lambda_i(t) -\widehat{\lambda}_i(t) \right| \right)
\\
&= O\left(\lambda_i(t) -\widehat{\lambda}_i(t)   \right)
= O_p\left( \rho T^{-2/5} \right) .
\end{align*}
Using a similar deduction,
we get $| D_i -1 | = O_p\left( \rho T^{-2/5} \right) $.
In addition, the estimation error bounds of $\widehat{\bm{w}}_j$ and $(\widehat{\mu}_i , \widehat{\bm{\beta}}_i)$ in Assumptions~\ref{assumption3} and~\ref{assumption4}
imply that $\widehat{z}_j(t) - \widehat{\bm{z}}^\top_{-j}(t) \widehat{ \bm{w}}_{j,-j} $
and $ \widehat{w}_{j0}\left( \widehat{z}_j(t) - \widehat{\bm{z}}^\top_{-j}(t)\widehat{ \bm{w}}_{j,-j} \right)$ are bounded in probability.
Combining these results,
\begin{align} \label{eq:Uhat0_Uhat_bound}
\left| \widehat{ \Upsilon}^0_j -\widehat{ \Upsilon}_j  \right| = O_p\left( \rho T^{-2/5}\right) ,
\end{align}
with probability at least $1-C_3p^2 T \exp( - C_4  T ^{1/5})$.
Therefore, using \eqref{eq:RR_hatR},
\begin{align*}
\lVert  \Upsilon_j^{1/2} \widehat{\Upsilon}_j^{-1}  \Upsilon_j^{1/2} - I \rVert_\infty
&= O\left(  \lVert \widehat{\Upsilon}_j - \Upsilon_j \rVert_\infty \right)\\
&= O\left(  \lVert  \widehat{ \Upsilon}_j -\widehat{ \Upsilon}^0_j \rVert_\infty \right)
+  O\left(  \lVert  \widehat{ \Upsilon}^0_j - \Upsilon_j \rVert_\infty \right) \\
&= O_p\left( s^2\rho^2 T^{-2/5}\right) ,
\end{align*}
with probability at least $1-C_5p^2 T \exp( - C_6  T ^{1/5})$


Following a similar deduction and separately using the bounds of
$ \lVert  \widehat{ \Upsilon}_j -\widehat{ \Upsilon}^0_j \rVert_\infty $
and  $ \lVert  \widehat{ \Upsilon}^0_j - \Upsilon_j \rVert_\infty$ derived in the above,
we get the bounds for
$\left \lVert \left( \widehat{\Upsilon}^0_j \right)^{1/2} \left(  \widehat{\Upsilon}_j \right)^{-1}  \left( \widehat{\Upsilon}^0_j \right)^{1/2} - I  \right \rVert_\infty = O_p(\rho T^{-2/5}) $
and $\left \lVert  \Upsilon_j^{1/2} \left( \widehat{\Upsilon}^0_j\right)^{-1}  \Upsilon_j^{1/2} - I \right \rVert_\infty = O_p\left( s^2\rho^2 T^{-2/5} \right)$, respectively.
%

In addition, following a similar deduction as \eqref{eq:Uhat0_Uhat_bound}, we can show the same probabilistic bound holds for $\left| \widetilde{ \Upsilon}^0_j -\widetilde{ \Upsilon}_j  \right|$ (except for a change in constants), where $\widetilde{ \Upsilon}_j $ is defined in \eqref{eq:upsilon_hat_tilde} and $\widetilde{ \Upsilon}^0_j$ is the same as $\widetilde{ \Upsilon}_j$ except it involves the true $\sigma_i(t)$ in its construction.
\end{proof}

\begin{lemma}\label{lemma_U0_U}
Suppose the stationary linear Hawkes model defined in  (\ref{eq:linear_hawkes})
satisfies Assumption~\ref{assumption1}-\ref{assumption4}. In addition, $(\widehat{\mu}_i,\widehat{\bm{\beta}}_i)$ and $\widehat{w}_j$ satisfy Assumption~\ref{assumption_beta} and~\ref{assumption_w}, respectively. Under the null hypothesis and the alternative with $\phi \ge 1/2$,
\begin{align}
\mathbb{P}\left( | \widehat{U}_T - \widehat{U}^0_T |  > C_1  \rho T^{-1/5} \right)
\le C_2 p^2 T \exp( - C_3 T^{1/5}) + C_4 T^{-1/8} + C_5 s^2\rho^2 T^{-1/5} ,
\end{align}
where $C_k,k=1,\dots,5$ are constants only depending on the model parameter $(\bm{\mu}, \Theta)$ and the transition kernel function. $\widehat{U}_T $ and $ \widehat{U}^0_T$ are defined in \eqref{eq:U_hat_T} and \eqref{eq:U_0_hat}, respectively.
\end{lemma}

\begin{proof}
We first give a proof under the null hypothesis when $\beta_{ij}=0$.
We extend the proof for the alternative hypothesis setting at the end.

Similar to the proof of Theorem~\ref{theorem1}, we
quantify $\widehat{ U}_T^0 - U_T$ (see equation \eqref{eq:bound_U0_UT_null}).
To study the difference between $\widehat{ U}^0_T$ and $\widehat{U}_T$ as shown in \eqref{eq: U_T0_UT},
it is sufficient to bound $\widehat{S}_{ij} - \widehat{S}^0_{ij}$
and $\widehat{\Upsilon}_j - \widehat{\Upsilon}^0_j$, where $\widehat{S}^0_{ij}$, $\widehat{\Upsilon}^0_j$
are defined in \eqref{eq:S_0_hat} and \eqref{eq:U_0_hat}.

The bound for $ \widehat{\Upsilon}_j - \widehat{\Upsilon}^0_j$ is given by Lemma~\ref{lemma_Rhat_R0_R}.
We focus on quantifying the difference between  $\widehat{S}_{ij}$ and $\widehat{S}^0_{ij} $.
We have
\begin{align*}
&\widehat{S}_{ij} - \widehat{S}^0_{ij} \\
&=
\frac{1}{T}\sum_{t=1}^{T}  \frac{1}{\widehat{\sigma}_i(t)} \left( Y_i(t) - \widehat{\mu}_i - \bm{x}^\top_{-j}(t)\widehat{\bm{\beta} }_{i,-j} \right)
\left( x_j(t)/\widehat{\sigma}_i(t) - \widehat{w}_{j0} - \bm{x}^\top_{-j}(t)/\widehat{\sigma}_i(t) \widehat{ \bm{w}}_{j,-j}  \right)  \\
&-
\frac{1}{T}\sum_{t=1}^{T}  \frac{1}{\sigma_i(t)} \left( Y_i(t) - \widehat{\mu}_i - \bm{x}^\top_{-j}(t)\widehat{\bm{\beta} }_{i,-j} \right)
\left(x_j(t)/\sigma_i(t)- \widehat{w}_{j0} - \bm{x}^\top_{-j}(t)/\sigma_i(t) \widehat{ \bm{w}}_{j,-j}  \right) \\
&=  \frac{1}{T}\sum_{t=1}^{T}
\frac{\sigma^2_i(t)}{\widehat{\sigma}^2_i(t)}
\frac{1}{\sigma_i(t)} \left( Y_i(t) - \widehat{\mu}_i - \bm{x}^\top_{-j}(t)\widehat{\bm{\beta} }_{i,-j} \right)
\left( x_j(t)/\sigma_i(t)  - \widehat{w}_{j0} \widehat{\sigma}_i(t)/\sigma_i(t) - \bm{x}^\top_{-j}(t)/\sigma_i(t) \widehat{ \bm{w}}_{j,-j}  \right)  \\
&-
\frac{1}{T}\sum_{t=1}^{T}  \frac{1}{\sigma_i(t)} \left( Y_i(t) - \widehat{\mu}_i - \bm{x}^\top_{-j}(t)\widehat{\bm{\beta} }_{i,-j}  \right)
\left(x_j(t)/\sigma_i(t)- \widehat{w}_{j0} - \bm{x}^\top_{-j}(t)/\sigma_i(t) \widehat{ \bm{w}}_{j,-j}  \right) \\
&=
\frac{1}{T}\sum_{t=1}^{T} \left(
\frac{\sigma^2_i(t)}{\widehat{\sigma}^2_i(t)}  - 1
\right)  \frac{1}{\sigma_i(t)} \left( Y_i(t) - \widehat{\mu}_i - \bm{x}^\top_{-j}(t)\widehat{\bm{\beta} }_{i,-j} \right)
\left(x_j(t)/\sigma_i(t)- \widehat{w}_{j0} - \bm{x}^\top_{-j}(t)/\sigma_i(t) \widehat{ \bm{w}}_{j,-j}  \right) \\
&+ \frac{1}{T}\sum_{t=1}^{T}
\frac{\sigma^2_i(t)}{\widehat{\sigma}^2_i(t)}
\frac{1}{\sigma_i(t)} \left( Y_i(t) - \widehat{\mu}_i - \bm{x}^\top_{-j}(t)\widehat{\bm{\beta} }_{i,-j} \right)
\widehat{w}_{j0}\left( 1 - \frac{\widehat{\sigma}_i(t)}{ \sigma_i(t) }
\right)\\
&=
\frac{1}{T}\sum_{t=1}^{T} \left(
\frac{\sigma^2_i(t)}{\widehat{\sigma}^2_i(t)}  - 1
\right)  \frac{1}{\sigma_i(t)} \left( Y_i(t) - \widehat{\mu}_i -\bm{x}^\top_{-j}(t)\widehat{\bm{\beta} }_{i,-j} \right)
\left( x_j(t)/\sigma_i(t)- \widehat{w}_{j0} - \bm{x}^\top_{-j}(t)/\sigma_i(t) \widehat{ \bm{w}}_{j,-j}  \right) \\
&+ \frac{1}{T}\sum_{t=1}^{T}
\frac{\sigma^2_i(t)}{\widehat{\sigma}^2_i(t)}
\frac{1}{\sigma_i(t)} \left( Y_i(t) - \widehat{\mu}_i - \bm{x}^\top_{-j}(t)\widehat{\bm{\beta} }_{i,-j}  \right)
\widehat{w}_{j0} \frac{\sigma^2_i(t) -\widehat{\sigma}^2_i(t)}{ \sigma_i(t) \left(
	\sigma_i(t) + \widehat{\sigma}_i(t)	 \right) }\\
&\equiv  A + B \label{eq:Shat_S0} .
\end{align*}
For ease of notation, let
\begin{gather*}
\bm{\eta} = \mat{\mu_i \\ \bm{\beta}_i},\quad
\widehat{\bm{\eta}} = \mat{\widehat{\mu}_i \\ \widehat{\bm{\beta}}_i},
\quad\text{and}\quad
\bm{x}_t =\mat{1 & \bm{x}^\top(t)}^\top.
\end{gather*}
Then,
\begin{align}
\sigma^2_i(t) -\widehat{\sigma}^2_i(t)
&= \bm{x}^\top_t\bm{\eta} (1- \bm{x}^\top_t\bm{\eta}) - \bm{x}^\top_t \widehat{\bm{\eta}} (1- \bm{x}^\top_t \widehat{\bm{\eta}}) = (\bm{\eta} - \widehat{\bm{\eta}})^\top  \bm{x}_t  \left( 1 - \bm{x}^\top_t( \widehat{\bm{\eta}} + \bm{\eta} )   \right).
\label{eq:sigma_hat_sigma_diff}
\end{align}
Since $\lVert \bm{x}_t \rVert_\infty$ is
bounded by Assumption~\ref{assumption3},
we have
\begin{align*}
\lVert \sigma^2_i(t) -\widehat{\sigma}^2_i(t)  \rVert_\infty
\le \lVert  \bm{\eta} - \widehat{\bm{\eta}}\rVert_1  \lVert \bm{x}_t \rVert_\infty \lVert 1 -  \bm{x}^\top_t( \widehat{\bm{\eta}} + \bm{\eta} )  ) \rVert_\infty = O \left( \lVert  \bm{\eta} - \widehat{\bm{\eta}}\rVert_1   \right)
=O_p( \rho T^{-2/5} ),
\end{align*}
where the bound of $\lVert  \bm{\eta} - \widehat{\bm{\eta}}\rVert_1 $ is given in Assumption~\ref{assumption_beta}.

Let
\begin{gather*}
\bm{\eta}_{-j} = \mat{\mu_i \\ \bm{\beta}_{i,-j} }
\quad\text{and}\quad
\bm{x}^{-j}_t =\mat{1 \\ \bm{x}_{-j}(t)}.
\end{gather*}
Under the null hypothesis that $\beta_{ij}=0$,
\begin{align*}
Y_i(t) - \widehat{\mu}_i - \bm{x}^\top_{-j}(t)\widehat{\bm{\eta} }_{i,-j}
&= \epsilon_i(t) +
(\mu_i - \widehat{\mu}_i ) + \bm{x}^\top_{-j}(t)( \bm{\eta}_{i,-j} - \widehat{\bm{\eta} }_{i,-j} ) \\
&=  \epsilon_i(t) +  \left(\bm{x}^{-j}_t \right)^\top( \bm{\eta}_{-j} - \widehat{\bm{\eta}}_{-j} ) .
\end{align*}
Let $C_i(t) = \left( 1 - \bm{x}^\top_t ( \widehat{\bm{\eta}} + \bm{\eta} )   \right)
\frac{1}{\widehat{\sigma}^2_i(t) \sigma_i(t)}
\left(x_j(t)/\sigma_i(t)- \widehat{w}_{j0} - \bm{x}^\top_{-j}(t)/\sigma_i(t) \widehat{ \bm{w}}_{j,-j}  \right)$.
Then, we write part A as follows:
\begin{align*}
A &=
\frac{1}{T}\sum_{t=1}^{T}
\left( \sigma^2_i(t) - \widehat{\sigma}^2_i(t)  \right)
\epsilon_i(t)
\frac{1}{\widehat{\sigma}^2_i(t) \sigma_i(t)}
\left(x_j(t)/\sigma_i(t)- \widehat{w}_{j0} - \bm{x}^\top_{-j}(t)/\sigma_i(t) \widehat{ \bm{w}}_{j,-j}  \right) \\
&+
\frac{1}{T}\sum_{t=1}^{T}
\left( \sigma^2_i(t) - \widehat{\sigma}^2_i(t)  \right)
\bm{x}^{-j}_t( \bm{\eta}_{-j} - \widehat{ \bm{\eta}}_{-j} )
\frac{1}{\widehat{\sigma}^2_i(t) \sigma_i(t)}
\left(x_j(t)/\sigma_i(t)- \widehat{w}_{j0} - \bm{x}^\top_{-j}(t)/\sigma_i(t) \widehat{ \bm{w}}_{j,-j}  \right) \\
&=
(\bm{\eta} - \widehat{\bm{\eta}})^\top  \frac{1}{T}\sum_{t=1}^{T}
\bm{x}^\top_t
\epsilon_i(t)
C_i(t)
 +
(\bm{\eta} - \widehat{\bm{\eta}})^\top  \left( \frac{1}{T}\sum_{t=1}^{T}
\bm{x}^\top_t
C_i(t)
\bm{x}^{-j}_t
\right)  ( \bm{\eta}_{-j} - \widehat{ \bm{\eta}}_{-j} )    \\
&\equiv A_1 + A_2 .
\end{align*}
By the estimation consistency of $\widehat{w}_j$ on $w_j^*$ and $(\widehat{\mu}_i,\widehat{\bm{\beta}}_i)$ on $(\mu_i,\bm{\beta}_i)$ in Assumptions~\ref{assumption_beta} and~\ref{assumption_w} as well as the bounded $\sigma^2(t)$ and $x_j(t)$ implied by Assumptions~\ref{assumption3} and~\ref{assumption4}, $C_i(t)$ is bounded in probability. Thus, applying Lemma~\ref{lemma_beta_empirical} and under Assumption~\ref{assumption_beta},
\begin{align*}
| A_1 |
\le  \lVert\widehat{\bm{\eta}} -\bm{\eta} \rVert_1
\left \lVert \frac{1}{T}\sum_{t=1}^{T}
\bm{x}^\top_t  \epsilon_i(t) C_i(t) \right \rVert_\infty \le O_p\left( \rho T^{-4/5} \right),
\end{align*}
with probability at least $1- C_1 p \exp(-C_2 T^{1/5})$.
Similarly,
\[
| A_2 | \le   |C_i(t)| \lVert x_j(t) \rVert^2_\infty \lVert \widehat{\bm{\eta}} -\bm{\eta} \lVert_2^2 \le  O_p\left( \rho T^{-4/5} \right) .
\]

Next, let $C_i'(t) =  \left( 1 - \bm{x}^\top_t ( \widehat{\bm{\eta}} + \bm{\eta} )   \right)
\frac{ \widehat{w}_{j0} }{ \widehat{\sigma}^2_i(t) \left(
	\sigma_i(t) + \widehat{\sigma}_i(t)	 \right)} $. Then,
\begin{align*}
B &=  \frac{1}{T}\sum_{t=1}^{T}
\left( \epsilon_i(t) + \bm{x}^\top_t(\bm{\eta} -  \widehat{\bm{\eta}})   \right)
\widehat{w}_{j0} \frac{\sigma^2_i(t) -\widehat{\sigma}^2_i(t)}{ \widehat{\sigma}^2_i(t) \left(
	\sigma_i(t) + \widehat{\sigma}_i(t)	 \right) }\\
&= ( \bm{\eta} - \widehat{\bm{\eta}})^\top  \frac{1}{T}\sum_{t=1}^{T} \epsilon_i(t)
\bm{x}_t C'(t)  \\
&+ (\bm{\eta} - \widehat{\bm{\eta}})^\top  \left( \frac{1}{T}\sum_{t=1}^{T}
\bm{x}_t C'(t)
(\bm{x}^{-j}_t )^\top
\right)( \bm{\eta}_{-j} - \widehat{ \bm{\eta}}_{-j} ) \\
&\equiv B_1 + B_2  .
\end{align*}
By the estimation consistency of $\widehat{w}_j$ on $w_j^*$ and $(\widehat{\mu}_i,\widehat{\bm{\beta}}_i)$ on $(\mu_i,\bm{\beta}_i)$ in Assumptions~\ref{assumption_beta} and~\ref{assumption_w} as well as Lemma~\ref{lemma_bounded_terms} that $\sigma^2(t)$ and $x_j(t)$ are bounded, we get $C'_i(t)=O_p(1)$. In addition, applying Lemma~\ref{lemma_vandergeer1995},
\begin{align*}
| B_1 |
\le  \lVert\widehat{\bm{\eta}} -\bm{\eta} \rVert_1
\left \lVert \frac{1}{T}\sum_{t=1}^{T}
\bm{x}^\top_t  \epsilon_i(t) C'_i(t) \right \rVert_\infty = O_p\left( \rho T^{-4/5} \right),
\end{align*}
with probability at least $1- C_3 p \exp(-C_4 T^{1/5})$.
Similarly,
\begin{align*}
| B_2 | \le   |C'_i(t)| \lVert \widehat{\bm{\eta}} -\bm{\eta} \lVert_1 \lVert \bm{x}_t \bm{x}^\top_t \rVert_\infty \lVert \widehat{\bm{\eta}}_{-j} -\bm{\eta}_{-j} \rVert_1  \le  O_p\left( \rho^2 T^{-4/5} \right) .
\end{align*}
Then, $ | B |  = O_p \left( \rho T^{-4/5} \right) $.

Therefore, taking the bound of $A$ and $B$ back to \eqref{eq:Shat_S0},
\begin{align}
\left | \widehat{S}_{ij} - \widehat{S}^0_{ij} \right |
= O_p\left( \rho T^{-4/5} \right ), \label{eq:S_S0_diff}
\end{align}
with probability at least $1-C_5 \exp(-C_6 T^{1/5})$.

We follow a similar deduction as \eqref{eq:bound_U0_UT_null} and
let $E = \sqrt{T} \left( \widehat{\Upsilon}^0_j \right)^{-1/2} (\widehat{S}_{ij}- \widehat{S}^0_{ij} )$. Then,
\begin{align}
| \widehat{U}^0_T - \widehat{U}_T |  &\le
\lVert E \rVert_2^2 +  2 \lVert \widehat{V}_T \rVert_2 \lVert E \rVert_2 +
\left \lVert\left( \widehat{\Upsilon}^0_j \right)^{1/2} \left( \widehat{\Upsilon}_j \right)^{-1}  \left( \widehat{\Upsilon}^0_j \right)^{1/2} - I  \right \rVert_\infty
\left ( \lVert \widehat{V}_T \rVert_2 + \lVert E \rVert_2 \right )^2 .
\end{align}
Using the consistency of $\widehat{\Upsilon}_j$ to $\Upsilon_j$ in Lemma~\ref{lemma_Rhat_R0_R} and \eqref{eq:S_S0_diff}, $\lVert E \rVert_2^2  = O_p\left( \rho^2 T^{-3/5} \right )$.
Lemma~\ref{lemma_Rhat_R0_R} gives $\left \lVert \left( \widehat{\Upsilon}^0_j \right)^{1/2} \left(  \widehat{\Upsilon}_j \right)^{-1}  \left( \widehat{\Upsilon}^0_j \right)^{1/2} - I  \right \rVert_\infty = O_p(\rho T^{-2/5}) $.
Using an intermediate result shown in Theorem~\ref{theorem1} that $\lVert V^0_T \rVert^2_2$ weakly converges to $\chi_d^2$ in \eqref{eq:U_hat_0_to_chisq} and then applying Lemma~\ref{lemma_VT_ineq}, we get
$$ \mathbb{P}(\lVert \widehat{V}^0_T \rVert_2 > T^{1/10} ) \le  C_7\rho^2T\exp(-C_8 T^{1/5})+  C_9 s^2\rho^2 T^{-1/5} + C_{10} T^{-1/8}. $$

Therefore,
\begin{align}\label{eq:U_hat_U_0_hat_null}
\mathbb{P}\left( | \widehat{U}_T - \widehat{U}^0_T |  > C_{11}  \rho T^{-1/5} \right)
\le C_{12} p^2 T \exp( - C_{13} T^{1/5}) + C_{14} s^2\rho^2 T^{-1/5} + C_{15} T^{-1/8}   ,
\end{align}
which establishes the result under the null hypothesis.

Next, we examine the difference between $\widehat{U}_T$
and $\widehat{U}^0_T$ under the alternative hypothesis setting,
($\beta_{ij} = \Delta T^{-\phi}$). Here we modify $A_1$ and $B_1$,
where we replace $\epsilon_i(t)$ by $\epsilon_i(t) + x_j(t)\beta_{ij}$
when bounding $\lVert \widehat{S}_{ij} - \widehat{S}^0_{ij}  \rVert_2 $.

Notice that for  $\phi \ge \frac{1}{2}$, the order of the upper bounds
of $A_1$, $B_1$ remain unchanged. Using the estimation error of
$(\mu_i,\bm{\beta}_i)$ given in Assumption~\ref{assumption_beta},
\begin{align*}
 | A_1 | \le  \lVert\widehat{\bm{\eta}} -\bm{\eta} \rVert_1
\left \lVert \frac{1}{T}\sum_{t=1}^{T}
\bm{x}^\top_t  \epsilon_i(t) C_i(t) \right \rVert_\infty
+ \lVert\widehat{\bm{\eta}} -\bm{\eta} \rVert_1 \lVert x_j(t) \rVert^2_\infty |\beta_{ij}|   \\
= O_p\left( \rho T^{-4/5} \right) + O_p \left(  T^{-2/5 -\phi} \Delta \right) = O_p\left( \rho T^{-4/5} \right),
\end{align*}
with probability at least $1-C_1 p^2 T\exp(-C_2 T^{1/5})$.
Similarly, we show part $B_1$ is bounded in probability with the same order of bound as before.
 Therefore, under the alternative hypothesis with $\phi \ge \frac{1}{2}$, $\left| \widehat{U}_T - \widehat{U}^0_T \right |$ is bounded by the same order of upper bound with a similar probability (except for a change in constants) as that under the null.

\end{proof}

\begin{lemma}\label{lemma_w_empirical}
	Suppose the linear Hawkes model with its intensity function defined in
	\eqref{eq:linear_hawkes_para_transfer} is stationary and satisfies Assumptions~\ref{assumption1} --~\ref{assumption4}. Then,
	\begin{align*}
	P\left (  \left \lVert
	\frac{1}{T} \sum_{t=1}^{T}
	z^*_j(t)
	\mat{ 1/\sigma_i(t) & \bm{z}^\top_{-j}(t) }
	\right	\rVert_\infty > C_1 (\sqrt{s}+1) \rho T^{-2/5}  \right )
	\le  C_2 p^2 T \exp( - C_3 T^{1/5}) ,
	\end{align*}
	where $C_1,C_2,C_3$ are constants only depending on the model parameter $(\bm{\mu}, \Theta)$ and the transition kernel function.
\end{lemma}

\begin{proof}

We separately bound
\[
\frac{1}{T} \sum_{t=1}^{T} z^*_j(t)  z_k(t), \text{ for }k\ne j,
\quad\text{ and }\quad
\frac{1}{T} \sum_{t=1}^T \frac{1}{\sigma_i(t)}z^*_j(t).
\]
Notice that $x_k(t)  = \left( \bm{x}(t) \right)^\top e_k $, where $e_k$
is the $k$th canonical basis vector.
Let $\bm{w}_j$ be such that $z^*_j(t) =  \mat{1 & \bm{z}^\top (t)}\bm{w}_j$,
where $\bm{w}_j = \left (w^*_{j0} , \{ w^*_{jl}\mathbf{1}(l\ne j) + \mathbf{1}(l = j)\}_{1\le l \le p} \right )^\top \in \mathbb{R}^{p+1}$ and $\lVert \bm{w}_j \rVert_1  =\lVert \bm{w}^*_j \rVert_1  + 1 $.
Then,
\begin{align*}
\frac{1}{T} \sum_{t=1}^T  z^*_j(t) z_k(t)
&=
\bm{w}^\top_j  \left( \frac{1}{T} \sum_{t=1}^T  \mat{1 \\ \bm{z}(t)}   \bm{z}^\top (t) \right) e_k.
\end{align*}
Note that $\mathbb{E}\left( \frac{1}{T} \sum_{t=1}^T  z^*_j(t) z_k(t)   \right) = 0,
k\ne j$, by the construction of $z^*_j(t)$ in \eqref{eq:x_tilde_star}.
Then,
\begin{align*}
\left \lVert
\frac{1}{T} \sum_{t=1}^T  z^*_j(t) z_k(t)
\right \rVert_\infty
&= \left \lVert
\frac{1}{T} \sum_{t=1}^T  z^*_j(t) z_k(t)
-\mathbb{E}\left( \frac{1}{T} \sum_{t=1}^T  z^*_j(t) z_k(t)   \right)
\right \rVert_\infty    \\
&\le \lVert \bm{w}_j \rVert_1
\left \lVert
\frac{1}{T} \sum_{t=1}^T
\mat{1 \\ \bm{z} (t)}   \bm{z}^\top(t)
- \mathbb{E}\left(
\mat{1 \\ \bm{z} (t)}   \bm{z}^\top(t)  \right)
\right \rVert_\infty  \lVert e_k \rVert_1 \\
&\le \left(\lVert \bm{w}_j^* \rVert_1 + 1 \right )
\left \lVert
\frac{1}{T} \sum_{t=1}^T
\mat{1 \\ \bm{z} (t)}   \bm{z}^\top(t)
- \mathbb{E}\left(
\mat{1 \\ \bm{z} (t)}   \bm{z}^\top(t)  \right)
\right \rVert_\infty  \lVert e_k \rVert_1 \\
&\le
O_p\left(  (\sqrt{s_j}+ 1)\rho T^{-2/5}  \right) ,
\end{align*}
with probability at least $1- C_1 p^2 T\exp(-C_2 T^{1/5})$,
where the last inequality is by Lemma~\ref{lemma_w_bound_norm} and~Lemma~\ref{lemma_concen_ineq_z} (taking $\epsilon = T^{-2/5}$).

For the second part, by the construction of $z^*_j(t)$ in \eqref{eq:x_tilde_star},
$$\mathbb{E} \left(  \frac{1}{\sigma_i(t)} z^*_j(t) \right)
=  \mathbb{E} \left[  \frac{1}{\sigma_i(t)}  \mathbb{E}
\left(  z^*_j(t) \lvert \bm{x}(t)  \right) \right]  = 0.$$
Then,
\begin{align*}
\left\lVert \frac{1}{T} \sum_{t=1}^T  \frac{1}{\sigma_i(t)} z^*_j(t) \right\rVert_\infty
&= \left \lVert
\frac{1}{T} \sum_{t=1}^T  z^*_j(t) \frac{1}{\sigma_i(t)}
-\mathbb{E}\left( \frac{1}{T} \sum_{t=1}^T  z^*_j(t) \frac{1}{\sigma_i(t)}   \right)
\right \rVert_\infty  \\
&\le
\left(\lVert \bm{w}_j^* \rVert_1 + 1 \right )
\left \lVert
\frac{1}{T} \sum_{t=1}^T
\mat{1/\sigma_i(t) \\ \bm{z} (t)/\sigma_i(t)}
- \mathbb{E}\left(
\mat{1/\sigma_i(t) \\ \bm{z} (t)/\sigma_i(t)}    \right)
\right \rVert_\infty.
\end{align*}
Lemma~\ref{lemma_w_bound_norm} implies $\lVert \bm{w}_j^* \rVert_1 = O(\sqrt{s})$.
We bound $ \lVert \frac{1}{\sigma_i(t)} - \mathbb{E}\frac{1}{\sigma_i(t)}  \rVert_\infty $
using a similar steps as did in Lemma~\ref{lemma_concen_ineq_z}. That is,
we treat $\sigma_i(t)$ as a function of $\bm{x}(t)$ and
then apply Talyor expansion on $\frac{1}{\sigma_i(t)} - \mathbb{E}\frac{1}{\sigma_i(t)} $ at $\E \bm{x}(t)$.
As a result,
\begin{align*}
\left \lVert \frac{1}{\sigma_i(t)} - \mathbb{E}\frac{1}{\sigma_i(t)} \right  \rVert_\infty
&=
O\left( \rho \left \lVert \frac{1}{T} \sum_{t=1}^T   \bm{x}^\top (t) \bm{x}(t)  - \mathbb{E}\left( \bm{x}(t)^\top  \bm{x}(t) \right) \right \rVert_\infty \right)\\
&+
 O \left( \rho
\left \lVert \frac{1}{T} \sum_{t=1}^T  \bm{x}(t) - \mathbb{E}\left(\bm{x}(t) \right)  \right \rVert_\infty \right) \\
&= O_p(\rho T^{-2/5}),
\end{align*}
where the last equality is by Lemma~\ref{lemma_concen_ineq_x} on the deviation bound of $\bm{x}(t)$.

To derive the bound for
$ \left \lVert \frac{1}{T} \sum_0^T
\bm{z} (t)/\sigma_i(t) -  \E \bm{z} (t)/\sigma_i(t) \right \rVert_\infty $, we repeat the steps of \eqref{eq:xt_Ext_expand} and \eqref{eq:xt_Ext_dev_bound} in Lemma~\ref{lemma_concen_ineq_x},
 but with $f_1(s) = \frac{1}{\sigma^2_i(t)}\int_s^T k_j(t-s) dt$. With $\sigma_i(t) =O(1)$ by Lemma~\ref{lemma_bounded_terms},
\begin{align*}
\left \lVert
\frac{1}{T} \sum_0^T
\bm{z} (t)/\sigma_i(t) -  \E \bm{z} (t)/\sigma_i(t)  \right \rVert_\infty = O_p( T^{-2/5}).
\end{align*}

Combining the above,
\begin{align*}
\left\lVert \frac{1}{T} \sum_{t=1}^T  \frac{1}{\sigma_i(t)} z^*_j(t) \right\rVert_\infty
&= O_p\left(  (\sqrt{s_j} +1 ) \rho T^{-2/5} \right) ,
\end{align*}
with probability at least $1-C_3 p^2 T \exp( - C_4 T^{1/5})$.

\end{proof}

\begin{lemma}\label{lemma_beta_empirical}
	Suppose the linear Hawkes model with its intensity function defined in
	\eqref{eq:linear_hawkes_para_transfer} is stationary and satisfies Assumptions~\ref{assumption1} --~\ref{assumption4}.
	Then
	\begin{align}
	P\left(  \left \lVert  \frac{1}{T}
	\sum_{t=1}^{T} \widetilde{\epsilon}_i(t)\mat{ 1 \\ \bm{z}_{-j}(t) }
	\right \rVert_\infty > C_1 T^{-2/5} \right)
	\le C_2 p \exp( -   T^{1/5} ),
	\end{align}
	where $C_1,C_2$ are constants only depending on the model parameter $(\bm{\mu}, \Theta)$ and the transition kernel function.
\end{lemma}
%

\begin{proof}
The concentration inequality for the linear Hawkes process has been discussed
\citet{Shizhe2017}. The result is a direct application of the martingale
inequality by \citet[][Theorem 3.1]{vandegeer1995} stated in
Lemma~\ref{lemma_vandergeer1995}.

Using Lemma~\ref{lemma_vandergeer1995}, we reach the conclusion by taking
$\epsilon = T^{-4/5}$, and separately set $H(t) = \frac{1}{\sigma_i(t)} z_k(t)$,
$H(t)=\frac{1}{\sigma_i(t)}$, for $k\ne j$, where $H(t)$ is bounded by
Assumption~\ref{assumption3} and~\ref{assumption4} of a bounded intensity
function and an integrable transition kernel.

\end{proof}

\begin{lemma}\label{lemma_bounded_terms}
	Under Assumptions~\ref{assumption3} and~\ref{assumption4},
	for $1\le i,j \le p$ and $\forall t \in [0,T]$,
	\begin{gather*}
	 x_j(t) = O(1),\ \
	 0 < \sigma^2_i(t)  = O(1)  ,\ \
	 z_j(t)   = O(1)   ,\ \
	 \epsilon_i(t) = O(1),\\
	 \left\lVert \frac{1}{T} \sum_{t=1}^T    \mat{1 \\ \bm{x}(t) }
	 \mat{1 & \bm{x}^\top(t) }  \right\rVert_\infty  = O(1) ,\ \
	\left\lVert \frac{1}{T} \sum_{t=1}^T    \mat{1 \\ \bm{z}(t) }
	\mat{1 & \bm{z}^\top(t) }  \right\rVert_\infty = O(1)  ,
	\end{gather*}
where $x_j(t), \sigma^2_i(t)$, $z_j(t)$ and $\epsilon_i(t)$ are defined
in Section~\ref{sec:hawkes}, and  $\bm{x}(t)= ( x_1(t), \dots, x_p(t) )^\top$ and
$\bm{z}(t)= ( z_1(t), \dots, z_p(t) )^\top$.
\end{lemma}
\begin{proof}
The result follows directly from Assumptions~\ref{assumption3} and~\ref{assumption4}.
\end{proof}

\begin{lemma}\label{lemma_VT_ineq}
Let $ \lVert V_T \rVert_2^2 \sim \chi^2_d$, then for a constant $\Delta$,
\begin{align}
P\left(   \left\lVert V_T + \Delta  \right\rVert_2  > y  \right)
&\le  y^{-2} \label{eq:P_V_T} .
\end{align}
\end{lemma}
\begin{proof}
The conclusion is by the tail bound on $\chi^2$ distribution given in \citet{Zheng2018}(see Theorem~3.2).
\end{proof}


\begin{lemma}\label{Lemma_beta_lasso}
	Let
	$H= \frac{1}{T}\sum_{t=1}^{T} \mat{1\\\bm{x}(t)} \mat{1 & \bm{x}^\top(t) } $.
	 Suppose the linear Hawkes model with its intensity function defined in
	\eqref{eq:linear_hawkes_para_transfer} is stationary and satisfies Assumptions~\ref{assumption1} --~\ref{assumption4}. In addition, $(\widehat{\mu}_i,\widehat{\bm{\beta}}_i )$ are given in \eqref{eq:lasso_beta}. Then, taking $\lambda = O\left( T^{-2/5} \right)$ and assuming $\sqrt{\rho} \vee \log p = o\left(T^{1/5} \right)$,  $\forall i=1,\dots,p$,
	\begin{align*}
	\left \lVert \mat{\widehat{ \mu}_i \\ \widehat{\bm{\beta}}_i} - \mat{ \mu_i \\ \bm{\beta}_i} \right  \rVert_2
	&\le  C_1 \sqrt{ \rho +1  } T^{-2/5} \\
	\left(
	\mat{\widehat{ \mu}_i \\ \widehat{\bm{\beta}}_i} - \mat{ \mu_i \\ \bm{\beta}_i}
	\right)^\top  H
	\left(
	\mat{\widehat{ \mu}_i \\ \widehat{\bm{\beta}}_i} - \mat{ \mu_i \\ \bm{\beta}_i}
	\right)
	 &\le C_1  \sqrt{ \rho + 1}   T^{-2/5}  \\
	\left \lVert  	\mat{\widehat{ \mu}_i \\ \widehat{\bm{\beta}}_i} - \mat{ \mu_i \\ \bm{\beta}_i} \right \rVert_1 &\le  C_1  (\rho + 1)  T^{-2/5} ,
	\end{align*}
	with probability at least $1-C_2p^2 T \exp(- C_3 T^{1/5})$,
	where $C_1,C_2,C_3$ are constants only depending on the model parameter $(\bm{\mu}, \Theta)$ and the transition kernel function, and $\rho_i = \lVert \bm{\beta}_i \rVert_0$ and $\rho = \max_{1\le i\le p} \rho_i$.
\end{lemma}

\begin{proof}

The proof follows a typical framework for the analysis of lasso-type estimators
\citep{Peter2011,negahban2012}. The crucial difficulty in the proof is showing
that the key conditions are satisfied for the linear Hawkes process. To be
specific, we start with the basic inequality by the construction of the lasso
estimator in \eqref{eq:lasso_beta}. We then bound the prediction error of the
lasso regression using the results of Lemma~\ref{lemma_beta_empirical}. In
addition, we show that the restricted eigen-value (RE) condition is satisfied
with high probability in the setting with the multivariate Hawkes process, where
the proof essentially uses Proposition~\ref{prop_eigen} and
Lemma~\ref{lemma_concen_ineq_x}. In what follows $C_k, k=1,\dots,5$ are
constants only depending on the model parameter $(\bm{\mu}, \Theta)$ and the
transition kernel function.

By the construction of the lasso estimator in \eqref{eq:lasso_beta},
we have
\begin{align}
\frac{1}{T} \sum_{t=1}^{T} \left ( Y_i(t) - \widehat{\mu}_i - \bm{x}(t)\widehat{ \bm{\beta} }_i \right  )^2 + \lambda \lVert \widehat{ \bm{\beta} }_i \rVert_1
\le
\frac{1}{T} \sum_{t=1}^{T} \left  ( Y_i(t) - \mu_i -  \bm{x}^\top(t)\bm{\beta}_i  \right )^2 + \lambda \lVert \bm{\beta}_i \rVert_1. \label{eq:basic_ineq}
\end{align}
Let
$u=\widehat{\mu}_i - \mu_i$ and $\bm{v} = \widehat{ \bm{\beta} }_i - \bm{\beta}_i$.
Define $S =\{ j: \beta_{ij} \ne 0 \}$ and $S^c =\{ j: \beta_{ij} = 0 \}$. It follows
from \eqref{eq:basic_ineq} that
\begin{align*}
\mat{ u & \bm{v}^\top }   H 	\mat{u \\ \bm{v} }
& \le
2 \left \lVert  \frac{1}{T} \sum_{t=1}^{T} \epsilon_i(t) \mat{  1 \\   \bm{x}(t)   }   \right \rVert_\infty
\left \lVert \mat{u \\ \bm{v} } \right \rVert_1 + \lambda \lVert \bm{v}_S \rVert_1 - \lambda \lVert \bm{v}_{S^c} \rVert_1.
\end{align*}
Taking $\lambda =  4 \left \lVert  \frac{1}{T} \sum_{t=1}^{T} \epsilon_i(t) \mat{  1 \\   \bm{x}(t)   }   \right \rVert_\infty $,
we have
\begin{align*}
0\le \mat{ u & \bm{v}^\top }   H 	\mat{u \\ \bm{v} }  &\le
\frac{3\lambda}{2} \lVert \bm{v}_S  \rVert_1 	 - \frac{\lambda}{2} \lVert \bm{v}_{S^c} \rVert_1
+ \frac{1}{2}\lambda\lVert u \rVert_1
\le
\frac{3\lambda}{2} \lVert  \mat{ u & \bm{v}_S }  \rVert_1 	 - \frac{\lambda}{2} \lVert \bm{v}_{S^c} \rVert_1 .
\end{align*}
Let $\bm{\eta} = \mat{u \\ \bm{ v} } \in \mathbb{R}^{p+1}$. Define $\bm{\eta}_s =(u, \bm{v}_S )$ and $\bm{\eta}_{s^c} =(\bm{v}_{S^c}) $. Then,
\begin{align*}
0 \le \bm{\eta}^T H \bm{\eta} \le  	\frac{3\lambda}{2} \lVert \bm{\eta}_s  \rVert_1 	 - \frac{\lambda}{2} \lVert \bm{\eta}_{s^c} \rVert_1
\quad\text{ and }\quad
\lVert \bm{\eta}_{s^c} \rVert_1 \le 3 \lVert \bm{\eta}_s  \rVert_1 .
\end{align*}
Let $\mathcal{C}(J, \kappa) = \{ \bm{\eta} : \lVert  \bm{\eta}_{J^c} \rVert_1 \le \kappa \lVert \bm{\eta}_J \rVert_1 \}$.
Thus, $\bm{\eta} = \mathcal{C}(S, 3)$.
Further, denote
$\check{\Gamma} = \mathbb{E} H$. Then,
\begin{align*}
&\inf\{ \bm{\eta}^\top  H \bm{\eta} : \bm{\eta} \in \mathcal{C}(J, \kappa), \lVert \bm{\eta} \rVert_2 \le 1 \} \\
&\ge
\inf\{ \bm{\eta}^\top \check{\Gamma} \bm{\eta} : \bm{\eta} \in \mathcal{C}(J, \kappa), \lVert \bm{\eta} \rVert_2 \le 1 \}
- \sup \left\{
\left | \bm{\eta}^\top \left (H- \check{\Gamma} \right )\bm{\eta}  \right|: \bm{\eta} \in \mathcal{C}(J, \kappa),\lVert \bm{\eta} \rVert_2 \le 1	 \right\} \\
&\ge
\inf\{ \bm{\eta}^\top \check{\Gamma} \bm{\eta} :  \lVert \bm{\eta} \rVert_2 \le 1 \}
- \sup_{\lVert \bm{\eta} \rVert_2 \le 1 } \left\{
\left | \bm{\eta}^\top \left (H- \check{\Gamma} \right )\bm{\eta}  \right|	 \right\} \\
&\ge  \Lambda_{\min}\left(\check{\Gamma}\right)
- \sup_{\lVert \bm{\eta} \rVert_2 \le 1 } \left\{
\left | \bm{\eta}^\top \left (H- \check{\Gamma} \right )\bm{\eta}  \right|	 \right\} .
\end{align*}
Proposition~\ref{prop_eigen} shows that $\Lambda_{\min}\left(\textrm{Cov}\left(
\bm{x}(t)\right) \right) > 0  $, which implies that the $p$-unit multivariate
process $\{ x_j(t)\}_{1\le j \le p}$ are not linearly dependent. In addition, by
Assumption~\ref{assumption4}, the process $x_j(t)$ is not a trivial process of
constants. We conclude the minimum eigenvalue of $\check{\Gamma}$ is strictly
positive. Otherwise, there exists $\eta \in \mathbb{R}^{p+1}$ and $\lVert \eta
\rVert_1 > 0$, s.t., $\eta^\top \mat{\mathbf{1} \\   \bm{x}(t) } = 0$. This
implies either $\{x_j(t)\}_{j=1}^p$ are linearly dependent or $x_j(t)$ is a
trivial process of constants, which leads to a contradiction.

Using Lemma~\ref{lemma_concen_ineq_x}, with $\epsilon = -2/5$,
\begin{align*}
\left\{
\left | \bm{\eta}^\top \left (H- \check{\Gamma} \right )\bm{\eta}  \right|: \lVert \bm{\eta} \rVert_2 \le 1	 \right\} \le \lVert \bm{\eta} \rVert^2_2 \left \lVert H- \check{\Gamma}   \right \rVert_\infty
= O_p\left( T^{-2/5}  \right),
\end{align*}
with probability at least
$1- C_1 p^2 T \exp(- C_2 T^{1/5})$. Thus, assuming $\sqrt{\rho} \vee \log p = o(T^{1/5})$,
`there exists a constant $C_3$ such that
when $T> C_3 $,
\begin{align*}
\sup_{\lVert \bm{\eta} \rVert_2 \le 1 } \left\{
\left | \bm{\eta}^\top \left (H- \check{\Gamma} \right )\bm{\eta}  \right|	 \right\} \ge \frac{1}{2}\Lambda_{\min}\left(\check{\Gamma}\right).
\end{align*}
Thus,
\begin{align*}
\inf\{ \bm{\eta}^\top  H \bm{\eta} : \bm{\eta} \in \mathcal{C}(J, \kappa), \lVert \bm{\eta} \rVert_2 \le 1 \}
\ge \Lambda_{\min}\left(\check{\Gamma}\right)
- \frac{1}{2} \Lambda_{\min}\left(\check{\Gamma}\right)
\ge \frac{1}{2} \Lambda_{\min}\left(\check{\Gamma}\right) ,
\end{align*}
with probability at least $1-C_1p^2 T \exp(-C_2 T^{1/5})$. Then, with the same probability,
\begin{multline*}
\frac{1}{2} \Lambda_{\min}\left(\check{\Gamma}\right) \left \lVert \mat{ u & \bm{v}^\top }  \right \rVert_2^2
\le
\mat{ u & \bm{v}^\top }   H \mat{ u \\ \bm{v}} \\
\le
\frac{3\lambda}{2}  \left  \lVert  \mat{ u & \bm{v}_S^\top } \right \rVert_1
\le \frac{3}{2}
\lambda
  \sqrt{\rho_i +1 } \left \lVert \mat{ u & \bm{v}^\top } \right \rVert_2.
\end{multline*}
At last, using Lemma~\ref{lemma_beta_empirical} to bound
$\lambda = 4\left \lVert  \frac{1}{T} \sum_{t=1}^{T} \epsilon_i(t) \mat{  1 \\   \bm{x}(t)   }   \right \rVert_\infty $,
 we have
\begin{align*}
\lVert \bm{\eta}  \rVert_2
&\le  C_4 \sqrt{ \rho_i +1 } T^{-2/5}, \\
\bm{\eta}^\top  H \bm{\eta} &\le C_4  \sqrt{ \rho_i + 1}   T^{-2/5},  \\
\lVert  	\bm{\eta} \rVert_1 &\le 4 \lVert 	\bm{\eta}_s \rVert_1 \le 4 \sqrt{\rho_i+1} \lVert 	\bm{\eta}_s \rVert_2  \le C_4  (\rho_i + 1)  T^{-2/5},
\end{align*}
with probability at least $1-C_5p^2 T \exp(- C_6 T^{1/5})$.

\end{proof}

\begin{lemma}\label{Lemma_w_lasso}
	Let
	$\widehat{H}= \frac{1}{T}\sum_{t=1}^{T} \mat{1\\ \widehat{\bm{z}}(t)} \mat{1 & \widehat{\bm{z}}^\top(t) } $.
		 Suppose the linear Hawkes model with its intensity function defined in
	\eqref{eq:linear_hawkes_para_transfer} is stationary and satisfies Assumptions~\ref{assumption1} --~\ref{assumption4}.
	In addition, $\widehat{\bm{w}}_j$ is given in \eqref{eq:lasso_w_hat}. Then, taking $\lambda = O\left( \sqrt{s+1}\rho T^{-2/5} \right)$ and assuming $ \sqrt{(s+1)\rho} \vee \log p = o (T^{1/5})$,
	\begin{align*}
	\lVert  \widehat{\bm{w}}_j - \bm{w}_j \rVert_2
	&\le  C_1 \sqrt{s +1} \rho T^{-2/5} \\
	\left(  \widehat{\bm{w}}_j - \bm{w}_j  \right)^\top  \widehat{H}
	\left( \widehat{\bm{w}}_j - \bm{w}_j \right) &\le C_1  \sqrt{s + 1} \rho  T^{-2/5}  \\
	\lVert  	 \widehat{\bm{w}}_j - \bm{w}_j  \rVert_1 &\le  C_1 (s + 1)\rho  T^{-2/5} ,
	\end{align*}
	with probability at least $1-C_2p^2 T \exp(- C_3 T^{1/5})$,
	where $C_1,C_2,C_3$ are constants only depending on the model parameter $(\bm{\mu}, \Theta)$ and the transition kernel function, and
	$s_j = \lVert\bm{w}_j^* \rVert_0$ and $s = \max_{1\le j\le p} s_j$; $\rho_i
	= \lVert \bm{\beta}_i \rVert_0$ and $\rho = \max_{1\le i\le p} \rho_i$.
\end{lemma}
%


\begin{proof}

    The proof is similar to the steps in Lemma~\ref{Lemma_beta_lasso} except that we use
    Lemma~\ref{lemma_w_empirical} to bound
    $\left \lVert
    \frac{1}{T} \sum_{t=1}^{T}
    \widehat{z}^*_j(t)
    \mat{ 1  & \widehat{\bm{z}}^\top_{-j}(t) }
    \right	\rVert_\infty $,
    and use Lemma~\ref{lemma_concen_ineq_z} to verify the RE condition for $\widehat{H}= \frac{1}{T}\sum_{t=1}^{T} \mat{1\\ \widehat{\bm{z}}(t)} \mat{1 & \widehat{\bm{z}}^\top(t) }$.

\end{proof}

\begin{lemma}\label{lemma_sparsity_s_rho}
	$s_j = \lVert \bm{w}_j^* \rVert_0$ and $s = \max_{1\le j\le p} s_j$; $\rho_i = \lVert \bm{\beta}_i \rVert_0$ and $\rho = \max_{1\le i\le p} \rho_i$.
	 Suppose the linear Hawkes model with its intensity function defined in
	\eqref{eq:linear_hawkes_para_transfer} is stationary and satisfies Assumptions~\ref{assumption1} --~\ref{assumption4}.
	For the connectivity matrix of block structure, 	$ s \le 2\rho + 1 $.
\end{lemma}
\begin{proof}
	By the choice of $\bm{w}_j^* $ in (\ref{eq: w_j_star}),
	\begin{align*}
	\textrm{Cov}\left( z_j(t) - \bm{z}^\top _{-j}(t) \bm{w}_{j,-j}^* , \bm{z}_{-j}(t) \right) &= 0
	\end{align*}
    This leads to
	\begin{align*}
	\bm{w}_{j,-j}^* &= \textrm{Cov}\left(  z_j(t) , \bm{z}_{-j}(t) \right)
	\left(   \textrm{Cov}\left(  \bm{z}_{-j}(t) , \bm{z}_{-j}(t) \right) \right)^{-1} \\
	&= \textrm{Cov}\left(  z_j(t) , \bm{z}_{-j}(t) \right) \left( \left( \Upsilon \right)_{-j,-j} \right)^{- 1} ,
	\end{align*}
	where $\Upsilon_{-j,-j} = \textrm{Var}\left( \bm{z}_{-j} \right)$.
	Therefore,
	\begin{align*}
	\lVert \bm{w}_{j,-j}^* \rVert_0  \le
	\lVert \{ k: j\ne k, \textrm{Cov}\left(  z_j(t) , z_k(t) \right) \ne 0  \} \rVert_0.
	\end{align*}
	Recall that $z_j(t) = x_j(t)/\sigma_i(t)$ and $\sigma_i(t) = \lambda_i(t)(1-\lambda_i(t))$,where $\lambda_i(t) = \mu_i + \bm{x}^\top(t)\bm{\beta}_i$.
	Notice that by the sparsity of $\bm{\beta}_i$ tat $\lVert \bm{\beta}_i \rVert_0 \le \rho $, the number of $x_j(t)$'s that $\lambda_i(t)$ depends on is at most $\rho$. Let $S =\{ j: \bm{\beta}_{i,j} \ne 0 \}$ and $S^c =\{ j: \bm{\beta}_{i,j} = 0 \}$. Thus,
	\begin{align*}
	\lVert \bm{w}_{j,-j}^* \rVert_0
	&\le
	\lVert \{ k: k \ne j\quad\text{and}\quad k \in S^c, \textrm{Cov}\left(  z_j(t) , z_k(t) \right) \ne 0  \} \rVert_0 + \rho \\
    &\le \lVert \{ k: k \ne j\quad\text{and}\quad k \in S^c, \textrm{Cov}\left(  x_j(t) , x_k(t) \right) \ne 0  \} \rVert_0 + \rho.
	\end{align*}
	Recalling that $x_j(t) = \int_0^t k_j(t-s) d N_j(s) $,
	\begin{align*}
	\textrm{Cov}\left(  x_j(t) , x_k(t) \right)
	&= \textrm{Cov}\left( \int_0^t  k_j(t-s) dN_j(s) ,\int_0^t  k_j(t-s) dN_j(s) \right) \\
	&= \int_0^t \int_0^t   k_j(t-s) k_j(t-r)
	\textrm{Cov}\left(   dN_j(s) ,   dN_j(r) \right).
	\end{align*}
	Therefore, with a positive transition kernel in Assumption~\ref{assumption4},
	\begin{multline*}
	\left \lVert \{ k: k \ne j\quad\text{and}\quad k \in S^c, \textrm{Cov}\left(  x_j(t) , x_k(t) \right) \ne 0  \} \right \rVert_0\\
	\le
	\left \lVert \{ k: k \ne j\quad\text{and}\quad k \in S^c, 	\textrm{Cov}\left(   dN_j(s) ,   dN_j(r) \right) \ne 0  \} \right \rVert_0 .
	\end{multline*}
	Consider a connectivity matrix of block structure, for each $j$, all units that the unit $j$ depends on must stay in one of the blocks on the connectivity matrix. Therefore,
	the possible number of units it depends on is at most $\rho$; that is,
	\begin{align*}
	\left \lVert \{ k: j\ne k, \textrm{Cov}\left(   dN_j(s) ,   dN_j(r) \right) \ne 0  \} \right \rVert_0 \le \rho ,
	\end{align*}
	which implies
	\begin{align*}
	s =\max_{ 1 \le j \le p}  \lVert \bm{w}_j^* \rVert_0  \le  1 + \max_{ 1 \le j \le p} \lVert \bm{w}_{j,-j}^* \rVert_0      \le 2\rho + 1 .
	\end{align*}
	%
\end{proof}



\begin{lemma}\label{lemma_w_bound_norm}
  	Suppose the linear Hawkes model with its intensity function defined in
  \eqref{eq:linear_hawkes_para_transfer} is stationary and satisfies Assumptions~\ref{assumption1} --~\ref{assumption4}.  Then,
	\begin{align*}
	\lVert  \bm{w}^*_j \rVert_2^2 \le  C
  \quad\text{and}\quad
	\lVert \bm{w}^*_j \rVert_1 \le  \sqrt{s_j C},
	\end{align*}
	where $C$ is a constant only depending on the model parameter $(\bm{\mu}, \Theta)$ and the transition kernel function.
\end{lemma}
\begin{proof}

%
Recall that  $\Upsilon = \textrm{Cov}\left( \bm{z}(t) \right)$ whose eigenvalue is
bounded by $\Lambda_{\max}\left(\Upsilon_x \right)$ and $\Lambda_{\min}\left(\Upsilon_x \right)$ since $\sigma^2_i(t)$ is bounded implied by Assumptions~\ref{assumption3} and~\ref{assumption4}.
Let $\Upsilon_{-j,-j} = \textrm{Var}\left( \bm{z}_{-j} \right)$. Then,
$\Lambda_{\max}\left( \Upsilon_{-j,-j} \right) \le \Lambda_{\max}\left( \Upsilon \right) $.
In addition, let $ \Upsilon_{j,-j} = \textrm{Cov}\left( z_j, \bm{z}_{-j} \right)$. Then,
$\lVert \Upsilon_{j,-j}  \rVert^2_2 \le \lVert \Upsilon_{j,.} \rVert^2_2 =  \left( \Upsilon^2 \right)_{j,j} \le \Lambda_{\max}\left(\Upsilon^2\right) \le \Lambda^2_{\max}\left(\Upsilon\right)$.
Thus,
\begin{align*}
\lVert  \bm{w}^*_j \rVert_2^2  &= 1 + \lVert  \bm{w}^*_{j,-j} \rVert_2^2 \\
&\le 1 + \left( \Lambda_{\max} \left(  \Upsilon^{-1}_{-j,-j} \right) \right )^2 \lVert \Upsilon_{j,-j} \rVert_2^2 \\
&\le 1 + \left( \Lambda_{\min} \left ( \Upsilon \right) \right)^{-2}   \Lambda^2_{\max}(\Upsilon) \le C ,
\end{align*}
where the last inequality is by Proposition~\ref{prop_eigen} for some constant $C$ only depend on $(\Theta, \bm{\mu})$ and the transition kernel function.
Then,
\begin{align*}
\lVert \bm{w}^*_j \rVert_1 \le \sqrt{s_j} \lVert \bm{w}^*_j \rVert_2 \le \sqrt{s_j C} .
\end{align*}
\end{proof}


\begin{lemma}[Theorem 3, \citet{Shizhe2017}]
	\label{lemma_SZ_concen_ineq}
	Suppose the linear Hawkes model with its intensity function defined in
	\eqref{eq:linear_hawkes_para_transfer} is stationary and satisfies Assumptions~\ref{assumption1} --~\ref{assumption3}.
	For $\mathcal{H}_t$-predictable function $f(\cdot)$ that is bounded, let
	$$
	y_{ij} = \frac{1}{T} \int_0^T \int_0^T f(t , t') dN_i(t) dN_j(t').
	$$
	Then there exists constants $C_1, C_2 > 0$ such that
	$$
	\mathbb{P} \left( \left| y_{ij} - \mathbb{E} y_{ij} \right| > \epsilon \right) \le
	C_1 T \exp\left( - C_2 (\epsilon T)^{1/3} \right) .
	$$
\end{lemma}
\begin{proof}
	The conditions required by \citet{Shizhe2017} are satisfied by our Assumptions~\ref{assumption1} and~\ref{assumption3}, where the Assumption 2 in \citet{Shizhe2017} is satisfied with $b_0=0$
	by the Assumption~\ref{assumption2} in our case. Thus, taking $r \rightarrow \infty$, Lemma~\ref{lemma_SZ_concen_ineq} is a direct result following the proof of Theorem 3 in \citet{Shizhe2017}.
\end{proof}

\begin{lemma}[\citet{vandegeer1995}]
	\label{lemma_vandergeer1995}
	Suppose the linear Hawkes model with its intensity function defined in
	\eqref{eq:linear_hawkes_para_transfer} is stationary and there exists $\lambda_{\max} $ such that $\lambda_i(t) \le \lambda_{\max}$ for all $1\le i \le p$. Let $H(t)$ be a bounded function that is $\mathcal{H}_t$-predictable.
	Then, for any $\epsilon > 0$,
	\begin{align*}
	\frac{1}{T} \int_0^T H(t) \bigg \{ \lambda_i(t) dt - dN_i(t) \bigg \}
	\le 4 \bigg \{ \frac{\lambda_{\max} }{2T }   \int_0^T H^2(t) dt \bigg \}  ^{1/2} \epsilon^{1/2},
	\end{align*}
	with probability at least $1 - C\exp(-\epsilon T)$, for some constant $C$.
\end{lemma}

\begin{lemma}
    \label{lemma_concen_ineq_x}
	Suppose the linear Hawkes model with its intensity function defined in
	\eqref{eq:linear_hawkes_para_transfer} is stationary and satisfies Assumptions~\ref{assumption1} --~\ref{assumption4}.
   Then, for $\delta > 0$ and $1\le i,j \le p$,
\begin{align*}
\mathbb{P} \left (\left| \frac{1}{T} \sum_{t=1}^T  x_j(t) - \mathbb{E} x_j(t)  \right | >  C_1 \epsilon^{1/2} \right )
\le  C_2 \exp(-  \epsilon T )
\end{align*}
and
\begin{align*}
\mathbb{P} \left( \left|\frac{1}{T} \sum_{t=1}^T     x_i(t) x_j(t)     - \mathbb{E}   x_i(t) x_j(t)    \right| > \epsilon \right) \le C_3 T \exp(-C_4 (\epsilon T)^{1/3}) ,
\end{align*}
where $C_k,k=1,\dots, 4$ are constants only depending on the model parameter $(\bm{\mu}, \Theta)$ and the transition kernel function.
\end{lemma}

\begin{proof}
The concentration bound for the first order statistics of $x_j(t), 1\le j \le p$
is a direct application of  Lemma~\ref{lemma_vandergeer1995}.
Notice that
\begin{align}
\frac{1}{T} \int_0^T  \left( x_j(t) - \mathbb{E} x_j(t) \right) dt
&= \frac{1}{T} \int_0^T  \int_0^t k_j(t-s) (dN_j(s) - \lambda_j(s)ds  ) dt  \nonumber  \\
&=  \frac{1}{T} \int_0^T  \int_s^T k_j(t-s) dt  (dN_j(s) - \lambda_j(s) ds) \label{eq:xt_Ext_expand},
\end{align}
where the last equality is by Fubini's theorem. Let $f_1(s) = \int_s^T k_j(t-s) dt$.
Since the transition kernel is integrable, $f_1(s)$ is bounded.
 In addition, $\lambda_j(t)$ is bounded by $\lambda_{\max}$ by Assumption~\ref{assumption3}.
 Thus, by Lemma~\ref{lemma_vandergeer1995},
\begin{align}
\mathbb{P} \left (\left| \frac{1}{T} \sum_{t=1}^T  x_j(t) - \mathbb{E} x_j(t)  \right | >  C_1 \epsilon^{1/2} \right )
\le  C_2 \exp(-  \epsilon T ) \label{eq:xt_Ext_dev_bound} .
\end{align}
%
Next, we consider the second order statistics of $x_j(t)$.
Let $f_2(s, r ) = \int_{\max\{s, r \} }^T k_i(t-s) k_j(t-r) dt$.
Then,
\begin{align*}
\frac{1}{T} \int_0^T   x_i(t) x_j(t)  dt
&= \frac{1}{T} \int_0^T   dt   \int_0^t \int_0^t k_i(t-s) k_j(t-r) dN_i(s) dN_j(r) \\
&=   \frac{1}{T} \int_0^T   dN_i(s)  \int_s^T \int_0^T   k_j(t-r)  dN_j(r) \\
&=  \frac{1}{T} \int_0^T \int_0^T  f_2(s, r )  dN_i(s) dN_j(r)  ,
\end{align*}
where the second and third equalities are based on Fubini's theorem.
By  Assumption~\ref{assumption4},   $k_i(t)$ is integrable.
Therefore, $|f_2(s,r)| \le \left (  \int_0^T k_i(t ) dt \right)^2 < \infty$ is bounded.
Applying Lemma~\ref{lemma_SZ_concen_ineq},
\begin{align*}
\mathbb{P} \left( \left|\frac{1}{T} \int_0^T   x_i(t) x_j(t)     - \mathbb{E}   x_i(t) x_j(t)  dt  \right| > \epsilon \right) \le C_3 T \exp(-C_4 (\epsilon T)^{1/3}) .
\end{align*}
Since the observations are in discrete time,
 we replace the integral above by the numerical integration to reach the conclusion.
\end{proof}

\begin{lemma}\label{lemma_concen_ineq_z}
Suppose the linear Hawkes model with its intensity function defined in
\eqref{eq:linear_hawkes_para_transfer} is stationary and satisfies Assumptions~\ref{assumption1} --~\ref{assumption4}.
Then,
	\begin{align*}
	\mathbb{P} \left (\left\lVert \frac{1}{T} \sum_{t=1}^T  \bm{z}(t) - \mathbb{E} \bm{z}(t)  \right \rVert >  C_1 \epsilon^{1/2} \right )
	\le  C_2 p \exp(-  \epsilon T )
	\end{align*}
	and
	\begin{align*}
	\mathbb{P}  \left(
	\left \lVert
	\frac{1}{T} \sum_{t=1}^{T}  \bm{z}^\top(t)   \bm{z}(t)  -
	\mathbb{E}  \bm{z}^\top(t)  \bm{z}(t)
	\right \rVert_\infty
	>   \rho \epsilon
	\right)
	&\le  C_3 p^2 T \exp(-C_4 (\epsilon T)^{1/3} )   .
		\end{align*}
\end{lemma}

\begin{proof}
	The proof is essentially based on the concentration inequality of the first and second order statistics of $\bm{x}(t)$ established in Lemma~\ref{lemma_concen_ineq_x}. The difference is that this lemma focuses on the scaled design column, $\bm{z}(t) = \bm{x}(t)/\sigma_i(t)$ when testing $\beta_{ij}$. We denote $\sigma_i(t)$ as $\sigma(t)$ for short from now on. The extra term, $\sigma(t)$, makes the technical proof challenging. Fortunately, since $\sigma^2(t)$ depends on $\bm{x}(t)$ and $\bm{x}(t)\bm{x}^\top(t)$, we can expand the term $\bm{z}(t)=\bm{x}(t)/\sigma(t)$ using Taylor expansion in the combinations of $\bm{x}(t)$ and $\bm{x}(t)\bm{x}^\top(t)$, and then apply  Lemma~\ref{lemma_concen_ineq_x} to establish a similar concentration inequality for $\bm{z}(t)$. In general, for any function $h(\bm{x}(t),\bm{\eta})$ as a function of $\bm{x}(t)\in \mathbb{R}^p$ and $\bm{\eta}\in \mathbb{R}^{p+1}$ that is second-order differentiable with bounded derivatives w.r.t $\bm{x}(t)$, the term $\frac{1}{T}\sum_{t=1}^T h(\bm{x}(t),\bm{\eta}) - \mathbb{E} h(\bm{x}(t),\bm{\eta})$ can be bounded using the concentration inequality of the first- and second order statistics of $\bm{x}(t)$ in Lemma~\ref{lemma_concen_ineq_x}.

	First, let $h_{jk}$ be a function of $\bm{x}(t)$ and $\bm{\eta}=(\mu,\bm{\beta}^\top)^\top$,
	\begin{align*}
	h_{jk}\left (\bm{x}(t), \bm{\eta}=(\mu,\bm{\beta}^\top)^\top \right) &\equiv \frac{1}{\sigma^2(t)} x_j(t) x_k(t)
	=\frac{1}{\sigma^2(t)} \bm{x}^\top(t) I_{jk} \bm{x}(t).
	\end{align*}
	Here $I_{jk}=\bm{e}_j \bm{e}_k^\top $ and $\bm{e}_i$ are canonical basis vector,
  while, with a little abuse of notation, $\sigma^2(t)$
  is a function of $\bm{\eta}=(\mu,\bm{\beta}^\top)^\top$
  and $\bm{x}(t)$, where $\sigma^2(t)= \lambda(t)(1-\lambda(t)) $ and $\lambda(t) = \mu + \bm{x}^\top(t)\bm{\beta}$.

	The derivative of $h_{jk}$ w.r.t $\bm{x}(t)$ is:
	\begin{multline*}
	h_{jk}'\left(\mathbb{E}\left(\bm{x}(t) \right), \bm{\eta} \right)
	= \sigma^{-4}(t) (1 - 2(\mu + ( \mathbb E \bm{x}(t) )^\top \bm{\beta}) )\bm{\beta}
	\mathbb{E}\left(\bm{x}^\top(t) \right) I_{jk} \mathbb{E}\left(\bm{x}(t) \right) \\
	+  \sigma^{-4}(t) \mathbb{E}\left(\bm{x}(t) \right) (I_{jk}+I^T_{jk}) .
	\end{multline*}
	Let $\rho = \lVert \bm{\beta} \rVert_0 $.
  Since $\lambda(t)$ and $\bm{x}(t)$ are bounded,
  as implied by Assumption~\ref{assumption3} and~\ref{assumption4},
	\begin{align*}
	\left \lVert \frac{1}{T} \sum_{t=1}^T  h_{jk}'(\mathbb{E}\left(\bm{x}(t) \right),\bm{\eta})\left( \bm{x}(t) - \mathbb{E}\left(\bm{x}(t) \right) \right)^\top
	\right \rVert_\infty
& = O\left( \rho
	\left \lVert \frac{1}{T} \sum_{t=1}^T   \left( \bm{x}(t) - \mathbb{E}\left(\bm{x}(t) \right) \right)^\top
\right \rVert_\infty \right) .
	\end{align*}
  The second derivative of $h_{jk}$ w.r.t $\bm{x}(t)$ is:
	\begin{align*}
	&h_{jk}^{(2)}(\mathbb{E}\left(\bm{x}(t) \right),\bm{\eta})\\
	&=
	-2\sigma^{-6}(t) (1 - 2(\mu + ( \mathbb E \bm{x}(t) )^\top \bm{\beta}) )\bm{\beta}
	\left( (1 - 2(\mu + ( \mathbb E \bm{x}(t) )^\top \bm{\beta}) )\bm{\beta}  \right)^\top
	\mathbb{E}\left(\bm{x}^\top(t) \right) I_{jk} \mathbb{E}\left(\bm{x}(t) \right) \\
	&+ \sigma^{-4}(t) (  -2 \bm{\beta}\bm{\beta}^\top \mathbb{E}\left(\bm{x}^\top(t) \right) I_{jk} \mathbb{E}\left(\bm{x}(t) \right)  \\
	&+  (1 - 2(\mu + ( \mathbb E \bm{x}(t) )^\top \bm{\beta}) )\bm{\beta}( \bm{e}^\top_{j} + \bm{e}^\top_{k})    )
	+  \sigma^{-4}(t)(I_{jk}+I^T_{jk}) .
	\end{align*}
	Thus, with bounded $x_j(t)$ and $\lambda_i(t)$,
	\begin{align*}
	\lVert h_{jk}^{(2)}(\mathbb{E}\left(\bm{x}(t) \right),\bm{\eta}) \rVert_1 =
     O\left( \rho \lVert \bm{\beta}\rVert_\infty^2 \right) +  O\left( \lVert \bm{\beta}\rVert_\infty\right),
 	\end{align*}
 	and
 	\begin{align*}
 	& \lVert  h_{jk}^{(2)}(\mathbb{E}\left(\bm{x}(t) \right),\bm{\eta}) \mathbb{E}\left(\bm{x}(t) \right) \rVert_1
 = O\left( \rho \lVert \bm{\beta}\rVert_\infty \right) .
 	\end{align*}
    In addition, under a stationary Hawkes process, $\mathbb{E}\left(\bm{x}(t) \right)$ is a constant only depending on the model parameters $(\bm{\mu}, \Theta)$. Thus,
	\begin{multline*}
	\left | \frac{1}{T} \sum_{t=1}^T   \bm{x}^\top(t) h_{jk}^{(2)}(\mathbb{E}\left(\bm{x}(t) \right),\bm{\eta}) \bm{x}(t) - \mathbb{E}\left( \bm{x}^\top(t) h_{jk}^{(2)}(\mathbb{E}\left(\bm{x}(t) \right),\bm{\eta}) \bm{x}(t) \right) \right |  \\
	=
	O\left( \rho \left \lVert \frac{1}{T} \sum_{t=1}^T   \bm{x}^\top (t) \bm{x}(t)  - \mathbb{E}\left( \bm{x}(t)^\top  \bm{x}(t) \right) \right \rVert_\infty \right) ,
	\end{multline*}
    and
	\begin{align*}
	&\left \lVert \frac{1}{T} \sum_{t=1}^T  \left( \bm{x}(t) - \mathbb{E}\left(\bm{x}(t) \right) \right)^\top  h_{jk}^{(2)}(\mathbb{E}\left(\bm{x}(t) \right),\bm{\eta}) \mathbb{E}\left(\bm{x}(t) \right) \right \rVert_\infty \\
	&\qquad\qquad\le   \left \lVert  h_{jk}^{(2)}(\mathbb{E}\left(\bm{x}(t) \right),\bm{\eta}) \mathbb{E}\left(\bm{x} (t) \right) \right \rVert_1 \left \lVert \frac{1}{T} \sum_{t=1}^T  \bm{x}(t) - \mathbb{E}\left(\bm{x}(t) \right)  \right \rVert_\infty \\
	&\qquad\qquad\le
	 O \left( \rho
	 \left \lVert \frac{1}{T} \sum_{t=1}^T  \bm{x}(t) - \mathbb{E}\left(\bm{x}(t) \right)  \right \rVert_\infty \right) .
	\end{align*}
	The Taylor expansion of $h_{jk}(\bm{x}(t) , \bm{\eta})$ around $\mathbb{E}\left(\bm{x}(t) \right)$ is
	\begin{align*}
	h_{jk}(\bm{x}(t),\bm{\eta}) &=
	h_{jk}(\mathbb{E}\left(\bm{x}(t) \right),\bm{\eta})
	+ h_{jk}'\left(\mathbb{E}\left(\bm{x}(t) \right),\bm{\eta}\right)^\top\left( \bm{x}(t) - \mathbb{E}\left(\bm{x}(t) \right) \right)\\
	+& \frac{1}{2} \left( \bm{x}(t) - \mathbb{E}\bm{x}(t)  \right)^\top   h_{jk}^{(2)}\left(\mathbb{E}\left(\bm{x}(t) \right),\bm{\eta}\right)\left( \bm{x}(t) - \mathbb{E}\left(\bm{x}(t) \right) \right) \\
	+& o\left( \left( \bm{x}(t) - \mathbb{E}\left(\bm{x}(t) \right) \right)^\top  \left( \bm{x}(t) - \mathbb{E}\left(\bm{x}(t) \right) \right)   \right) .
	\end{align*}
	Then,
	\begin{align*}
	& \frac{1}{T} \sum_{t=1}^T    h_{jk}(\bm{x}(t),\bm{\eta}) - \mathbb{E} \left( h(\bm{x}(t),\bm{\eta}) \right) \\
	=& \frac{1}{T} \sum_{t=1}^T   \left( \bm{x}(t) - \mathbb{E}\left(\bm{x}(t) \right) \right)^\top h_{jk}'(\mathbb{E}\left(\bm{x}(t) \right),\bm{\eta}) \\
	+&
	\frac{1}{T} \sum_{t=1}^T   \bm{x}^\top (t) h_{jk}^{(2)}(\mathbb{E}\left(\bm{x}(t) \right),\bm{\eta}) \bm{x}(t) - \mathbb{E}\left( \bm{x}^\top (t) h_{jk}^{(2)}(\mathbb{E}\left(\bm{x}(t) \right),\bm{\eta}) \bm{x}(t) \right) \\
	-& 2 \frac{1}{T} \sum_{t=1}^T   \left( \bm{x}(t) - \mathbb{E}\left(\bm{x}(t) \right) \right)^\top  h_{jk}^{(2)}(\mathbb{E}\left(\bm{x}(t) \right),\bm{\eta}) \mathbb{E}\left(\bm{x} (t) \right) \\
	+& o\left(
	\frac{1}{T} \sum_{t=1}^T   \bm{x}^\top(t) \bm{x} (t) - \mathbb{E}\left( \bm{x}^\top(t) \bm{x}(t) \right)
	- 2 \frac{1}{T} \sum_{t=1}^T   \left( \bm{x}(t) - \mathbb{E}\left(\bm{x}(t) \right) \right)^\top  \mathbb{E}\left(\bm{x}(t) \right)
	\right) .
	\end{align*}
	Therefore,
	\begin{align*}
	\left \lVert  \frac{1}{T} \sum_{t=1}^T    h_{jk}(\bm{x}(t),\bm{\eta}) - \mathbb{E} \left( h_{jk}(\bm{x}(t),\bm{\eta}) \right) \right \rVert_\infty
	&\le
	O\left( \rho \left \lVert \frac{1}{T} \sum_{t=1}^T   \bm{x}^\top (t) \bm{x}(t)  - \mathbb{E}\left( \bm{x}(t)^\top  \bm{x}(t) \right) \right \rVert_\infty \right) \\
	&+
 O\left( \rho
	\left \lVert \frac{1}{T} \sum_{t=1}^T  \bm{x}(t) - \mathbb{E}\left(\bm{x}(t) \right)  \right \rVert_\infty  \right).
	\end{align*}
Using Lemma~\ref{lemma_concen_ineq_x} and taking a union bound,
\begin{align*}
\mathbb{P} \left( \left|\frac{1}{T} \int_0^T   \bm{z}(t) \bm{z}^\top(t)     - \mathbb{E}   \bm{z}(t) \bm{z}^\top(t)   dt  \right| >  \rho \epsilon \right) \le C_1 p^2 T \exp(-C_2 (\epsilon T)^{1/3}) .
\end{align*}


To derive the bound for $\left\lVert \frac{1}{T} \sum_{t=1}^T  \bm{z}(t) -
\mathbb{E} \bm{z}(t)  \right \rVert $, we repeat the steps of
\eqref{eq:xt_Ext_expand} and \eqref{eq:xt_Ext_dev_bound} in
Lemma~\ref{lemma_concen_ineq_x}, but with $f_1(s) = \frac{1}{\sigma_i(t)}\int_s^T
k_j(t-s) dt$. By $\sigma_i(t) =O(1)$ in Lemma~\ref{lemma_bounded_terms}, we
reach the conclusion.

\end{proof}


\clearpage
\section*{Appendix~C: An Alternative Concentration Inequality for Discrete Time Domain}

Following the discussion at the end of Section~\ref{sec:theory}, here we give an alternative concentration inequality for the first- and second-order integrated processed in discrete time domain. We first state an alternative assumption on the transfer kernel function. 

\begin{assumption}\label{assumption_transit_kernel}
	There exists
	$b > a > 0$ such that the transfer kernel function satisfies
	$$
	0 < \max_{1\le j\le p} k_j(t) \le a \exp( - b t).
	$$
\end{assumption}
Compared to Assumption~\ref{assumption4}, we require a more stringent structure on the transition kernel function. Such structure allows us to generate an improved convergence rate of the second order statistics of $\bm{x}(t)$ as
stated in Lemma~\ref{lemma_concen_ineq_x2}.

\begin{proposition} \label{prop_transit_kernel_implication}
	Suppose Assumption~\ref{assumption3} and Assumption~\ref{assumption_transit_kernel} hold.
	Then
\begin{enumerate}
	\item there exists $ C_{\beta}$, such that $\lVert\bm{\beta}_i \rVert_\infty \le C_{\beta} < \infty$, $i = 1, \ldots, p$;
	\item
$\max_{1\le j\le p}  \sum_{t=1}^T  k_j(t)   < \infty $ ;
\end{enumerate}
Furthermore, let $\omega_{ij}^{*n}$ be $n$-th auto-convolution of $\omega_{ij}$.
Then
\begin{align*}
\omega_{ij}^{*n}(t) &\le \beta_{ij} a^n \frac{t^{(n-1)} }{(n-1)!} \exp(-bt), \nonumber \\
\Psi_{ij}(t) &\equiv \sum_{n=1}^{\infty }  \omega_{ij}^{*n}(t) \le   \beta_{ij} a \exp( - (b-a)t ) , \nonumber \\
\intertext{and}
\xi  &\equiv  \max_{1\le i \le p}  \sum_{j=1}^p \sum_{t=1}^T  |\Psi_{ij}(t) |  \le
\rho C_{\beta}  \frac{ a \exp( -(b-a) )  }{ 1 - \exp( -(b-a) )   }  < \infty .
\label{eq:xi}
\end{align*}
\end{proposition}

\begin{proof}

First, Assumption~\ref{assumption_transit_kernel} implies that
$0 < x_j(t) < \infty$, $j = 1, \ldots, p$. Thus,
if $\exists \beta_{ij} = \pm \infty $,
then $\lambda_i(t) = \mu_i + \bm{x}^\top(t)\bm{\beta}_i$
goes to $\infty$, which contradicts with Assumption~\ref{assumption3}
that $\lambda_i(t)$ is bounded.

By direct algebra, we have
	\begin{align}
	 \max_{1\le j\le p}  \sum_{t=1}^T  k_j(t)  \le \frac{ a \exp(-b)}{1 - \exp(-b)}  < \infty  ,
	\end{align}
	which proves the second point.

For the $n$-th auto-convolution, we have
for $n=2$,
	\begin{multline*}
	\omega_{ij}^{*2}(t) = \int_0^T \omega_{ij}(t-s) \omega_{ij}(s) ds\\
	\le \beta_{ij} a^2 \int_0^t a \exp( - b (t-s) ) a \exp( - b s) ds
	= \beta_{ij} a^2 t \exp(-bt);
	\end{multline*}
	and, for $n=3$,
	\begin{multline*}
	\omega_{ij}^{*3}(t) = \int_0^T \omega_{ij}^{*2}(t-s) \omega_{ij}(s) ds\\
	\le   \int_0^t\beta_{ij} a^2 (t-s) \exp(-b(t-s) )  a \exp( - b s) ds
	= \beta_{ij} a^3 \frac{t^2}{2} \exp(-bt) .
\end{multline*}
Suppose the result holds for $n$, then for $n+1$:
\begin{multline*}
\omega_{ij}^{*(n+1)}(t) = \int_0^T \omega_{ij}^{*n}(t-s) \omega_{ij}(s) ds\\
\le   \int_0^t \beta_{ij}   a^n \frac{t^{(n-1)} }{(n-1)!} \exp(-b(t-s))  a \exp( - b s) ds
=  \beta_{ij} a^{n+1} \frac{t^{n} }{n!} \exp(-bt),
	\end{multline*}
which implies that the result holds by induction.
Then, by direct algebra, we have
	\begin{align*}
	\sum_{n=1}^{\infty }  \omega_{ij}^{*n}(t) & \le   \beta_{ij} a \exp( - (b-a)t )  \\
	\intertext{and}
	 \max_{1\le i \le p}  \sum_{j=1}^p \sum_{t=1}^T  |\Psi_{ij}(t) |  &\le
	\rho C_{\beta}  \frac{ a \exp( -(b-a) )  }{ 1 - \exp( -(b-a) )   }  < \infty .
	\end{align*}

\end{proof}

\begin{lemma}
    \label{lemma_concen_ineq_x2}
    Suppose the linear Hawkes process with its intensity function defined in
    \eqref{eq:linear_hawkes_para_transfer} is a stationary process defined in a discrete time domain that satisfies Assumptions~\ref{assumption1}--\ref{assumption3}
		and Assumption~\ref{assumption_transit_kernel}.
		Then, for $\delta > 0$ and $1\le i,j \le p$,
	\begin{multline}
	\mathbb{P} \left (
	\left| \frac{1}{T} \sum_{t=1}^{T} x_i(t)x_j(t)^\top   -
	\mathbb{E}  \left( \frac{1}{T} \sum_{t=1}^{T} x_i(t)x_j(t)^\top   \right) \right| > \delta  \right)\\
	 \le
	C_1 \exp \left ( - C_2
	\min\left \{ \sqrt{\frac{T}{\rho }}\delta, \frac{T}{\rho}\delta^2  \right \}
	\right ) ,\label{eq:xt2_consistency}
	\end{multline}
and
\begin{align}
\mathbb{P} \left (\left| \frac{1}{T} \sum_{t=1}^T  x_j(t) - \mathbb{E} x_j(t)  \right | > \delta  \right )
\le  C_3 \exp(-C_4 T \delta^2)  ,
\label{eq:xt_consistency}
\end{align}
where $\rho = \max \lVert \bm{\beta}_i \rVert_0$ and
where $C_k, k=1,\dots, 4$ are constants only depending on the model parameter $(\bm{\mu}, \Theta)$ and the transition kernel function.
\end{lemma}

\begin{proof}

Consider the linear Hawkes process defined on a discrete time domain in
$t=1,\dots, T$. As defined in \eqref{eq:design_column_xt}, $x_j(t) =
\sum_{s=1}^{t-1} k_j(t-s) Y_j(s) , t>1 $, and set $x_j(1) = 0$.
Let
$$
K_j = \mat{ 0& k_j(1) & k_j(2)  & k_j(3)   & \dots & k_j(T-1)  \\
	0& 0  & k_j(1) & k_j(2)   & \dots & k_j(T-2)  \\
	0& 0  &  0 & k_j(1)   & \dots & k_j(T-3)  \\
	0& . &   .            & .        &   \dots   & .  \\
	0& 0  &  0            &0               &0           & k_j(1)  \\
	0& 0  &  0            &0               &0           & 0
}
\in \mathbb{R}^{T \times T} ,
$$
where $k_j(\cdot)$ is the transition kernel function, and
let $Y_{jt} = (Y_j(t) ,\dots, Y_j(1) )^\top  \in \mathbb{R}^t  $. Then,
\begin{align*}
\bm{x}_j \equiv \mat{ x_j(T) \\ \dots \\ x_j(1) } =  K_j Y_{jT} .
\end{align*}
Recall that $\lambda_j(t)$, defined in \eqref{eq:linear_hawkes_para_transfer}, is the intensity function of unit $j$ at time $t$; $\epsilon_j(t) =  Y_j(t) - \lambda_j(t) $. Let $\bm{\epsilon}(s) = ( \epsilon_1(s),\dots, \epsilon_p(s) )^\top  \in \mathbb{R}^p $.
Under a stationary Hawkes process
\citep[][Proposition 1]{Bacry2011}:
\begin{align*}
Y_j(t) = \Lambda_j +  \Psi_j * \epsilon_j(t)
= \Lambda_j +    \sum_{s=1}^t \Psi_j(t-s)  \bm{\epsilon}(s) .
\end{align*}
Here $\Lambda_j = \mathbb{E} \lambda_j(t) $; $\Psi_j(t) = ( \Psi_{j1}(t),\dots, \Psi_{jp}(t) )$ and
$\Psi_{jl}(t) = \sum_{n=1}^\infty  \omega_{jl}^{*n}(t) $, where
$\omega_{jl}^{*n}$ is $n$-th auto-convolution of the transfer function $\omega_{jl}$ defined in \eqref{eq:transfer_function}.

Let
$$
\Xi_j = \mat{ \Psi_j(1) & \Psi_j(2)  & \Psi_j(3)  & \dots & \Psi_j(T)  \\
	0  & \Psi_j(1) & \Psi_j(2)   & \dots & \Psi_j(T-1)  \\
	0  &  0 & \Psi_j(1)  & \dots  & \Psi_j(T-2)  \\
	. &   .            & .                 &   \dots   & .    \\
	0  &  0            &0               &0         & \Psi_j(1) }   \in \mathbb{R}^{T\times Tp} ,
$$
Further, let $\bm{\epsilon} = (\bm{\epsilon}(T) ,\dots, \bm{\epsilon}(1) )^\top  \in \mathbb{R}^{Tp} $ , then
\begin{align*}
Y_{jT} &= \Lambda_j  +   \Xi_j \bm{\epsilon} .
\end{align*}
Thus,
\begin{align*}
\bm{x}_j    = K_j   ( \Lambda_j  +   \Xi_j \bm{\epsilon} ) &= K_j\Lambda_j + K_j\Xi_j \bm{\epsilon} ,\\
\bm{x}_j - \mathbb{E} \bm{x}_j &= K_j\Xi_j \bm{\epsilon},
\end{align*}
which leads to
$$
\bm{x}_i^\top  \bm{x}_j - \mathbb{E}\left( \bm{x}_i^\top  \bm{x}_j\right) =  \Lambda_i^\top K_i^\top   K_j \Xi_j \bm{\epsilon} +
\Lambda_j^\top  K_j^\top    K_i \Xi_i \bm{\epsilon} +
\bm{\epsilon}^\top   \Xi_i^\top  K_i^\top  K_j \Xi_j \bm{\epsilon} .
$$
We bound each term in the display above.
First, notice that
\begin{align*}
\lVert \Lambda_i^\top  K_i^\top   K_j \Xi_j \rVert_2^2
\le \lVert \Lambda_i \rVert_2^2 \Lambda_{\max}
\left( K_i^\top   K_j \Xi_j  \Xi^\top  _j K^\top _j K_i  \right ) .
\end{align*}
According to Assumption~\ref{assumption3}, $\lVert \Lambda_i \rVert_2^2 \le \lambda_{\max} T$.
In addition,
\begin{align*}
\Lambda_{\max}
\left( K_i^\top   K_j \Xi_j  \Xi^\top  _j K^\top _j K_i  \right )
&\le \Lambda_{\max} \left( K_i  \right )^2 \Lambda_{\max}\left(K_j \right)^2  \Lambda_{\max}\left( \Xi_j \right)^2 \\
&\le \left( \sum_{t=1}^T  k_i(t) \right)^2  \left(  \sum_{t=1}^T  k_j(t) \right)^2
\left( \sum_{k=1}^p \sum_{t=1}^T  \Psi_{ik}(t)   \right)^2  ,
\end{align*}
which is bounded by Assumption~\ref{assumption_transit_kernel} and
its implication in Proposition~\ref{prop_transit_kernel_implication}, where the second inequality is based on Perron–Frobenius theorem.
Therefore, applying the sub-Gaussian deviation bound \citep[][Prop 5.10]{vershynin2010},
$$
\mathbb{P}\left (\left\lVert  \Lambda_i^\top  K_i^\top   K_j \Xi_j \bm{\epsilon}  \right\rVert_2  > T\delta  \right )
\le  c_1 \exp(-c_2 T \delta^2) .
$$
Similarly,
$$
\mathbb{P}\left (\left\lVert  \Lambda_j^\top  K_j^\top    K_i \Xi_i \bm{\epsilon} \right\rVert_2  > T\delta \right )
\le  c_3 \exp(-c_4 T \delta^2).
$$

Next, we bound $\bm{\epsilon}^\top   \Xi_i^\top  K_i^\top  K_j \Xi_j \bm{\epsilon}$ by examining the structure of $K_i$ and $\Xi_i$.
By Assumption~\ref{assumption_transit_kernel} and its implication in Proposition~\ref{prop_transit_kernel_implication}, $k_j(t) \le a \exp(-b t)$ and $\Psi_{ij}(t) \le C_1 \exp(-c t)$, where $c = b-a$ and $C_1 =a C_\beta  $.
Note that $\lVert \bm{\beta}_i  \rVert_1 = \rho_i \le \rho = \max_{1\le i \le p} \rho_i$ implies that there are at most $\rho$ items of $\Psi_{il}(t) \ne 0$. So instead of considering the entire $p$-variate process, we only consider at most $2\rho$ units such that $\Psi_{il}(t) \ne 0$ and $\Psi_{jl}(t) \ne 0$. Let
$$
\widetilde{K} = a \mat{ 0& \exp(-bt ) &  \exp(-2bt )  & \exp(-3bt )   & \dots & \exp(-b(T-1) )  \\
	0& 0  & \exp(-bt )&  \exp(-2bt )   & \dots & \exp(-b(T-1) ) \\
	0& 0  &  0 & \exp(-bt )  & \dots & \exp(-b(T-3) )  \\
	0& . &   .            & .           &   \dots   & .  \\
	0& 0  &  0            &0               &0          & \exp(- bt )  \\
	0& 0  &  0            &0               &0          & 0
}
\in \mathbb{R}^{T \times T} ,
$$
and
$$
\widetilde{\Xi} =  C_1 \mat{ \mathbf{0}^\top _{2\rho}& \exp(-ct )\mathbf{1}^\top _{2\rho} &  \exp(-2ct )\mathbf{1}^\top _{2\rho}   & \exp(-3ct )\mathbf{1}^\top _{2\rho}   & \dots & \exp(-c(T-1) )\mathbf{1}^\top _{2\rho}  \\
	\mathbf{0}^\top _{2\rho}& \mathbf{0}^\top _{2\rho}  & \exp(-ct )\mathbf{1}^\top _{2\rho} &  \exp(-2ct )\mathbf{1}^\top _{2\rho}  & \dots & \exp(-c(T-2) )\mathbf{1}^\top _{2\rho} \\
	\mathbf{0}^\top _{2\rho}& \mathbf{0}^\top _{2\rho}  &  \mathbf{0}^\top _{2\rho} &  \exp(-ct )\mathbf{1}^\top _{2\rho}  & \dots & \exp(-c(T-3) )\mathbf{1}^\top _{2\rho}  \\
	\mathbf{0}^\top _{2\rho}& . &   .            & .        &   \dots    & .  \\
	\mathbf{0}^\top _{2\rho}& \mathbf{0}^\top _{2\rho}  &  \mathbf{0}^\top _{2\rho}            &\mathbf{0}^\top _{2\rho}     & \mathbf{0}^\top _{2\rho}      & \exp(-ct )\mathbf{1}^\top _{2\rho}   \\
	\mathbf{0}^\top _{2\rho}& \mathbf{0}^\top _{2\rho}  &  \mathbf{0}^\top _{2\rho}            &\mathbf{0}^\top _{2\rho}     & \mathbf{0}^\top _{2\rho}      & \mathbf{0}^\top _{2\rho}
}
\in \mathbb{R}^{T \times 2\rho T} ,
$$
where $\mathbf{a}^\top _{2\rho} = \underbrace{ (a,a,\dots, a) }_{2\rho}$ and $a \in \{0,1\}$.
Then,
\begin{align*}
\left\lVert \Xi_i^\top  K_i^\top  K_j \Xi_j   \right\rVert_2^2
\le \left\lVert  \widetilde{\Xi}^\top  \widetilde{K}^\top   \widetilde{K}  \widetilde{\Xi} \right\rVert_2^2 .
\end{align*}
Let $\Theta = \widetilde{K} \widetilde{\Xi}$, then we calculate
$$
\Theta_{1,k(s-1)+1} = a C_1 \sum_{s=1}^k \exp(-bs ) \exp( - c(k+1-s) ) \le
 C_2  \exp(-ak ) ,
$$
where $C_2 = a C_1 $.
Therefore,  $ \left\lVert \widetilde{K} \widetilde{\Xi} \right\rVert_2^2 \le
\left\lVert \widetilde{\Theta} \right\rVert_2^2
$
where
$$
\widetilde{\Theta} = C_2 \mat{ \mathbf{0}_{2\rho}^\top & \exp(-a )\mathbf{1}^\top _{2\rho} &  \exp(-2a )\mathbf{1}^\top _{2\rho}   & \exp(-3a )\mathbf{1}^\top _{2\rho}   & \dots & \exp(-a(T-1) )\mathbf{1}^\top _{2\rho}  \\
	\mathbf{0}_{2\rho}^\top & \mathbf{0}_{2\rho}^\top   & \exp(-a )\mathbf{1}^\top _{2\rho} &  \exp(-2a )\mathbf{1}^\top _{2\rho}  & \dots &  \exp(-a(T-2) )\mathbf{1}^\top _{2\rho} \\
	\mathbf{0}_{2\rho}^\top & \mathbf{0}_{2\rho}^\top   &  \mathbf{0}_{2\rho}^\top  &  \exp(-a )\mathbf{1}^\top _{2\rho}  & \dots & \exp(-a(T-3) )\mathbf{1}^\top _{2\rho}  \\
	\mathbf{0}_{2\rho}^\top & . &   .            & .        &   \dots   & .  \\
	\mathbf{0}_{2\rho}^\top & \mathbf{0}_{2\rho}^\top   &  \mathbf{0}_{2\rho}^\top             &\mathbf{0}_{2\rho}^\top                &\mathbf{0}_{2\rho}^\top      & \exp(-a )\mathbf{1}^\top _{2\rho}   \\
	\mathbf{0}_{2\rho}^\top & \mathbf{0}_{2\rho}^\top   &  \mathbf{0}_{2\rho}^\top             &\mathbf{0}_{2\rho}^\top                &\mathbf{0}_{2\rho}^\top     & \mathbf{0}_{2\rho}^\top
}
\in \mathbb{R}^{T \times 2\rho T} .
$$

Next, we check $ M = \widetilde{\Theta}^\top  \widetilde{\Theta}$. Due to the structure of $\widetilde{\Theta}$, we get
\begin{align*}
&M =\\
&C^2_2 \mat{ \mathbf{0}_{2\rho}^\top & \mathbf{0}_{2\rho}^\top  & \mathbf{0}_{2\rho}^\top  & \mathbf{0}_{2\rho}^\top  & \dots & \mathbf{0}_{2\rho}^\top  \\
	\mathbf{0}_{2\rho}^\top & m_1\mathbf{1}^\top _{2\rho} &  m_1\exp(-a)\mathbf{1}^\top _{2\rho}   & m_1\exp(-2a)\mathbf{1}^\top _{2\rho}   & \dots & m_1\exp(-a(T-2)) \mathbf{1}^\top _{2\rho}  \\
	. & .  & m_2\mathbf{1}^\top _{2\rho} &  m_2\exp(-a)\mathbf{1}^\top _{2\rho}  & \dots &  m_2\exp(-a(T-3)) \mathbf{1}^\top _{2\rho} \\
	.& .  &  . &   m_3\mathbf{1}^\top _{2\rho} & \dots &  m_3\exp(-a(T-4)) \mathbf{1}^\top _{2\rho}\\
	.& . &   .            & .        &   \dots  & .  \\
	.& .  &  .            &.               &.     &  m_{T-1}\mathbf{1}^\top _{2\rho} \\
	.& .  &  .            &.               &.     & \mathbf{0}_{2\rho}^\top
}
\in \mathbb{R}^{2\rho T \times 2\rho T} ,
\end{align*}
where $m_t = \sum_{l=1}^t  \exp(-al) \le  C_3 \exp(-a)$ for $t=1,\dots, T-1$.
We only write out the upper triangle part of $M$ in the above since $M =M^T$.
Therefore,
$$
\lVert M \rVert_2^2 \le  2 \sum_{i=1}^T  2\rho \sum_{k=1}^{T-1} m_i \exp(-ka ) \le  O\left( \rho T \right) .
$$

Since $\{\epsilon_i(t) \}_{1\le i\le p; t=1,\dots,T}$ are mutually independent
 centered at 0 with bounded variance according to Assumption~\ref{assumption3},
 applying Hanson-Wright inequality we get
\begin{multline*}
 \mathbb{P}  \left(  \left|
\bm{\epsilon}^\top   \Xi_i^\top  K_i^\top  K_j \Xi_j \bm{\epsilon}  -
 \mathbb{E}  \left(  \bm{\epsilon}^\top   \Xi_i^\top  K_i^\top  K_j \Xi_j \bm{\epsilon}  \right)  \right|  > T\delta \right)
 \\
\le c_5\exp( -c_6 \min\{ T\delta/\lVert M \rVert_2   , T^2\delta^2/\lVert M \rVert_2^2   \} )  .
\end{multline*}
Therefore,
\begin{align*}
 \mathbb{P}  \left(  \left|  \bm{x}_i^\top  \bm{x}_j - \mathbb{E} \bm{x}_i^\top  \bm{x}_j    \right|  > T\delta \right)
&\le  \mathbb{P}  \left(
\left | \Lambda_i^\top  K_i^\top   K_j \Xi_j \bm{\epsilon} \right | > T\delta  \right)\\
& +
 \mathbb{P}  \left(
\left|  \Lambda_j^\top  K_j^\top    K_i \Xi_i \bm{\epsilon} \right | > T\delta \right) \\
&+
 \mathbb{P}  \left(
\left | \bm{\epsilon}^\top   \Xi_i^\top  K_i^\top  K_j \Xi_j \bm{\epsilon}  \right | > T\delta
\right) \\
&\le  c_7\exp( - c_8 \min\big \{ \sqrt{\frac{T}{s}}\delta , \frac{T}{s} \delta^2 \big \}   ) .
\end{align*}
The above gives the concentration inequality of the second order statistics of $\bm{x}(t)$.

Now we derive the deviation bound for the first order statistics of $\bm{x}(t)$.
Recall that $\bm{x}_j - \mathbb{E} \bm{x}_j = K_j\Xi_j \bm{\epsilon}$ that we show earlier. In addition,
$$
\left\lVert  K_j\Xi_j  \right\rVert_2^2 \le
\left\lVert \widetilde{K} \widetilde{\Xi} \right\rVert_2^2 \le
\left\lVert \widetilde{\Theta} \right\rVert_2^2 = O\left( T \exp(-2a) \right).
$$
Then, applying the sub-Gaussian deviation bound stated in \citet[][Prop. 5.10]{vershynin2010},
\begin{align*}
 \mathbb{P} \left (\left| \frac{1}{T} \sum_{t=1}^T  x_j(t) -  \mathbb{E}  x_j(t)  \right | > \delta  \right)
\le c_9 \exp(-c_{10} T \delta^2) .
\end{align*}
\end{proof}

\end{document}